%% file: main.tex
\newif\ifdraft
\drafttrue

\ifdraft
\documentclass{article}[11pt]
\else
\documentclass[sigconf]{acmart}
\fi

\usepackage[utf8]{inputenc}
\ifdraft
\usepackage{amsmath, amsthm, amsfonts, amssymb}
\usepackage{natbib}
\else
\usepackage{amsmath, amsthm, amsfonts}
\fi
\usepackage{tikz}
\usepackage[boxed, noend]{algorithm2e}
\usepackage{xcolor}
\usepackage{cite}
\usepackage{subfig}
\usepackage{float}
\usepackage{bbm}
\usepackage{hyperref}
\ifdraft
\usepackage{fullpage}
\usepackage{authblk}
\fi
\usepackage{thm-restate}
\usepackage{array}

\ifdraft
\newcommand{\ar}[1]{\textcolor{blue}{[Aaron: #1]}}
\newcommand{\igh}[1]{\textcolor{red}{[Ira: #1]}}
\newcommand{\ph}[1]{\textcolor{orange}{[Peter: #1]}}
\newcommand{\mk}[1]{\textcolor{purple}{[Michael: #1]}}

\else
\newcommand{\ar}[1]{}
\newcommand{\igh}[1]{}
\newcommand{\ph}[1]{}
\newcommand{\mk}[1]{}
\fi

\ifdraft
\title{An Algorithmic Framework for Bias Bounties}
\author[1,2]{Ira Globus-Harris}
\author[1,2]{Michael Kearns}
\author[1,2]{Aaron Roth}
\affil[1]{\small University of Pennsylvania}
\affil[2]{\small Amazon AWS AI}
\else
\title{An Algorithmic Framework for Bias Bounties}
\author{Ira Globus-Harris}
\affiliation{
\institution{University of Pennsylvania and Amazon AWS AI}
\country{USA}
}

\email{igh@seas.upenn.edu}

\author{Michael Kearns}
\affiliation{
\institution{University of Pennsylvania and Amazon AWS AI}
\country{USA}
}
\email{mkearns@cis.upenn.edu}

\author{Aaron Roth}
\affiliation{
\institution{University of Pennsylvania and Amazon AWS AI}
\country{USA}
}
\email{aaroth@cis.upenn.edu}
\fi
\ifdraft
\else

\AtBeginDocument{%
  \providecommand\BibTeX{{%
    \normalfont B\kern-0.5em{\scshape i\kern-0.25em b}\kern-0.8em\TeX}}}

\setcopyright{acmcopyright}
\copyrightyear{2022}
\acmYear{2022}
\acmDOI{10.1145/1122445.1122456}

\acmConference[FAccT '22]{FAccT '22: ACM Conference on Fairness, Accountability, and Transparency}{2022}{}
\acmBooktitle{FAccT '22: ACM Conference on Fairness, Accountability, and Transparency}
\acmPrice{15.00}
\acmISBN{978-1-4503-XXXX-X/18/06}

\acmSubmissionID{513}

\fi
\begin{document}

\newcommand{\err}{\varepsilon}
\newcommand{\dom}{\text{dom}}

\newcommand{\cX}{\mathcal{X}}
\newcommand{\cY}{\mathcal{Y}}
\newcommand{\cD}{\mathcal{D}}
\newcommand{\cP}{\mathcal{P}}
\newcommand{\cG}{\mathcal{G}}
\newcommand{\cC}{\mathcal{C}}
\newcommand{\cH}{\mathcal{H}}
\newcommand{\cO}{\mathcal{O}}

\newcommand{\sqo}{\mathcal{O}_{\text{SQ}}}

\newcommand{\E}{\mathop{\mathbb{E}}}

\ifdraft
\newtheorem{theorem}{Theorem}
\newtheorem{lemma}[theorem]{Lemma}
\newtheorem{definition}[theorem]{Definition}
\fi
\newtheorem{problem}{Problem}
\newtheorem{remark}[theorem]{Remark}
\newtheorem{observation}[theorem]{Observation}

\SetKwInput{KwInput}{Input}
\SetKwInput{KwOutput}{Output}

\ifdraft
\else
\input{abstract}


\fi

\maketitle

\ifdraft
\input{abstract}
\thispagestyle{empty}
\setcounter{page}{0}
\newpage
\fi

\section{Introduction}

Modern machine learning (ML) is well known for its powerful applications, and more recently, for its potential to train discriminatory models. As organizations become more aware of these negative consequences, many are instituting partial
remedies such as responsible AI teams, auditing practices,  and model cards \citep{modelcards} and data sheets \citep{datasheets} that document ML workflows.

Such practices are likely to remain insufficient for at least two reasons. First, even an organization that diligently
audits its models on carefully curated datasets cannot anticipate all possible downstream use cases. A face recognition
system designed and tested for use in well-lit, front-facing conditions may be applied by a customer in less ideal
conditions, leading to degradation in performance overall and on particular subgroups. Second, we live in an era of AI activism, in which teams of researchers, journalists, and other external parties 
can independently audit models with commercial or open source APIs, and publish their findings in both the research community and mainstream media.
In the absence of more stringent regulation of AI and ML, this activism is 
one of the strongest forces pushing for change on
organizations developing ML models and tools. 
It has the ability to leverage distributed teams of researchers searching for a wide variety of kinds of problems
in deployed models and systems.

Unfortunately, the dynamic between ML developers and their critics currently has a somewhat adversarial tone.
An external audit will often be conducted privately, and its results aired publicly, without prior consultation
with the organization that built the model and system. That organization may endure criticism, and attempt
internally to fix the identified biases or problems, but typically there is little or no direct interaction
with external auditors.

Over time, the software and security communities have developed ``bug bounties'' in an attempt to turn similar
dynamics between system developers and their critics (or hackers) towards more interactive and productive ends.
The hope is that by deliberately inviting external parties to find software or hardware bugs in their systems, and often providing
monetary incentives for doing so, a healthier and more rapidly responding ecosystem will evolve.

It is natural for the ML community to consider a similar ``bias bounty'' approach to the timely discovery
and repair of models and systems with bias or other undesirable behaviors. Rather than finding bugs in software,
external parties are invited to find biases --- for instance, (demographic or other) subgroups of inputs on
which a trained model underperforms --- and are rewarded for doing so. Indeed, we are already starting to see early
field experiments in such events~\citep{twitterbounty}.

In this work, we propose and analyze a principled algorithmic framework for conducting bias bounties
against an existing trained model $f(x)$. Our framework has the following properties and features:
\vspace{-\topsep}
\begin{itemize}
    \item ``Bias hunters'' audit $f$ by submitting pairs of models $(g,h)$--- a model $g$ identifying a subset of inputs to $f$ that $f$ performs poorly on, and a proposed model $h$ which improves on this subset. Requiring participants to identify not just a group but an improved model ensures that improvement on the group
        is in fact possible (as opposed to identifying a fundamentally hard subgroup, such as
        images of occluded faces in a face recognition task).
    \item The proposed groups do not need to be identified in advance, and the improving models do not
        need to be chosen from predetermined parametric classes; indeed, both groups and improving models
        can be arbitrarily complex. This is in contrast to most fair ML frameworks, in which training is
        formulated as a constrained optimization problem over a fixed parametric family of models, with fairness
        constraints over fixed demographic subgroups.
    \item Given a proposed group and model pair, our algorithm validates the proposed improvement
        on a holdout set, using techniques from adaptive data analysis to circumvent the potential for overfitting
        due to the potentially unlimited complexity of the submitted pair.
    \item If the improvement is validated, our algorithm has a simple mechanism for automatically
        incorporating the improvement into an update of $f$ in a way that reduces {\em both the overall error
        and the error on the proposed subgroup}. Indeed, our algorithm has the further property that once a subgroup
        improvement has been accepted, the error on that subgroup will never increase (much) due to subsequent
        subgroup introductions. Thus there is {\em no tradeoff between overall and subgroup errors, or between 
        different subgroups}. This is again in contrast to constrained optimization approaches, where there
        is necessarily tension between fairness and accuracy. Our algorithm achieves this through the use of
        a new model class called {\em pointer decision lists}.
    \item Our algorithm provably and monotonically converges quickly to one of two possible outcomes: either we reach the Bayes optimal model, or 
    we reach a model that cannot be distinguished from the Bayes optimal, in the sense that the bias hunters can find no further improvements. If payments are made to the bias hunters in proportion to the scale of the improvement\ifdraft{ (the product of the size of the group they improve on and the magnitude of the improvement on that group)}\else\fi, we guarantee that the total payments made to the
        bias hunters can be bounded in advance.
    \item We can alternatively view our framework as an entirely algorithmic approach to training (minimax) fair
        and accurate models, by replacing the bias hunters with automated mechanisms for finding improving pairs.
        We propose and analyze two such mechanisms\ifdraft{: one based on a reduction to cost-sensitive classification,
        and an expectation-maximization style approach that alternates between finding the optimal model for
        a given subgroup and optimizing the subgroup for a given model.}\else{.}\fi
\end{itemize}
\vspace{-\topsep}
Our contributions include the introduction of the bias bounty framework; proofs of the
convergence, generalization and monotonicity properties of our algorithm; experimental
results on census-derived Folktables datasets~\citep{ding2021retiring}; and preliminary
findings from an initial bias bounty event held in an undergraduate class at the University of Pennsylvania.
\subsection{Limitations and Open Questions}
The primary limitation of our proposed framework is that it can only identify and correct sub-optimal performance on groups \emph{as measured on the data distribution for which we have training data}. It does not solve either of the following related problems:
\vspace{-\topsep}
\begin{enumerate}
    \item Our model appears to perform well on every group \emph{only because in gathering data, we have failed to sample representative members of that group}.
    \item The model that we have cannot be improved on some group \emph{only because we have failed to gather the features that are most predictive on that group}, and the performance would be improvable if only we had gathered better features.
\end{enumerate}
\vspace{-\topsep}
That is, our framework can be used to find and correct biases as measured on the data distribution from which we have data, but cannot be used to find and correct biases that come from having gathered the wrong dataset. In both cases, one primary obstacle to extending our framework is the need to be able to efficiently validate proposed fixes. For example, because we restrict attention to a single data distribution, given a proposed pair $(g,h)$, we can check on a holdout set whether $h$ in fact has improved performance compared to our model, on examples from group $g$. This is important to disambiguate distributional improvements compared to subsets of examples that amount to cherrypicking. How can we approach this problem when proposed improvements include evaluations on new datasets, for which we by definition do not have held out data?  Compelling solutions to this problem seem to us to be of high interest. We remark that a bias bounty program held under our proposed framework would at least serve to highlight \emph{where} new data collection efforts are needed, by disambiguating failures of training from failures for the data to properly represent a population: if a group continues to have persistently high error even in the presence of a large population of auditors in our framework, this is evidence that in order to obtain improved error on that group, we need to focus on better representing them within our data. 

\subsection{Related Work}
There are several strands of the algorithmic fairness literature that are closely related to our work. Most popular notions of algorithmic fairness (e.g. those that propose to equalize notions of error across protected groups as in e.g. \citep{hardt2016equality, agarwal2018reductions, gerrymandering,zafar2017fairness}, or those that aim to ``treat similar individuals similarly'' as in e.g. \citep{awareness,metriclearning,RY18,jung2019eliciting}) involve tradeoffs, in that asking for ``fairness'' involves settling for reduced accuracy. Several papers \citep{notradeoffs2,notradeoffs1} show that fairness constraints of these types need not involve tradeoffs (or can even be accuracy improving) \emph{on test data} if the training data has been corrupted by some bias model and is not representative of the test data. In cases like this, fairness constraints can act as corrections to undo the errors that have been introduced in the data. These kinds of results leverage differences between the training and evaluation data, and unlike our work, do not avoid tradeoffs between fairness and accuracy in settings in which the training data is representative of the true distribution.

A notable exception to the rule that fairness and accuracy must involve tradeoffs is the literature on multicalibration initiated by H{\'e}bert-Johnson et al.  \citep{multicalibration,multiaccuracy,momentmulti,multivalid,indistinguishability} that asks that a model's predictions be calibrated not just overall, but also when restricted to a large number of protected subgroups $g$. H{\'e}bert-Johnson et al. \citep{multicalibration} and Kim, Ghorbani, and Zou \citep{multiaccuracy} show that an arbitrary model $f$ can be postprocessed to satisfy multicalibration (and the related notion of ``multi-accuracy'') without sacrificing (much) in terms of model accuracy. Our aim is to achieve something similar, but for predictive error, rather than model calibration. There are two things driving the ``no tradeoff'' result for multicalibration, from which we take inspiration: 1) the fairness notion that they ask for is satisfied by the Bayes-optimal model, and 2) they do not optimize over a fixed model class, but rather a model class defined in terms of the groups $\cG$ that define their fairness notion. These two things will be true for us as well.

The notion of fairness that we ask for in this paper was studied in an online setting (in which the data, rather than the protected groups arrive online) by Blum and Lykouris \citep{BL20} and generalized by Rothblum and Yona \citep{RY21} as ``multigroup agnostic learnability.'' Noarov, Pai, and Roth \citep{NPR21} show how to obtain it in an online setting as part of the same unified framework of algorithms that can obtain multicalibrated predictions. The algorithmic results in these papers lead to complex models  --- in contrast, our algorithm produces ``simple'' predictors in the form of a decision list. In contrast to these prior works, we do not view the set of groups $\cG$ that we wish to offer guarantees to as fixed up front, but instead as something that can be discovered online, after models are deployed. Our focus is on fast algorithms to update existing models when new groups $g$ on which our model is performing poorly are discovered.

Concurrently and independently of our paper, Tosh and Hsu \citep{toshhsu21} study algorithms and sample complexity for multi-group agnostic learnability and give an algorithm (``Prepend'') that is equivalent to our Algorithm \ref{alg:basic} (``ListUpdate''). Their focus is on sample complexity of batch optimization, however --- in contrast to our focus on the discovery of groups on which our model is performing poorly online (e.g. as part of a ``bias bounty program''). They also are not concerned with the details of the optimization that needs to be solved to produce an update --- we give practical algorithms based on reductions to cost sensitive classification and empirical evaluation. Tosh and Hsu \citep{toshhsu21} also contains additional results, including algorithms producing more complex hypotheses but with improved sample complexity (again in the setting in which the groups are fixed up front).

The multigroup notion of fairness we employ \citep{BL20,RY21} aims to perform \emph{optimally} on each subgroup $g$, rather than \emph{equalizing} the performance across subgroups. This is similar in motivation to \emph{minimax} fairness, studied in \citep{minimax1,minimax2,minimax3}, which aim to minimize the error on the maximum error sub-group. \ifdraft Both of these notions of fairness have the merit that they produce models that pareto-dominate equal error solutions, in the sense that in a minimax (or multigroup) optimal solution, every group has weakly lower error than they would have in any equal error solution.  However, optimizing for minimax error over a fixed hypothesis class still results in tradeoffs in the sense that it results in higher error than the error optimal model in the same class, and that adding more groups to the guarantee makes the tradeoff more severe. \fi Our approach avoids tradeoffs by optimizing over a class that is dynamically expanded as the set of groups to be protected expands.

The idea of a ``bias bug bounty'' program dates back at least to a 2018 editorial of Amit Elazari Bar On \citep{bountyarticle}, and Twitter ran a version of a bounty program in 2021 to find bias issues in its image cropping algorithm \citep{twitterbounty}. These programs are generally envisioned to be quite different than what we propose here. On the one hand, we are proposing to automatically audit models and award bounties for the discovery of a narrow form of technical bias --- sub-optimal error on  well defined subgroups --- whereas the bounty program run by Twitter was open ended, with human judges and written submissions. On the other hand, the method we propose could underly a long-running program that could automatically correct the bias issues discovered in production systems at scale, whereas Twitter's bounty program was a labor intensive event that ran over the course of a week.

\section{Preliminaries}
We consider a supervised learning problem defined over \emph{labelled examples} $(x,y) \in \cX \times \cY$. This can represent (for example) a binary classification problem if $\cY = \{0,1\}$, a discrete multiclass classification problem if $\cY = \{1,\ldots,k\}$, or a regression problem if $\cY = \mathbb{R}$. We write $\cD$ to denote a joint probability distribution over labelled examples: $\cD \in \Delta (\cX \times \cY)$. We will write $D \sim \cD^n$ to denote a dataset consisting of $n$ labelled examples sampled i.i.d. from $\cD$. Our goal is to learn some \emph{model} represented as a function $f:\cX\rightarrow \cY$ which aims to predict the label of an example from its features, and we will evaluate the performance of our model with a loss function $\ell:\cY\times \cY \rightarrow [0,1]$, where $\ell(\hat y, y)$ represents the ``cost'' of mistakenly labelling an example that has true label $y$ with the prediction $f(x) = \hat y$. We will be interested in the performance of models $f$ not just \emph{overall} on the underlying distribution, but also on particular subgroups of interest. A subgroup corresponds to an arbitrary subset of the feature space $\cX$, which we will model using an indicator function:
\begin{definition}[Subgroups]
A \emph{subgroup} of the feature space $\cX$ will be represented as an indicator function $g:\cX\rightarrow \{0,1\}$. We say that $x \in \cX$ is in group $g$ if $g(x) = 1$ and $x$ is \emph{not} in group $g$ otherwise. Given a group $g$, we write $\mu_g(\cD)$ to denote its measure under the probability distribution $\cD$:
\ifdraft
$$\mu_g(\cD) = \Pr_\cD[g(x) = 1].$$
\else $\mu_g(\cD) = \Pr_\cD[g(x) = 1].$ \fi
We write $\mu_g(D)$ to denote the corresponding empirical measure under $D$, which results from viewing $D$ as the uniform distribution over its elements.
\end{definition}

We can now define the loss of a model both overall and on different subgroups:
\begin{definition}[Model Loss]
Given a model $f:\cX\rightarrow \cY$
We write $L(\cD,f)$ to denote the average loss of $f$ on distribution $\cD$:
\ifdraft
$$L(\cD,f) = \E_{(x,y) \sim \cD}[\ell(f(x),y)]$$
\else $L(\cD,f) = \E_{(x,y) \sim \cD}[\ell(f(x),y)]$. \fi
We  write $L(\cD,f,g)$ to denote the loss on $f$ conditional on membership in $g$:
$$L(\cD,f,g) = \E_{(x,y) \sim \cD}[\ell(f(x),y) | g(x) = 1].$$
Given a dataset $D$, we write $L(D,f)$ and $L(D,f,g)$ to denote the corresponding empirical losses on $D$, which result from viewing $D$ as the uniform distribution over its elements.
\end{definition}

The best we can ever hope to do in any prediction problem (fixing the loss function and the distribution) is to make predictions that are as accurate as those of a Bayes optimal model:
\begin{definition}
A \emph{Bayes Optimal} model $f^*:\cX \rightarrow \cY$ with respect to a loss function $\ell$ and a distribution $\cD$ satisfies:
\ifdraft
$$f^*(x) \in \arg\min_{y \in \cY} \E_{(x', y') \sim \cD}\left[\ell(y, y') | x' = x\right]$$
\else $f^*(x) \in \arg\min_{y \in \cY} \E_{(x', y') \sim \cD}\left[\ell(y, y') | x' = x\right]$, \fi
where $f^*(x)$ can be defined arbitrarily for any $x$ that is not in the support of $\cD$.
\end{definition}

The Bayes optimal model is pointwise optimal, and hence has the lowest loss of any possible model, not just overall, but simultaneously on \emph{every} subgroup. In fact, its easy to see that this is a characterization of Bayes optimality.
\begin{observation}
\label{obs:BO}
Fixing a loss function $\ell$ and a distribution $\cD$, $f^*$ is a Bayes optimal model if and only if for every group $g$ and every alternative model $h$:
\ifdraft
$$L(\cD, f^*, g) \leq L(\cD, h, g)$$
\else $L(\cD, f^*, g) \leq L(\cD, h, g)$.\fi
\end{observation}

The above characterization states that a model is Bayes optimal if and only if it induces loss that is as exactly as low as that of \emph{any possible} model $h$, when restricted to \emph{any possible} group $g$. It will also be useful to refer to approximate notions of Bayes optimality, in which the exactness is relaxed, as well as possibly the class of comparison models $\cH$, and the class of groups $\cG$. We call this $(\epsilon,\cG, \cH)$-Bayes optimality to highlight the connection to (exact) Bayes Optimality, but it is identical to what Rothblum and Yona \citep{RY21} call a ``multigroup agnostic PAC solution'' with respect to $\cG$ and $\cH$. Related notions were also studied in \citep{BL20,NPR21}.

\begin{definition}
A model $f:\cX\rightarrow \cY$ is $(\epsilon,\cC)$-Bayes optimal with respect to a collection of (group, model) pairs $\cC$ if for each $(g,h) \in \cC$, the performance of $f$ on $g$ is within $\epsilon$ of the performance of $h$ on $g$. In other words, for every $(g,h) \in \cC$:
\ifdraft
$$L(\cD,f,g) \leq L(\cD,h,g) + \frac{\epsilon}{\mu_g(\cD)}.$$
\else $L(\cD,f,g) \leq L(\cD,h,g) + \frac{\epsilon}{\mu_g(\cD)}.$ \fi
When $\cC$ is a product set $\cC = \cG\times \cH$, then we call $f$
``$(\epsilon,\cG,\cH)$-Bayes Optimal" and the condition is equivalent to requiring that for every $g \in G$, $f$ has performance on $g$ that is within $\epsilon$ of the \emph{best} model $h \in \cH$ on $g$. When $\cG$ and $\cH$ represent the set of all groups and models respectively, we call $f$ $\epsilon$-Bayes optimal.
\end{definition}

\begin{remark}
We have chosen to define approximate Bayes optimality by letting the approximation term $\epsilon$ scale proportionately to the inverse probability of the group $g$, similar to how notions of multigroup fairness are defined in  \citep{gerrymandering,momentmulti,multivalid}. An alternative (slightly weaker) option would be to require error that is uniformly bounded by $\epsilon$ for all groups, but to only make promises for groups $g$ that have probability $\mu_g$ larger than some threshold, as is done in \citep{multicalibration}. Some relaxation of this sort is necessary to provide guarantees on an unknown distribution based only on a finite sample from it, since we will necessarily have less statistical certainty about smaller subsets of our data.
\end{remark}

Note that $(\epsilon,\cG,\cH)$-Bayes optimality is identical to Bayes optimality when $\epsilon = 0$ and when $\cG$ and $\cH$ represent the classes of all possible groups and models respectively, and that it becomes an increasingly stronger condition as $\cG$ and $\cH$ grow in expressivity. 

\section{Certificates of Sub-Optimality and Update Algorithms}
\label{sec:certificates}

Recall that ``bias hunters" submit a group that our existing model performs poorly on and an improvement on that group. We need to formulate what our requirements for accepting their proposed improvement would be, and develop a method to incorporate these fixes into our model without retraining it. Formally, suppose we have an existing model $f$, and we find that it is performing sub-optimally on some group $g$. By Observation \ref{obs:BO}, it must be that $f$ is not Bayes optimal, and this will be witnessed by some model $h$ such that:
\ifdraft
$$L(\cD,f,g) > L(\cD,h,g)$$
\else $L(\cD,f,g) > L(\cD,h,g)$. \fi
We call such a pair $(g,h)$ a certificate of sub-optimality. \ifdraft Note that by Observation \ref{obs:BO}, such a certificate will exist if and only if $f$ is not Bayes optimal. \fi We can define a quantitative version of these certificates:
\begin{definition}
A group indicator function $g:\cX\rightarrow \{0,1\}$ together with a model $h:\cX \rightarrow \cY$ form a $(\mu, \Delta)$-certificate of sub-optimality for a model $f$  under distribution $\cD$ if:
\begin{enumerate}
    \item Group $g$ has probability mass at least $\mu$ under $\cD$: $\mu_g(\cD) \geq \mu$, and
    \item $h$ improves on the performance of $f$ on group $g$ by at least $\Delta$: $L(\cD,f,g) \geq L(\cD,h,g) + \Delta$
\end{enumerate}
\ifdraft We say that $(g,h)$ form a certificate of sub-optimality for $f$ if they form a $(\mu,\Delta)$-certificate of optimality for $f$ for any constants $\mu,\Delta > 0$. \fi
\end{definition}

The core of our algorithmic updating framework will rely on a close quantitative connection between certificates of sub-optimality and approximate Bayes optimality. The following theorem can be viewed as a quantitative version of Observation \ref{obs:BO}. \ifdraft\else Its proof can be found in the appendix.\fi

\begin{restatable}{theorem}{BOQ}
\label{thm:BOQ}
Fix any $\epsilon > 0$, and any collection of (group,model) pairs $\cC$. There exists a $(\mu,\Delta)$-certificate of sub-optimality $(g,h) \in \cC$ for $f$ if and only if $f$ is not $(\epsilon,\cC)$-Bayes optimal for $\epsilon < \mu\Delta$.
\end{restatable}
\ifdraft
\begin{proof}
  We need to prove two directions. First, we will assume that $f$ is $(\epsilon,\cC)$-Bayes optimal, and show that in this case there do not exist any pairs $(g,h) \in \cC$ such that $(g,h)$ form a $(\mu,\Delta)$-certificate of sub-optimality with $\mu\cdot \Delta > \epsilon$.  Fix a pair $(g,h) \in \cC$. Without loss of generality, we can take $\Delta = L(\cD,f,g) - L(\cD,h,g)$ (and if $\Delta \leq 0$ we are done, so we can also assume that $\Delta > 0$). Since $f$ is $(\epsilon,\cC)$-Bayes optimal, by definition we have that:
$$\Delta =L(\cD,f,g) - L(\cD,h,g) \leq \frac{\epsilon}{\mu_g(\cD)}$$
Solving, we get that $\Delta\cdot \mu_g \leq \epsilon$ as desired.

Next, we prove the other direction: We assume that there exists a pair $(g,h) \in \cC$ that form an $(\mu, \Delta)$-certificate of sub-optimality, and show that $f$ is not $(\epsilon, \cC)$-Bayes optimal for any $\epsilon < \mu\cdot \Delta$. Without loss of generality we can take $\mu = \mu_g(\cD)$ and conclude:
 $$L(\cD,f,g) - L(\cD,h,g) \geq \Delta = \frac{\mu\cdot \Delta}{\mu_g(\cD)} \geq \frac{\epsilon}{\mu_g(\cD)}$$
which falsifies $(\epsilon, \cC)$-Bayes optimality for any $\epsilon < \mu\cdot \Delta$ as desired.
\end{proof}
\fi

Theorem \ref{thm:BOQ} tells us that if we are looking for evidence that a model $f$ fails to be Bayes Optimal (or more generally, fails to be $(\epsilon, \cC)$-Bayes optimal), then without loss of generality, we can restrict our attention to certificates of sub-optimality with large parameters --- these exist if and only if $f$ is significantly far from Bayes optimal. But it does not tell us what to do if we find such a certificate. Can we use a certificate of sub-optimality $(g,h)$ for $f$ to easily update $f$ to a new model that both corrects the suboptimality witnessed by $(g,h)$ and makes measurable progress towards Bayes Optimality? It turns out that the answer is \emph{yes}, and we can do this with an exceedingly simple update algorithm, which we analyze next. The update algorithm (Algorithm \ref{alg:basic}) takes as input a model $f$ together with a certificate of sub-optimality for $f$, $(g,h)$, and outputs an improved model based on the following intuitive update rule: If an example $x$ is in group $g$ (i.e. if $g(x) = 1$), then we will classify $x$ using $h$; otherwise we will classify $x$ using $f$. 
\\

\begin{algorithm}[H]
\ifdraft
\else
\small
\fi 
\KwInput{A model $f_t$, and a certificate of suboptimality $(g_{t+1},h_{t+1})$}
Define $f_{t+1}$ as follows:
\[ f_{t+1}(x)= \begin{cases}
      f_t(x) & \text{if } g_{t+1}(x)=0 \\
      h_{t+1}(x) & \text{if } g_{t+1}(x)=1
   \end{cases}
\]

\KwOutput{$f_{t+1}$}
\caption{ListUpdate($f_t,(g_{t+1},h_{t+1})$): An Update Algorithm Constructing a Decision List}
\label{alg:basic}
\end{algorithm}

\begin{restatable}{theorem}{progress}
\label{thm:progress}
Algorithm \ref{alg:basic} (ListUpdate) has the following properties. If $(g_{t+1},h_{t+1})$ form a $(\mu,\Delta)$-certificate of sub-optimality for $f_{t}$, and $f_{t+1} = \textrm{ListUpdate}(f_t,(g_{t+1},h_{t+1}))$ then:
\begin{enumerate}
    \item The new model matches the performance of $h_{t+1}$ on group $g_{t+1}$: $L(\cD,f_{t+1},g_{t+1}) = L(\cD, h_{t+1},g_{t+1})$, and
    \item The overall performance of the model is improved by at least $\mu\cdot \Delta$: $L(\cD, f_{t+1}) \leq L(\cD, f_t) - \mu\cdot \Delta$.
\end{enumerate}
\end{restatable}
\ifdraft
\begin{proof}
It is immediate from the definition of $f_{t+1}$ that $L(\cD,f_{t+1},g_{t+1}) = L(\cD, h_{t+1},g_{t+1})$, since for any $x$ such that  $g_{t+1}(x) = 1$, $f_{t+1}(x) = h_{t+1}(x)$. It remains to verify the 2nd condition. 
Because we also have that for every $x$ such that $g_{t+1}(x) = 0$, $f_{t+1}(x) = f_t(x)$, we can calculate:
\begin{eqnarray*}
L(\cD,f_{t+1}) &=& \Pr_{\cD}[g_{t+1}(x) = 0]\cdot \E_{\cD}[\ell(f_{t+1}(x), y) | g_{t+1}(x) = 0] + \Pr_{\cD}[g_{t+1}(x) = 1]\cdot \E_{\cD}[\ell(f_{t+1}(x), y) | g_{t+1}(x) = 1] \\
&=&  \Pr_{\cD}[g_{t+1}(x) = 0]\cdot \E_{\cD}[\ell(f_{t}(x), y) | g_{t+1}(x) = 0] + \Pr_{\cD}[g_{t+1}(x) = 1]\cdot \E_{\cD}[\ell(h_{t+1}(x), y) | g_{t+1}(x) = 1]\\
&\leq& \Pr_{\cD}[g_{t+1}(x) = 0]\cdot \E_{\cD}[\ell(f_{t}(x), y) | g_{t+1}(x) = 0] + \Pr_{\cD}[g_{t+1}(x) = 1]
\left(\E_{\cD}[\ell(f_{t}(x), y) | g_{t+1}(x) = 1] - \Delta \right)\\
&\leq& L(\cD, f_t) - \mu\Delta
\end{eqnarray*}
\end{proof}
\else \fi
\ifdraft
\else 
\fi

\ifdraft
\begin{figure}
    
    \includegraphics[scale=0.2]{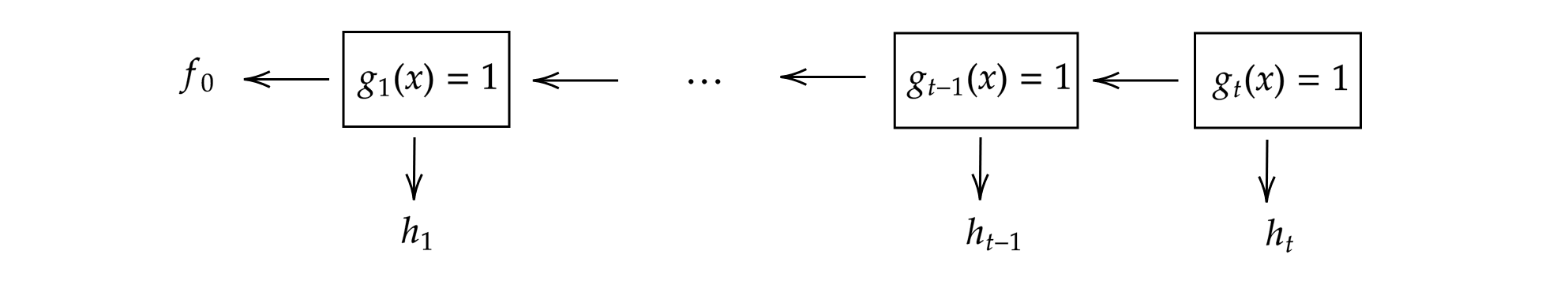}
    \caption{The decision list model $f_t$ output by Algorithm \ref{alg:basic} after $t$ iterations.}\label{fig:dl}
\end{figure}
\else
\fi

\ifdraft \else The proof can be found in the Appendix. \fi We can use Algorithm \ref{alg:basic} as an iterative update algorithm: If we have a model $f_0$, and then discover a certificate of sub-optimality $(g_1,h_1)$ for $f_0$, we can update our model to a new model $f_1$. If we then find a new certificate of sub-optimality $(g_2,h_2)$, we can once again use Algorithm \ref{alg:basic} to update our model to a new model, $f_2$, and so on.  The result is that at time $t$, we have a model $f_t$ in the form of a decision list in which the internal nodes branch on the group indicator functions $g_i$ and the leaves invoke the models $h_{i}$ or the initial model  $f_0$. \ifdraft See Figure \ref{fig:dl}. Note that to evaluate such a decision list on a point $x$, although we might need to evaluate $g_i(x)$ for every group indicator used to define the list, we only need to evaluate a single model $h_i(x)$. Moreover, as we will show next in Theorem \ref{thm:updatebound}, when the decision list is constructed iteratively in this manner, it cannot grow very long. Thus, evaluation can be fast even if the models used to construct it are complex.\fi The fact that each update makes progress towards Bayes Optimality (in fact, optimal progress, given Theorem \ref{thm:BOQ}) means that this updating process cannot go on for too long:
\begin{restatable}{theorem}{updatebound}
\label{thm:updatebound}
Fix any $\epsilon > 0$. For any initial model $f_0$ with loss $L(\cD,f_0) = \ell_0$ and any sequence of models $f_1,\ldots,f_T$, such that $f_i = \textrm{ListUpdate}(f_{i-1},(g_{i},h_{i}))$ and each pair $(g_i,h_{i})$ forms a $(\mu,\Delta)$-certificate of suboptimality for $f_{i-1}$ for some $\mu, \Delta$ such that  $\mu\cdot \Delta \geq \epsilon$, the length of the update sequence must be at most $T \leq \frac{\ell_0}{\epsilon} \leq \frac{1}{\epsilon}$.
\end{restatable}
\ifdraft
\begin{proof}
By assumption $L(\cD,f_0) = \ell_0$. Because each $(g_i,h_{i})$ is a $(\mu,\Delta)$-certificate of suboptimality of $f_{i-1}$ with $\mu\cdot \Delta \geq \epsilon$, we know from Theorem \ref{thm:progress} that for each $i$, $L(\cD,f_i) \leq L(\cD,f_{i-1}) - \epsilon$. Hence $L(\cD,f_T) \leq \ell_0 - T\epsilon$. But loss is non-negative: $L(\cD,f_T) \geq 0$. Thus it must be that $T \leq \frac{\ell_0}{\epsilon}$ as desired.
\end{proof}
\fi

What can we do with such an update algorithm? Given a model $f_t$, we can search for certificates of sub-optimality, and if we find them, we can make quantitative progress towards improving our model. We can then repeat the process. The guarantee of Theorem \ref{thm:updatebound} is that this process of searching and updating cannot go on for very many rounds $T$ before we arrive at a model $f_T$ that \emph{our search process is unable to demonstrate is not Bayes Optimal}. How interesting this is depends on what our search process is.

Suppose, for example, that we have an optimization algorithm that for some class of (group,model) pairs $\cC$ can find a certificate of sub-optimality $(g,h) \in \cC$ whenever one exists. Paired with our update algorithm, we obtain an algorithm which quickly converges to an $(\epsilon,\cC)$-Bayes Optimal model. We give such an algorithm in \ifdraft Section \ref{sec:restricted} \else Appendix \ref{app:algopt}\fi.

Suppose alternately that we open the search for certificates of sub-optimality to a large and motivated population: for example, to machine learning engineers, regulators and the general public, incentivized by possibly large monetary rewards. In this case, the guarantee of Theorem \ref{thm:updatebound} is that the process of iteratively opening our models up to scrutiny and updating whenever certificates of suboptimality are found cannot go on for too many rounds: at convergence, it must be \emph{either} that our deployed model is $\epsilon$-Bayes optimal, \emph{or} that if not, at least nobody can find any evidence to contradict this hypothesis. Since in general it is not possible to falsify Bayes optimality given only a polynomial amount of data and computation, this is in a strong sense the best we can hope for. We give a procedure for checking large numbers of arbitrarily complex submitted proposals for certificates of sub-optimality (e.g. that arrive as part of a bias bounty program) in Section \ref{sec:unrestricted}. There are two remaining obstacles, which we address in the next sections:

\ifdraft
\begin{enumerate}
    \item Our analysis so far is predicated on our update algorithm being given $(\Delta,\mu)$ certificates of sub-optimality $(g,h)$. But $\mu$ and $\Delta$ are defined with respect to the distribution $\cD$, and we will not have direct access to $\cD$ --- we will only have samples drawn from $\cD$. So how can we find certificates of sub-optimality and check their parameters? In an algorithmic application in which we search for certificates within a restricted class $\cC$, we can appeal to uniform convergence bounds, but the bias bounty application poses additional challenges. In this case, the certificates are not constrained to come from any fixed class, and so we cannot appeal to uniform convergence results. If we are opening up the search for certificates of sub-optimality to the public (with large monetary incentives), we also need to be prepared to handle a very large number of submissions. In Section \ref{sec:unrestricted} we show how to use techniques from adaptive data analysis to re-use a small holdout set to check a very large number of submissions, while maintaining strong worst case guarantees \citep{DFHPRR1,DFHPRR2,DFHPRR3,BH15,BNSSSU16,JLNRSS20}.
    \item Theorem \ref{thm:progress} gives us a guarantee that whenever we are given a certificate of sub-optimality $(g_{t+1},h_{t+1})$, our new model $f_{t+1}$ makes improvements both with respect to its error on $g_{t+1}$, and with respect to overall error. But it does not promise that the update does not increase error for some other previously identified group $g_{i}$ for $i \leq t$. This would be particularly undesirable in a ``bias bounty'' application, and would represent the kind of tradeoff that our framework aims to circumvent. However, we show in Section \ref{sec:monotone} that (up to small additive terms that come from statistical estimation error), our updates can be made to be groupwise monotonically error improving, in the sense that the update at time $t+1$ does not increase the error for any group $g_i$ identified at any time $i \leq t$.
\end{enumerate}
\else
(1) Our analysis so far is predicated on our update algorithm being given $(\Delta,\mu)$ certificates of sub-optimality $(g,h)$. But $\mu$ and $\Delta$ are defined with respect to the distribution $\cD$, and we will not have direct access to $\cD$ --- we will only have samples drawn from $\cD$. So how can we find certificates of sub-optimality and check their parameters? 

(2) Theorem \ref{thm:progress} gives us a guarantee that whenever we are given a certificate of sub-optimality $(g_{t+1},h_{t+1})$, our new model $f_{t+1}$ makes improvements both with respect to its error on $g_{t+1}$, and with respect to overall error. But it does not promise that the update does not increase error for some other previously identified group $g_{i}$ for $i \leq t$. This would be particularly undesirable in a ``bias bounty'' application, and would represent the kind of tradeoff that our framework aims to circumvent. However, we show in Section \ref{sec:monotone}  that (up to small additive terms that come from statistical estimation error), our updates can be made to be groupwise monotonically error improving.
\fi 

\ifdraft
\else
\vspace{-0.2cm}
\fi
\section{Obtaining Certificates of Suboptimality}
\label{sec:reusable}

In this section we show how to find and verify proposed certificates of sub-optimality $(g,h)$ given only a finite sample of data $D \sim \cD^n$. We consider two important cases:
\ifdraft 

\else
\fi
In Section \ref{sec:unrestricted}, we consider the ``bias bounty'' application in which the discovery of certificates of sub-optimality is crowd-sourced (aided perhaps with API access to the model $f_t$ and a training dataset). In this case, we face two main difficulties:
\ifdraft
\begin{enumerate}
    \item The submitted certificates $(g,h)$ might be arbitrary (and in particular, not guaranteed to come from a class of bounded complexity or expressivity), and
    \item We  expect to receive a very large number of submitted certificates, all of which need to be checked.
\end{enumerate}
\else
(1) The submitted certificates $(g,h)$ might be arbitrary (and in particular, not guaranteed to come from a class of bounded complexity or expressivity and (2) We  expect to receive a very large number of submitted certificates, all of which need to be checked.
\fi 
The first of these difficulties means that we cannot appeal to uniform convergence arguments to obtain rigorous bounds on the sub-optimality parameters $\mu$ and $\Delta$. The second of these difficulties means that we cannot naively rely on estimates from a (single, re-used) holdout set to obtain rigorous bounds on $\mu$ and $\Delta$.
\ifdraft

\else \fi
In Section \ref{sec:restricted} we consider the algorithmic application in which the discovery of certificates is treated as an optimization problem over $\cC$, for particular classes $\cC$. In this case we give two algorithms for finding $(\epsilon,\cC)$-Bayes optimal models via efficient reductions to cost sensitive classification problems over an appropriately defined class, solved over a polynomially sized dataset sampled from the underlying distribution\ifdraft \else, which we discuss in detail in Appendix \ref{app:algopt}\fi.
\ifdraft 
\else
\vspace{-0.2cm}
\fi 
\subsection{Unrestricted Certificates and Bias Bounties}
\label{sec:unrestricted}

In this section we develop a procedure to re-use a holdout set to check the validity of a very large number of proposed certificates of sub-optimality $(g_i,h_i)$ with rigorous guarantees. Here we make no assumptions at all about the structure or complexity of either the groups $g_i$ or models $h_i$, or the process by which they are generated. This allows us the flexibility to model e.g. a public bias bounty, in which a large collection of human competitors use arbitrary methods to find and propose certificates of sub-optimality, potentially adaptively as a function of all of the information that they have available. We use simple description length techniques developed in the adaptive data analysis literature \citep{DFHPRR2,BH15}. Somewhat more sophisticated techniques which employ noise addition \citep{DFHPRR1,DFHPRR3,BNSSSU16,JLNRSS20} could also be directly used here to improve the sample complexity bound in Theorems \ref{thm:adaptive} and \ref{thm:bounty-SC} by a $\sqrt{1/\epsilon}$ factor, but we elide this for clarity of exposition. First we give a simple algorithm (Algorithm \ref{alg:checker}) that takes as input a stream of arbitrary adaptively chosen triples $(f_i,g_i,h_i)$, and checks if each $(g_i,h_i)$ form a certificate of sub-optimality for $f_i$. We then use this as a sub-routine in Algorithm \ref{alg:FalsifyAndUpdate} which maintains a sequence of models $f_t$ produced by ListUpdate (Algorithm \ref{alg:basic}) and takes as input a sequence of proposed certificates $(g_i,h_i)$ which claim to be certificates of sub-optimality for the current model $f_t$: it updates the current model whenever such a proposed certificate is verified.

\begin{algorithm}
\ifdraft
\else 
\small
\fi
\KwInput{Holdout dataset $D$, Target $\epsilon$, and a stream of submissions  $(f_1,g_1,h_1),(f_2,g_2,h_2),\ldots$ of length at most $U$}
$\mathbf{NumberAccepted} \leftarrow 0$

\While{$\mathbf{NumberAccepted} \leq 2/\epsilon$ and $i \leq U$}{
  Consider the next submission $(f_i,g_i,h_i)$ \\

  Compute $\mu_i \leftarrow \mu_D(g_i)$, $\Delta_i \leftarrow L(D,f_i,g_i) - L(D,h_i,g_i)$

  \If{$\mu_i\cdot \Delta_i < \frac{3\epsilon}{4}$} {
    \KwOut{$\pi_i = \bot$ (Submission Rejected)}
  }
  \Else{
    \KwOut{$\pi_i = \top$ (Submission Accepted)}
    $\mathbf{NumberAccepted} \leftarrow \mathbf{NumberAccepted}$+1.
  }
}
\caption{CertificateChecker($\epsilon,D,U,(f_1,g_1,h_1),\ldots )$: An algorithm that takes as input a stream of submissions $(f_i,g_i,h_i)$ and checks if $(g_i,h_i)$ is a certificate of sub-optimality for $f_i$ with sufficiently high parameter values.}
\label{alg:checker}
\end{algorithm}

\begin{restatable}{theorem}{certchecker}
\label{thm:adaptive}
Let $\cD \in \Delta(\cX \times \cY)$ be any distribution over labelled examples, and let $D \sim \cD^n$ be a holdout dataset consisting of $n$ i.i.d. samples from $\cD$. Suppose:
\ifdraft
$$n \geq \frac{65 \ln(2U/\delta')}{\epsilon^3}$$
\else $n \geq \frac{65 \ln(2U/\delta')}{\epsilon^3}$. \fi
Let $\pi$ be the output stream generated by CertificateChecker$(\epsilon,D,(f_1,g_1,h_1),\ldots)$ (Algorithm \ref{alg:checker}).
Then for any possibly adaptive process generating a stream of up to $U$ submissions $(f_i,g_i,h_i)$ as a function of the output stream $\pi \in \{\bot,\top\}^*$, with probability $1-\delta$ over the randomness of $D$:
\begin{enumerate}
    \item For every round $i$ such that $\pi_i = \bot$ (the submission is rejected), we have that $(g_i,h_i)$ is not a $(\mu,\Delta)$ certificate of sub-optimality for $f_i$ for any $(\mu,\Delta)$ with $\mu\cdot \Delta \geq \epsilon$. And:
    \item For every round $i$ such that $\pi_i = \top$ (the submission is accepted), we have that $(g_i,h_i)$ is a $(\mu,\Delta)$-certificate of sub-optimality for $f_i$ for $\mu\cdot\Delta \geq \epsilon/2$.
\end{enumerate}
\end{restatable}
\ifdraft

The high level idea of the proof is as follows: For any \emph{fixed} (i.e. non-adaptively chosen) sequence of submissions, a Chernoff bound and a union bound are enough to argue that the estimate of the product of the parameters $\mu_i$ and $\Delta_i$ on the holdout set is with high probability close to their expected value on the underlying distribution. We then observe that submissions depend on the holdout set only through the transcript $\pi$, and so are able to union bound over all possible transcripts $\pi$. Since $\pi$ contains only $2/\epsilon$ instances in which the submission is accepted, the number of such transcripts grows only polynomially in $U$ rather than exponentially, and so we can union bound over all transcripts with only a logarithmic dependence in $U$.  

\begin{proof}
We first consider any \emph{fixed} triple of functions $f_i:\cX\rightarrow \cY, h_i:\cX\rightarrow,\cY, g_i:\cX\rightarrow \{0,1\}$. Observe that we can write:
\begin{eqnarray*}
\mu_i\cdot \Delta_i &=& \mu_D(g_i)\cdot \left(
 L(D,f_i,g_i) - L(D,h_i,g_i) \right) \\
 &=& \frac{\sum_{(x,y)\in D}\mathbbm{1}[g_i(x) = 1]}{n} \cdot \frac{\sum_{(x,y) \in D} \mathbbm{1}[g_i(x) = 1]\cdot \left(\ell(f_i(x),y) - \ell(h_i(x),y)\right)}{\sum_{(x,y) \in D} \mathbbm{1}[g_i(x) = 1]} \\
 &=& \frac{1}{n} \sum_{(x,y) \in D} \mathbbm{1}[g_i(x) = 1]\cdot \left(\ell(f_i(x),y) - \ell(h_i(x),y)\right)
\end{eqnarray*}
Since each $(x,y) \in D$ is drawn independently from $\cD$, each term in the sum $\mathbbm{1}[g_i(x) = 1]\cdot \left(\ell(f_i(x),y) - \ell(h_i(x),y)\right)$ is an independent random variable taking value in the range $[-1,1]$. Thus $\mu_i\cdot \Delta_i$ is the average of $n$ independent bounded random variables and we can apply a Chernoff bound to conclude that for any value of $\delta' > 0$:
$$\Pr_{D \sim \cD^n}\left[\left| \mu_i\cdot \Delta_i -\mu_\cD(g_i)\cdot \left(
 L(\cD,f_i,g_i) - L(\cD,h_i,g_i)\right) \right|  \geq \sqrt{\frac{2\ln(2/\delta')}{n}}\right] \leq \delta'$$
 Solving for $n$ we have that with probability $1-\delta'$, we have $\left| \mu_i\cdot \Delta_i -\mu_\cD(g_i)\cdot \left(
 L(\cD,f_i,g_i) - L(\cD,h_i,g_i)\right) \right| \leq \frac{\epsilon}{4}$ if:
 $$n \geq \frac{32 \ln(2/\delta')}{\epsilon^2}$$
 This analysis was for a fixed triple of functions $(f_i,h_i,g_i)$, but these triples  can be chosen arbitrarily as a function of the transcript $\pi$. We therefore need to count how many transcripts $\pi$ might arise. By construction, $\pi$ has length at most $U$ and has at most $2/\epsilon$ indices such that $\pi_i = \top$. Thus the number of transcripts $\pi$ that can arise is at most: ${U \choose \frac{2}{\epsilon}}\cdot 2^{2/\epsilon} \leq U^{2/\epsilon}$, and each transcript results in some sequence of $U$ triples $(f_i,h_i,g_i)$. Thus for any mechanism for generating triples from transcript prefixes, there are at most $U^{2/\epsilon + 1}$ triples that can ever arise. We can complete the proof by union bounding over this set. Taking $\delta' = \frac{\delta}{U^{2/\epsilon + 1}}$ and plugging into our Chernoff bound above, we obtain that with probability $1-\delta$ over the choice of $D$, for any method of generating a sequence of $U$ triples $\{(f_i,h_i,g_i)\}_{i=1}^U$ from transcripts $\pi$, we have that:
 $\max_i\left| \mu_i\cdot \Delta_i -\mu_\cD(g_i)\cdot \left(
 L(\cD,f_i,g_i) - L(\cD,h_i,g_i)\right) \right| \leq \frac{\epsilon}{4}$ so long as:
  $$n \geq \frac{32(\frac{2}{\epsilon} + 1) \ln(2U/\delta')}{\epsilon^2} \geq \frac{65 \ln(2U/\delta')}{\epsilon^3}$$

 Finally, note that whenever this event obtains, the conclusions of the theorem hold, because we have that $\pi_i = \top$ exactly when $\mu_i\cdot \Delta_i \geq \frac{3\epsilon}{4}$. In this case, $\mu_\cD(g_i)\cdot \left(
 L(\cD,f_i,g_i) - L(\cD,h_i,g_i)\right) \geq \frac{3\epsilon}{4} - \frac{\epsilon}{4} = \frac{\epsilon}{2}$ as desired. Similarly, whenever $\pi_i = \bot$, we have that $\mu_\cD(g_i)\cdot \left(
 L(\cD,f_i,g_i) - L(\cD,h_i,g_i)\right) \leq \frac{3\epsilon}{4} + \frac{\epsilon}{4} = \epsilon$ as desired.
\end{proof}
\else
The proof can be found in the Appendix.
\fi

We conclude this section by showing that we can use CertificateChecker (Algorithm \ref{alg:checker}) to run an algorithm FalsifyAndUpdate (Algorithm \ref{alg:FalsifyAndUpdate}) which\ifdraft:
\begin{enumerate}
    \item Persistently maintains a current model $f_t$ publicly, and elicits a stream of submissions $(g_i,h_i)$ attempting to falsify the hypothesis that the current model $f_t$ is approximately Bayes optimal,
    \item With high probability does not reject any submissions that falsify the assertion that $f_t$ is $\epsilon$-Bayes optimal,
    \item With high probability does not accept any submissions that do not falsify the assertion that $f_t$ is $\epsilon/2$-Bayes optimal,
    \item Whenever it accepts a submission $(g_i,h_i)$, it updates the current model $f_t$ and outputs a new model $f_{t+1}$ such that $L(\cD,f_{t+1}) \leq L(\cD,f_t) - \frac{\epsilon}{2}$ and such that $(g_i,h_i)$ no longer falsifies the sub-optimality of $f_{t+1}$, and
    \item With high probability does not halt until receiving $U$ submissions.
\end{enumerate}
\else \ persistently maintains a current model, and is able to correctly accept and reject proposed certificates of sub-optimality. \fi

\begin{algorithm}
\ifdraft
\else 
\small
\fi
\KwInput{An initial model $f_0$, a holdout dataset $D$, Target $\epsilon$, and a stream of submissions  $(g_1,h_1),(g_2,h_2),\ldots$ of length at most $U$}
Let $t \leftarrow 0$ \\

Initialize an instance of CertificateChecker($\epsilon, D, U, \ldots$)

\While{CertificateChecker has not halted}{
  Consider the next submission $(g_i,h_i)$ \\

  Feed the triple $(f_t,g_i,h_i)$ to CertificateChecker and receive $\pi_i\in \{\bot,\top\}$ \\

  \If{$\pi_i = \bot$}{
    \KwOut{Submission $(g_i,h_i)$ is rejected.}
  }
  \Else{
   Let $t \leftarrow t+1$ and let $f_t = \textrm{ListUpdate}(f_{t-1},(g_i,h_i))$ \\

   \KwOut{Submission $(g_i,h_i)$ is accepted. The new model is $f_{t}$.}
  }

}
\caption{FalsifyAndUpdate($\epsilon,D,(f_1,g_1,h_1),\ldots )$: An algorithm that maintains a sequence of models $f_1,\ldots,f_T$ and accepts submissions of proposed certificates of sub-optimality $(g_i,h_i)$ that attempt to falsify the assertion that the current model $f_t$ is not $\epsilon$-Bayes optimal. It either accepts or rejects each submission: if it accepts a submission then it also updates its current model $f_t$ and outputs a new model $f_{t+1}$.}
\label{alg:FalsifyAndUpdate}
\end{algorithm}

\begin{restatable}{theorem}{bountySC}
\label{thm:bounty-SC}
Fix any $\epsilon,\delta > 0$. Let $\cD \in \Delta(\cX \times \cY)$ be any distribution over labelled examples, and let $D \sim \cD^n$ be a holdout dataset consisting of $n$ i.i.d. samples from $\cD$. Suppose:
\ifdraft
$$n \geq \frac{65 \ln(2U/\delta')}{\epsilon^3}.$$
\else $n \geq \frac{65 \ln(2U/\delta')}{\epsilon^3}.$ \fi
Then for any (possibly adaptive) process generating a sequence of at most $U$ submissions $\{(g_i,h_i)\}_{i=1}^U$, with probability at least $1-\delta$, we have that FalsifyAndUpdate$(\epsilon,D,\cdots)$ satisfies:
\begin{enumerate}
    \item If $(g_i,h_i)$ is rejected, then $(g_i,h_i)$ is not a $(\mu,\Delta)$-certificate of sub-optimality for $f_t$, where $f_t$ is the current model at the time of submission $i$, for any $\mu, \Delta$ such that $\mu\cdot \Delta \geq \epsilon$.
    \item If $(g_i,h_i)$ is accepted, then $(g_i,h_i)$ is a $(\mu,\Delta)$-certificate of sub-optimality for $f_t$, where $f_t$ is the current model at the time of submission $i$, for some $\mu, \Delta$ such that $\mu\cdot \Delta \geq \frac{\epsilon}{2}$. Moreover, the new model $f_{t+1}$ output satisfies $L(\cD,f_{t+1}, g_i) = L(\cD, h_i, g_i)$ and $L(\cD,f_{t+1}) \leq L(\cD,f_t) - \frac{\epsilon}{2}$.
    \item FalsifyAndUpdate does not halt before receiving all $U$ submissions.
\end{enumerate}
\end{restatable}
\ifdraft

At a high level, this proof reduces to the guarantees of CertificateChecker and ListUpdate. Note that the models $f_t$ produced in the run of this algorithm depend on the holdout set \emph{only} through the transcript $\pi$ produced by certificate checker --- i.e. given the stream of submissions and the output of CertificateChecker, one can reproduce the decision lists $f_t$ output by FalsifyAndUpdate. Thus we inherit the sample complexity bounds proven for CertificateChecker. 

\begin{proof}
  This theorem follows straightforwardly from the properties of Algorithm \ref{alg:basic} and Algorithm \ref{alg:checker}. From Theorem \ref{thm:adaptive}, we have that with probability $1-\delta$, every submission accepted by CertificateChecker (and hence by FalsifyandUpdate) is a $(\mu,\Delta)$-certificate of sub-optimality for $f_t$ with $\mu\cdot \Delta \geq \epsilon/2$ and every submission rejected is not a $(\mu,\Delta)$-certificate of sub-optimality for any $\mu, \Delta$ with $\mu\cdot \Delta \geq \epsilon$.

  Whenever this event obtains, then for every call that FalsifyAndUpdate makes to $\textrm{ListUpdate}(f_{t-1},(g_i,h_i))$ is such that $(g_i,h_i)$ is a $(\mu,\Delta)$-certificate of sub-optimality for $f_{t-1}$ for $\mu\cdot\Delta \geq \epsilon/2$. Therefore by Theorem \ref{thm:progress}, we have that $L(\cD,f_{t+1}, g_i) = L(\cD, h_i, g_i)$ and $L(\cD,f_{t+1}) \leq L(\cD,f_t) - \frac{\epsilon}{2}$. Finally, by Theorem \ref{thm:updatebound}, if each invocation of  the iteration $f_t = \textrm{ListUpdate}(f_{t-1},(g_i,h_i))$ is such that $(g_i,h_i)$ is a $(\mu,\Delta)$-certificate of sub-optimality for $f_{t-1}$ with $\mu\cdot \Delta \geq \epsilon/2$, then there can be at most $2/\epsilon$ such invocations. Since FalsifyAndUpdate makes one such invocation for every submission that is accepted, this means there can be at most $2/\epsilon$ submissions accepted in total. But CertificateChecker has only two halting conditions: it halts when either more than $2/\epsilon$ submissions are accepted, or when $U$ submissions have been made in total. Because with probability at least $1-\delta$ no more than $2/\epsilon$ submissions are accepted, it must be that with probability $1-\delta$, FalsifyAndUpdate does not halt until all $U$ submissions have been received.
\end{proof}
\else
The proof can be found in the Appendix.
\fi

\begin{remark}
Note that FalsifyAndUpdate has sample complexity scaling only logarithmically with the total number of submissions $U$ that we can accept, and no dependence on the complexity of the submissions. This means that a relatively small holdout dataset $D$ is sufficient to run an extremely long-running bias bounty program (i.e. handling a number of submissions that is exponentially large in the size of the holdout set) that automatically updates the current model whenever submissions are accepted and bounties are awarded.
\end{remark}

\subsubsection{Guaranteeing Groupwise Monotone Improvements}
\label{sec:monotone}
\ifdraft FalsifyAndUpdate (Algorithm \ref{alg:FalsifyAndUpdate}) has the property that whenever it accepts a submission $(g,h)$ falsifying that its current model $f_t$ is $\epsilon/2$-Bayes optimal, it produces a new model $f_{t+1}$ that has overall loss that has decreased by at least $\epsilon/2$. It also promises that $f_{t+1}$ has strictly lower loss than $f_t$ on group $g$. However, because the groups $g_i$ can have arbitrary intersections, this does \emph{not} imply that $f_{t+1}$ has error that is lower than that of $f_t$ on groups that were previously identified. \else
 Because the groups $g_i$ can have arbitrary intersections, Algorithm \ref{alg:FalsifyAndUpdate} does \emph{not} guarantee that $f_{t+1}$ has error that is lower than that of $f_t$ on all groups that were previously identified. \fi
Specifically, let $\cG(f_t)$ denote the set of at most $2/\epsilon$ groups $g_i$ that make up the internal nodes of decision list $f_t$ --- i.e. the set of groups corresponding to submissions $(g_i,h_i)$ that were previously accepted and incorporated into model $f_t$.  It might be that for some $g_j \in \cG(f_t)$,  $L(\cD,f_{t+1},g_j) \gg L(\cD, f_t, g_j)$. This kind of non-monotonic behavior is extremely undesirable in the context of a bias bug bounty program, because it means that previous instances of sub-optimal behavior on a group $g_i$ that were explicitly identified and corrected for can be undone by future updates. \ifdraft Note that simply repeating the update $(g_j,h_j)$ when this occurs does not solve the problem --- this would return the performance of the model on $g_j$ to what it was \emph{at the time that it was originally introduced} --- but since the model's performance on $g_j$ might have improved in the mean time, it would not guarantee groupwise error monotonicity. \fi

There is a simple fix, however: whenever a new proposed certificate of sub-optimality $(g_i,h_i)$ for a model $f_t$ is accepted and a new model $f_{t+1}$ is generated, add the proposed certificates $(g_j,f_{\ell})$ to the front of the stream of submissions, for each pair of $g_j \in \cG(f_t)$ and $\ell \leq t$. Updates resulting from these submissions (which we call \emph{repairs}) might themselves generate new non-monotonicities, but repeating this process recursively is sufficient to guarantee approximate groupwise monotonicity --- and because we know from Theorem \ref{thm:updatebound} that the total number of updates $T$ cannot exceed $2/\epsilon$, this process never adds more than $\frac{8}{\epsilon^3}$ submissions to the existing stream, and thus affects the sample complexity bound only by low order terms. \ifdraft This is because for each of the at most $2/\epsilon$ updates, there are at most $|\cG(f_t)|^2 \leq \frac{4}{\epsilon^2}$ proposed certificates that can be generated in this way. Moreover, if any of these constructed submissions trigger a model update, these updates too count towards the limit of $2/\epsilon$ updates that can \emph{ever} occur --- and so do not increase the maximum length of the decision list that is ultimately output.\fi The procedure, which we call MonotoneFalsifyAndUpdate, is described as Algorithm \ref{alg:MonFalsifyAndUpdate}. Here we state its guarantee:

\begin{algorithm}
\ifdraft\else 
\small
\fi
\KwInput{A model $f_0$, a holdout dataset $D$, Target $\epsilon$, and a stream of submissions  $(g_1,h_1),(g_2,h_2),\ldots$ of length at most $U$}
Let $t \leftarrow 0$ and Initialize an instance of CertificateChecker($\epsilon, D, U + \frac{8}{\epsilon^3}, \ldots$)

\While{CertificateChecker has not halted}{
  Consider the next submission $(g_i,h_i)$ \\

  Feed the triple $(f_t,g_i,h_i)$ to CertificateChecker and receive $\pi_i\in \{\bot,\top\}$ \\

  \If{$\pi_i = \bot$}{
    \KwOut{Submission $(g_i,h_i)$ is rejected.}
  }
  \Else{
    MonotoneProgress $\leftarrow$ FALSE \\
    $(g_U,h_U) \leftarrow (g_i,h_i)$ \\
    $t' \leftarrow t$ and $f'_{t'} \leftarrow f_t$ \\
  \While{MonotoneProgress $=$ FALSE and CertificateChecker has not halted}{
    Let $t' \leftarrow t'+1$ and let $f'_{t'} = \textrm{ListUpdate}(f_{t'-1}',(g_U,h_U))$ \\
    MonotoneProgress $\leftarrow$ TRUE

    \For{each pair $\ell < t$, $g_j \in
    \cG(f_t)$}{
      Feed the triple $(f_t,g_j,f_{\ell})$ to CertificateChecker and receive $\pi_U \in \{\bot,\top\}$ \\
      \If{$\pi_U = \top$ (Submission accepted)} {
        MonotoneProgress $\leftarrow$ FALSE \\
        $(g_U,h_U) \leftarrow (g_j, f_\ell)$
      }

    }

   }
   Let $t \leftarrow t+1$ and let $f_t = f'_{t'}$ \\

   \KwOut{Submission $(g_i,h_i)$ is accepted. The new model is $f_{t}$.}
  }

}
\caption{MonotoneFalsifyAndUpdate($\epsilon,D,(f_1,g_1,h_1),\ldots )$: A version of FalsifyAndUpdate that guarantees approximate group-wise error monotonicity.}
\label{alg:MonFalsifyAndUpdate}
\end{algorithm}
\ifdraft
\else
\fi 
\begin{restatable}{theorem}{bountymonotone}
\label{thm:bounty-monotone-SC}
Fix any $\epsilon,\delta > 0$. Let $\cD \in \Delta(\cX \times \cY)$ be any distribution over labelled examples, and let $D \sim \cD^n$ be a holdout dataset consisting of $n$ i.i.d. samples from $\cD$. Suppose: \ifdraft
$$n \geq \frac{65 \ln\left(\frac{2(U+\frac{8}{\epsilon^3})}{\delta'}\right)}{\epsilon^3}.$$
\else $n \geq \frac{65 \ln\left(\frac{2(U+\frac{8}{\epsilon^3})}{\delta'}\right)}{\epsilon^3}.$\fi
Then for any (possibly adaptive) process generating a sequence of at most $U$ submissions $\{(g_i,h_i)\}_{i=1}^U$, with probability at least $1-\delta$, we have that MonotoneFalsifyAndUpdate$(\epsilon,D,\cdots)$ satisfies all of the properties proven in Theorem \ref{thm:bounty-SC} for FalsifyAndUpdate, and additionally satisfies the following error monotonicity property. Consider any model $f_t$ that is output, and any group $g_j \in \cG(f_t)$. Then:
\ifdraft
$$L(\cD,f_t,g_j) \leq \min_{\ell < t} L(\cD, f_{\ell}, g_j) + \frac{\epsilon}{\mu_\cD(g_j)} $$
\else $L(\cD,f_t,g_j) \leq \min_{\ell < t} L(\cD, f_{\ell}, g_j) + \frac{\epsilon}{\mu_\cD(g_j)}$. \fi
\end{restatable}
\ifdraft
\begin{proof}
The proof that MonotoneFalsifyAndUpdate satisfies the first two conclusions of Theorem \ref{thm:bounty-SC}:
\begin{enumerate}
  \item If $(g_i,h_i)$ is rejected, then $(g_i,h_i)$ is not a $(\mu,\Delta)$-certificate of sub-optimality for $f_t$, where $f_t$ is the current model at the time of submission $i$, for any $\mu, \Delta$ such that $\mu\cdot \Delta \geq \epsilon$.
    \item If $(g_i,h_i)$ is accepted, then $(g_i,h_i)$ is a $(\mu,\Delta)$-certificate of sub-optimality for $f_t$, where $f_t$ is the current model at the time of submission $i$, for some $\mu, \Delta$ such that $\mu\cdot \Delta \geq \frac{\epsilon}{2}$. Moreover, the new model $f_{t+1}$ output satisfies  $L(\cD,f_{t+1}) \leq L(\cD,f_t) - \frac{\epsilon}{2}$.
\end{enumerate}
are identical and we do not repeat them here. We must show that with probability $1-\delta$, CertificateChecker (and hence MonotoneFalsifyAndUpdate) does not halt before processing all $U$ submissions. Note that MonotoneFalsifyAndUpdate initializes an instance of CertificateChecker that will not halt before receiving $U + \frac{8}{\epsilon^3}$ many submissions. Thus it remains to verify that our algorithm does not produce more than $8/\epsilon^3$ many submissions to CertificateChecker in its monotonicity update process. But this will be the case, because by Theorem \ref{thm:updatebound}, $t \leq \frac{2}{\epsilon}$, and so after each call to ListUpdate, we generate at most $4/\epsilon^2$ many submissions to CertificateChecker. Since there can be at most $2/\epsilon$ such calls to ListUpdate, the claim follows.

To see that the monotonicity property holds, assume for sake of contradiction that it does not --- i.e. that there is a model $f_t$, a group $g_j \in \cG(f_t)$, and a model $f_\ell$ with $\ell < t$ such that:
$$L(\cD,f_t,g_j) >  L(\cD, f_{\ell}, g_j) + \frac{\epsilon}{\mu_\cD(g_j)}$$
In this case, the pair $(g_j,f_{\ell})$ would form a $(\mu,\Delta)$-certificate of sub-optimality for $f_t$ with $\mu\cdot \Delta \geq \epsilon$. But if $L(\cD,f_t,g_j) >  L(\cD, f_{\ell}, g_j) + \frac{\epsilon}{\mu_\cD(g_j)}$, then this certificate must have been rejected, which we have already established is an event that occurs with probability at most $\delta$.
\end{proof}
\else
The proof can be found in the appendix.
\fi

\begin{figure}
    \centering
    \ifdraft
    \includegraphics[scale=0.35]{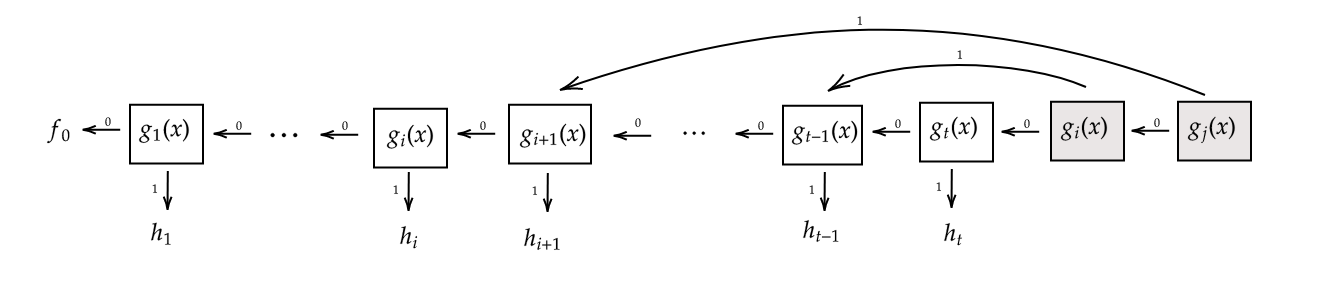}
    \else
    \includegraphics[scale=0.18]{pdl.png}
    \fi
    \caption{A pointer decision list. The shaded nodes are monotonicity repairs caused by update $g_t$. In this example, there were two monotonicity repairs, corresponding to certificates of sub-optimality $(g_i,f_{t-1})$ and $(g_j,f_{i+1})$. Rather than replicating the models $f_{t-1}$ and $f_{i+1}$, the updates are implemented using pointers to the prefix of the decision list representing models $f_{t-1}$ and $f_{i-1}$. Theorem \ref{thm:bounty-monotone-SC} guarantees that the entire length of the list (including nodes that implement back-pointers) cannot grow beyond $2/\epsilon$.}
    \label{fig:pdl}
\end{figure}
Algorithm \ref{alg:MonFalsifyAndUpdate} introduces updates of the form $(g_j,f_{\ell})$, where $f_{\ell}$ is a decision list previously generated by ListUpdate. We might worry that these updates could produce a very large model, since on such an update, the new model $f_\ell$ at the first leaf of the new decision list entirely replicates some previous decision list $f_{\ell}$. However, note that for any such update, $f_{\ell}$ is a \emph{prefix} of $f_t$. Therefore, rather than replicating $f_{\ell}$, we can  introduce a backpointer to level $\ell$ of our existing decision list, without increasing our model's size. We call such a model a \emph{pointer decision list}, as shown in Figure \ref{fig:pdl}. We call the pointer nodes that are introduced to repair non-monotonicities introduced by previous updates \emph{repair nodes} of our pointer decision list. These are the nodes that are shaded in our illustration in Figure \ref{fig:pdl}. 
\ifdraft
\else
\vspace{-0.2cm}
\fi

\subsection{Certificates of Bounded Complexity and Algorithmic Optimization}
\label{sec:restricted}

\ifdraft
In this section we show how to use our ListUpdate method as part of an algorithm for explicitly computing $(\epsilon,\cC)$-Bayes optimal models from data sampled i.i.d. from the underlying distribution $\cD$. First, we must show that if we find certificates of sub-optimality $(g,h) \in C$ on our dataset, that we can be assured that they are certificates of sub-optimality on the underlying distribution. Here, we invoke uniform convergence bounds that depend on $\cC$ being a class of bounded complexity. Next, we must describe an algorithmic method for finding $(\mu,\Delta)$ certificates of sub-optimality that maximize $\mu\cdot \Delta$. Here we give two approaches. The first approach is a reduction to cost sensitive (ternary) classification: the result of the reduction is that the ability to solve weighted multi-class classification problems over some class of models gives us the ability to find certificates of sub-optimality over a related class whenever they exist. The second approach takes an ``EM'' style alternating maximization approach over $g_i \in \cG$ and $h_i \in \cH$ in turn, where each alternating maximization step can be reduced to a binary classification problem. It is only guaranteed to converge to a local optimum (or saddlepoint) of its objective \ifdraft --- i.e. to find a certificate of sub-optimality $(g_i,h_i)$ that cannot be improved by changing either $g_i$ or $h_i$ unilaterally --- \fi but has the merit that it requires only standard binary classification algorithms for a class $\cG$ and $\cH$ to search for certificates of sub-optimality in $\cG\times \cH$. For simplicity of exposition, in this section we restrict attention to the binary classification problem, where the labels are binary $(\cY = \{0,1\}$) and our loss function corresponds to classification error $(\ell(\hat y, y) = \mathbbm{1}[\hat y \neq y])$ --- but the approach readily extends to more general label sets (replacing VC-dimension with the appropriate notion of combinatorial dimension as necessary).

The algorithmic problem we need to solve at each round $t$ is, given an existing model $f_{t-1}$, find a $(\mu,\Delta)$-certificate of optimality $(g,h) \in \cC$ for $f_{t-1}$ that maximizes $\mu\cdot \Delta$ as computed on the empirical data $D$ --- i.e. to solve:
\begin{equation}
\label{eq:optproblem}
(g_t,h_t) =\arg\max_{(g,h) \in \cC }\mu_{D}(g)\cdot (L(D,f_{t-1},g) - L(D,h,g))
\end{equation}

We defer for now the algorithmic problem of finding these certificates, and describe a generic algorithm that can be invoked with any method for finding such certificates. \ifdraft First, we state a useful sample complexity bound proven in \citep{gerrymandering} in a related context of multi-group fairness --- we here state the adaptation to our setting.
\begin{lemma}[Adapted from \citep{gerrymandering}]
\label{lem:samplecomplexity}
Let $\cG$ denote a class of group indicator functions with VC-dimension $d_G$ and $\cH$ denote a class of binary models with VC-dimension $d_H$.  Let $f$ be an arbitrary binary model. Then if:
 $$n \geq \tilde O\left(\frac{(d_H + d_G) + \log(1/\delta)}{\eta^2} \right)$$
 we have that with probability $1-\delta$ over the draw of a dataset $D \sim \cD^n$, for every $g \in \cG$ and $h \in \cH$:
 $$\left|\mu_{\cD}(g)\cdot (L(\cD,f,g) - L(\cD,h,g)) - \mu_{D}(g_p)\cdot (L(D,f,g) - L(D,h,g)) \right| \leq \eta$$
\end{lemma}

With our sample complexity lemma in hand we are ready to present our generic reduction from training an $(\epsilon, \cC)$-Bayes optimal model to the optimization problem over $\cC$ given in \eqref{eq:optproblem}.
\else
It is a reduction from the problem of training an $(\epsilon,\cC)$-Bayes optimal model to the optimization problem over $\cC$ given in \eqref{eq:optproblem}.
\fi

\begin{algorithm}
\ifdraft
\else
\small
\fi
\KwInput{A dataset $D$, a set of group/model pairs $\cC$, and accuracy parameter $\epsilon$.}
Let $f_0:\cX\rightarrow \{0,1\}$ be an arbitrary initial model and $T = 2/\epsilon$ \\

Randomly divide $D$ into $T$ equally sized datasets: $D_1,\ldots,D_T$.

\For{$t = 1$ to $T$}{
Let
$$(g_t,h_t) =\arg\max_{(g,h) \in \cC }\mu_{D}(g)\cdot (L(D,f_{t-1},g) - L(D,h,g))$$

\If{$\mu_{D}(g_t)\cdot (L(D,f_{t-1},g_t) - L(D,h_t,g_t)) \leq \frac{3\epsilon}{4}$}{
\KwOut{Model $f_{t-1}$}
}
\Else{
Let $f_t = \textrm{ListUpdate}(f_{t-1},(g_t,h_t))$
}
}
\KwOut{Model $f_T$}

\caption{TrainByOpt($D,\cC,\epsilon)$: An algorithm for training an $(\epsilon,\cC)$-Bayes optimal model given the ability to optimize over certificates $(g,h) \in \cC$.}
\label{alg:TrainByOpt}
\end{algorithm}
\ifdraft
\else
\vspace{-0.5cm}
\fi
\begin{restatable}{theorem}{samplecomplexity}
Fix an arbitrary distribution $\cD$ over $\cX\times \cY$, a class of group indicator functions $\cG$ of VC-dimension $d_G$ and a class of binary models $\cH$ of VC-dimension $d_H$. Let $\cC \subseteq \cG \times \cH$ and $\epsilon > 0$ be arbitrary.  If: \ifdraft
$$n \geq \tilde O\left(\frac{(d_H + d_G) + \log(1/\delta)}{\epsilon^3} \right)$$
\else $n \geq \tilde O\left(\frac{(d_H + d_G) + \log(1/\delta)}{\epsilon^3} \right)$ \fi
and $D \sim \cD^n$,
then with probability $1-\delta$,  TrainByOpt$(D,\cC,\epsilon)$ (Algorithm \ref{alg:TrainByOpt}) returns a model $f$ that is $(\epsilon,\cC)$-Bayes optimal, using at most $2/\epsilon$ calls to a sub-routine for solving the optimization problem over $(g,h) \in \cC$ given in \eqref{eq:optproblem}.
\end{restatable}

\ifdraft
\begin{proof}
  Each of the partitioned datasets $D_t$ has size at least $\tilde O\left(\frac{(d_H + d_G) + \log(1/\delta)}{\epsilon^2} \right)$ and is selected independently of $f_{t-1}$ and so we can invoke Lemma \ref{lem:samplecomplexity} with $\eta = \epsilon/8$ and $\delta = \delta/T$ to conclude that with probability $1-\delta$, for every round $t$:
  $$\left|\mu_{D}(g_t)\cdot (L(D,f_{t-1},g_t) - L(D,h_{t},g_t)) -\mu_{\cD}(g_t)\cdot (L(\cD,f_{t-1},g_t) - L(\cD,h_{t},g_t))\right| \leq \frac{\epsilon}{8}$$
  For the rest of the argument we will assume this condition obtains.
 By assumption, at every round $t$ the models $(g_t,h_t)$ exactly maximize $\mu_{D}(g_t)\cdot (L(D,f_{t-1},g_t) - L(D,h_{t},g_t))$ amongst all models $(g,h) \in \cC$, and so in combination with the above uniform convergence bound, they are $\epsilon/4$-approximate maximizers of $\mu_{\cD}(g_t)\cdot (L(\cD,f_{t-1},g_t) - L(\cD,h_{t},g_t))$. Therefore, if at any round $t \leq T$ the algorithm outputs a model $f_{t-1}$, it must be because:
 $$\max_{(g,h) \in \cC}\mu_{\cD}(g)\cdot (L(\cD,f_{t-1},g) - L(\cD,h,g)) \leq \frac{3\epsilon}{4} + \frac{\epsilon}{4} = \epsilon$$
 Equivalently, for every $(g,h) \in \cC$, it must be that:
 $$L(\cD,f_{t-1},g) \leq L(\cD,h,g) + \frac{\epsilon}{\mu_{\cD}(g)}$$
 By definition, such a model is $(\epsilon,\cC)$-Bayes optimal.

Similarly, if at any round $t < T$ we do not output $f_{t-1}$, then it must be that $(g_t,h_t)$ forms a $(\mu,\Delta)$-certificate of sub-optimality for $\mu = \mu_{\cD}(g_t)$ and $\Delta = L(\cD,f_{t-1},g) - L(\cD,h,g)$ such that $\mu\cdot \Delta \geq \frac{3\epsilon}{4} - \frac{\epsilon}{4} = \frac{\epsilon}{2}$. By Theorem \ref{thm:updatebound} we have that for any sequence of models $f_0, f_1, \ldots, f_T$ such that $f_t = \textrm{ListUpdate}(f_{t-1},(g_t,h_t))$ and $(g_t,h_t)$ forms a $(\mu,\Delta)$-certificate of sub-optimalty for $f_{t-1}$ with $\mu\cdot \Delta \geq \frac{\epsilon}{2}$, it must be that $T \leq \frac{2}{\epsilon}$. Therefore, it must be that if our algorithm outputs model $f_T$, then $f_T$ must be $\frac{\epsilon}{2}$-Bayes optimal. If this were not the case, then by Theorem \ref{thm:BOQ}, there would exist a $(\mu,\Delta)$-certificate of sub-optimality $(g^*,h^*)$ for $f_T$ with $\mu\cdot \Delta \geq \frac{\epsilon}{2}$, which we could use to extend the sequence by setting $f_{T+1} = \textrm{ListUpdate}(f_T,(g^*,h^*))$ --- but that would contradict Theorem \ref{thm:updatebound}. Since $\epsilon/2$-Bayes optimality is strictly stronger than $(\epsilon,\cC)$-Bayes optimality, this completes the proof.
\end{proof}
\else
The proof can be found in the Appendix.
\fi

Algorithm \ref{alg:TrainByOpt} is an efficient algorithm for training an $(\epsilon, \cC)$-Bayes optimal model whenever we can solve optimization problem \eqref{eq:optproblem} efficiently. We now turn to this optimization problem. In Section \ref{sec:ternary} we show that the ability to solve cost sensitive classification problems over a \emph{ternary} class $K$ gives us the ability to solve optimization problem \ref{eq:optproblem} over a related class $\cC_K$. In \ifdraft Section \else Appendix \fi \ref{sec:EM} we show that the ability to solve standard empirical risk minimization problems over classes $\cG$ and $\cH$ respectively give us the ability to run iterative updates of an alternating-maximization (``EM style'') approach to finding certificates $(g,h) \in \cG \times \cH$.

\subsubsection{Finding Certificates via a Reduction to Cost-Sensitive Ternary Classification}
\label{sec:ternary}

For this approach, we start with an arbitrary class $K$ of \emph{ternary} valued functions $p:\cX\rightarrow \{0,1,?\}$. It will be instructive to think of the label ``$?$'' as representing the decision to ``defer'' on an example, leaving the classification outcome to another model. We will identify such a ternary-valued function $p$ with a pair of binary valued functions $g_p:\cX\rightarrow \{0,1\}, h_p:\cX\rightarrow \{0,1\}$ representing a group indicator function and a binary model $h$ that might form a certificate of sub-optimality $(g_p,h_p)$. They are defined as follows:
\begin{definition}
Given a ternary valued function $p:\cX\rightarrow \{0,1,?\}$, the $p$-derived group $g_p$ and model $h_p$ are defined as:
\ifdraft
$$g_p(x) = \begin{cases}
      1 & \text{if } p(x) \in \{0,1\} \\
      0 & \text{if } p(x) = ?
   \end{cases} \ \ \ \ \ \ \  \ h_p(x) = \begin{cases}
      p(x) & \text{if } p(x) \in \{0,1\} \\
      0 & \text{if } p(x) = ?
   \end{cases}$$
\else $g_p(x) = \begin{cases}
      1 & \text{if } p(x) \in \{0,1\} \\
      0 & \text{if } p(x) = ?
   \end{cases} \ \ \ \ \ \ \  \ h_p(x) = \begin{cases}
      p(x) & \text{if } p(x) \in \{0,1\} \\
      0 & \text{if } p(x) = ?
   \end{cases}$ \fi
In other words, interpreting $p(x) = ?$ as the decision for $p$ to ``defer'' on $x$, $g_p$ defines exactly the group of examples that $p$ does \emph{not} defer on, and $h_p$ is the model that makes the same prediction as $p$ on every example that $p$ does not defer on. Given a class of ternary functions $K$, let the $K$-derived certificates $\cC_K$ denote the set of pairs $(g_p,h_p)$ that can be so derived from some $p \in K$: $\cC_K = \{(g_p,h_p) : p \in K\}$. Similarly let $\cG_K = \{g_p : p \in K\}$ and $\cH_K = \{h_p : p \in K\}$ denote the class of group indicator functions and models derived from $K$ respectively.
\end{definition}

Given a model $f:\cX\rightarrow \{0,1\}$ Our goal is to reduce the problem of solving optimization problem \eqref{eq:optproblem} over $\cC_K$ to the problem of solving a ternary \emph{cost-sensitive classification problem} over $K$:
\begin{definition}
A cost-sensitive classification problem is defined by a model class $K$ consisting of functions $p:\cX\rightarrow \cY$, where $\cY$ is some finite label set,  a distribution $\cD$ over $\cX\times \cY$, and a set of real valued costs $c_{(x,y)}(\hat y)$ for each pair $(x, y)$ in the support of $\cD$ and each label $\hat y \in \cY$. A solution $p^* \in K$ to the cost-sensitive classification problem $(\cD, K, \{c_{(x,y)}\})$ is given by:
\ifdraft
$$p^* \in \arg\min_{p \in K} \E_{(x,y) \sim \cD}\left[ c_{(x,y)}(p(x))\right]$$
\else $p^* \in \arg\min_{p \in K} \E_{(x,y) \sim \cD}\left[ c_{(x,y)}(p(x))\right]$ \fi
i.e. the model that minimizes the expected costs for the labels it assigns to points drawn from $\cD$.
\end{definition}

The reduction will make use of the following induced costs:
\begin{definition}
Given a binary model $f:\cX\rightarrow \cY$, the induced costs of $f$  are defined as follows:
\ifdraft
$$c^f_{(x,y)}(\hat y) = \begin{cases}
      0 & \text{if } \hat y = ? \\
      1 & \text{if } f(x) = y \neq \hat y \\
      -1 & \text{if } \hat y = y \neq f(x) \\
      0 & \text{otherwise.}
   \end{cases}$$
\else $c^f_{(x,y)}(\hat y) = 0$ if $\hat y = ?$ or if $\hat y = f(x)$, $c^f_{(x,y)}(\hat y) = 1$ if $f(x) = y \neq \hat y$, and $c^f_{(x,y)}(\hat y) = -1$ if $\hat y = y \neq f(x)$. \fi
Intuitively, it costs nothing to defer a decision to the existing model $f$ --- or equivalently, to make the same decision as $f$. On the other hand, making the \emph{wrong} decision on an example $x$ costs $1$ when $f$ would have made the right decision, and making the right decision ``earns'' $1$ when $f$ would not have.
\end{definition}

\begin{restatable}{theorem}{lemCSC}
\label{lem:CSC}
Fix an arbitrary distribution $\cD$ over $\cX\times \cY$, let $K$ be a class of ternary valued functions, and let $f:\cX\rightarrow \{0,1\}$ be any binary valued model. Let $p^*$ be a solution to the cost-sensitive classification problem $(\cD, K, \{c_{(x,y)}^f\})$, where $c_{(x,y)}^f$ are the induced costs of $f$. We have that:
\ifdraft
$$(g_{p^*},h_{p^*}) \in \arg\max_{(g,h) \in \cC_K }\mu_{\cD}(g)\cdot (L(\cD,f,g) - L(\cD,h,g))$$
\else $(g_{p^*},h_{p^*}) \in \arg\max_{(g,h) \in \cC_K }\mu_{\cD}(g)\cdot (L(\cD,f,g) - L(\cD,h,g))$. \fi
In other words, when $\cD$ is the empirical distribution over $D$, $(g_{p^*},h_{p^*})$ form a solution to optimization problem \eqref{eq:optproblem}.
\end{restatable}
\ifdraft
\begin{proof}
 For any model $p \in K$, we can calculate its expected cost under the induced costs of $f$:
 \begin{eqnarray*}
  \E_{(x,y) \sim \cD}\left[ c_{(x,y)}(p(x))\right] &=& \E_{(x,y) \sim \cD}\left[\mathbbm{1}[p(x) \neq ?]\cdot \mathbbm{1}[p(x) \neq f(x)]\cdot (\mathbbm{1}[p(x) \neq y]-\mathbbm{1}[p(x) = y]) \right] \\
  &=&  \E_{(x,y) \sim \cD}\left[\mathbbm{1}[g_p(x) = 1]\cdot \mathbbm{1}[h_p(x) \neq f(x)](\mathbbm{1}[h_p(x) \neq y]-\mathbbm{1}[h_p(x) = y]) \right] \\
  &=& \mu_{\cD}(g_p)\cdot (L(\cD,h_p,g_p) - L(\cD,f,g_p)) \\
\end{eqnarray*}
Thus \emph{minimizing}   $\E_{(x,y) \sim \cD}\left[ c_{(x,y)}(p(x))\right]$ is equivalent to \emph{maximizing} $\mu_{\cD}(g_p)\cdot (L(\cD,f,g_p)-L(\cD,h_p,g_p))$ over $p \in K$.
\end{proof}
\else The proof can be found in the Appendix.
\fi

In other words, in order to be able to efficiently implement algorithm \ref{alg:TrainByOpt} for the class $\cC_K$, it suffices to be able to solve a ternary cost sensitive classification problem over $K$.  The up-shot is that if we can efficiently solve weighted multi-class classification problems (for which we have many algorithms which form good heuristics) over $K$, then we can find approximately $\cC_K$-Bayes optimal models as well.
\else

We can also use our ListUpdate method as part of an algorithm for explicitly computing $(\epsilon,\cC)$-Bayes optimal models from data sampled i.i.d. from the underlying distribution $\cD$. First, we must show that if we find certificates of sub-optimality $(g,h) \in C$ on our dataset, that we can be assured that they are certificates of sub-optimality on the underlying distribution. Next, we must describe an algorithmic method for finding $(\mu,\Delta)$ certificates of sub-optimality that maximize $\mu\cdot \Delta$. We provide such guarantees and algorithmic methods in Appendix \ref{app:algopt}.
\fi 

\ifdraft
\subsubsection{Finding Certificates Using Alternating Maximization}
\label{sec:EM}
The reduction from optimization problem \eqref{eq:optproblem} that we gave in Section \ref{sec:ternary} to ternary cost sensitive classification starts with a ternary class $K$ and then finds certificates of sub-optimality over a derived class $\cC_K$. What if we want to \emph{start} with a pre-defined class of group indicator functions $\cG$ and models $\cH$, and find certificates of sub-optimality $(g,h) \in \cG\times \cH$? In this section we give an alternating maximization method that attempts to solve optimization problem \ref{eq:optproblem} by alternating between maximizing over $g_t$ (holding $h_t$ fixed), and maximizing over $h_t$ (holding $g_t$ fixed):

\begin{equation}
\label{eq:gmin}
g_t=\arg\max_{g\in \cG }\mu_{D}(g)\cdot (L(D,f_{t-1},g) - L(D,h_t,g))
\end{equation}
\begin{equation}
\label{eq:hmin}
h_t =\arg\max_{h \in \cH }\mu_{D}(g_t)\cdot (L(D,f_{t-1},g_t) - L(D,h,g_t))
\end{equation}

We show that each of these alternating maximization steps can be reduced to solving a standard (unweighted) empirical risk minimization problem over $\cG$ and $\cH$ respectively. Thus, each can be solved for any heuristic for standard machine learning problems --- we do not even require support for weighted examples, as we do in the cost sensitive classification approach from Section \ref{sec:ternary}. This alternating maximization approach quickly converges to a local optimum or saddle point of the optimization objective from from \eqref{eq:optproblem} --- i.e. a solution that cannot be improved by either a unilateral change of either $g_t \in \cG$ or $h_t \in \cH$. 

We begin with the minimization problem \eqref{eq:hmin} over $h$, holding $g_t$ fixed, and show that it reduces to an  empirical risk minimization problem over $\cH$.

\begin{restatable}{lemma}{lemhopt}
\label{lem:hopt}
 Fix any $g_t \in \cG$ and dataset $D \in \cX^n$. Let $D_{g_t} = \{(x,y) \in \cD : g_t(x) = 1\}$ be the subset of $D$ consisting of members of group $g$. Let $h^* = \arg\min_{h \in \cH} L(D_{g_t},h)$. Then $h^*$ is a solution to optimization problem \ref{eq:hmin}.
\end{restatable}
\ifdraft
\begin{proof}
We observe that only the final term of the optimization objective in \eqref{eq:hmin} has any dependence on $h$ when $g_t$ is held fixed. Therefore:
\begin{eqnarray*}
\arg\max_{h \in \cH }\mu_{D}(g_t)\cdot (L(D,f_{t-1},g_t) - L(D,h,g_t)) &=& \arg\max_{h \in \cH }-\mu_{D}(g_t) \cdot L(D,h,g_t)) \\
&=& \arg\min_{h \in \cH} L(D,h,g_t) \\
&=& \arg\min_{h \in \cH} L(D_{g_t},h)
\end{eqnarray*}
\end{proof}
\else
The proof can be found in the appendix.
\fi

Next we consider the minimization problem \eqref{eq:gmin} over $g$, holding $h_t$ fixed and show that it reduces to an empirical risk minimization problem over $\cG$. The intuition behind the below construction is that the only points $x$ that matter in optimizing our objective are those on which the models $f_{t-1}$ and $h_t$ disagree. Amongst these points $x$ on which the two models disagree, we want to have $g(x) = 1$ if $h$ correctly predicts the label, and not otherwise.
\begin{restatable}{lemma}{lemgopt}
\label{lem:gopt}
 Fix any $h_t \in \cH$ and dataset $D \in \cX^n$. Let:
 \ifdraft
 $$D_{h_t}^1 = \{(x,y) \in D : h_t(x) = y \neq f_{t-1}(x)\} \ \ \ \ D_{h_t}^0 = \{(x,y) \in D : h_t(x) \neq y = f_{t-1}(x)\}$$

$$D_{h_t} = \left(\bigcup_{(x,y) \in D_{h_t}^1} (x, 1)\right) \cup \left(\bigcup_{(x,y) \in D_{h_t}^0} (x, 0)\right)$$
\else $D_{h_t}^1 = \{(x,y) \in D : h_t(x) = y \neq f_{t-1}(x)\}$, $D_{h_t}^0 = \{(x,y) \in D : h_t(x) \neq y = f_{t-1}(x)\}$, and $D_{h_t} = \left(\bigcup_{(x,y) \in D_{h_t}^1} (x, 1)\right) \cup \left(\bigcup_{(x,y) \in D_{h_t}^0} (x, 0)\right)$. \fi
Let $g^* = \arg\min_{g \in \cG} L(D_{h_t},g)$. Then $g^*$ is a solution to optimization problem \eqref{eq:gmin}.
\end{restatable}
\ifdraft
\begin{proof}
  For each $g \in \cG$, we partition the set of points $(x,y) \in D_{h_t}$ such that $g(x) = 1$ according to how they are labelled by $f_{t_1}$ and $h_t$:
  $$S_1(g) = \{(x,y) \in D_{h_t} : g(x) = 1, f_{t-1}(x) = h_t(x) = y\}, \ \ S_2(g) = \{(x,y) \in D_{h_t} : g(x) = 1, f_{t-1}(x) = h_t(x) \neq y\},$$
    $$S_3(g) = \{(x,y) \in D_{h_t} : g(x) = 1, f_{t-1}(x) \neq h_t(x) = y\}, \ \ S_4(g) = \{(x,y) \in D_{h_t} : g(x) = 1, f_{t-1}(x) =  y \neq h_t(x)\};$$

\noindent i.e. amongst the points in group $g$, $S_1(g)$ consists of the points that both $f_{t-1}$ and $h_t$ classify correctly, $S_2(g)$ consists of the points that both classify incorrectly, and $S_3(g)$ and $S_4(g)$ consist of points that $f_{t-1}$ and $h_t$ disagree on: $S_3(g)$ are those points that $h_t$ classifies correctly, and $S_4(g)$ are those points that $f_{t-1}$ classifies correctly. Write:
    $$w_1(g) = \frac{|S_1(g)|}{|D_{h_t}|} \ \ \ w_2(g) = \frac{|S_2(g)|}{|D_{h_t}|} \ \ \ w_3(g)  = \frac{|S_3(g)|}{|D_{h_t}|} \ \ \ w_4(g) = \frac{|S_4(g)|}{|D_{h_t}|}$$
to denote the corresponding proportions of each of the sets $S_i$ within $D_{h_t}$.

Now observe that $\mu_D(g) = w_1(g) + w_2(g) + w_3(g) + w_4(g)$, $L(D,f_{t-1},g) = \frac{w_2(g)+w_3(g)}{\mu_D(g)}$, and $L(D,h_t,g) = \frac{w_2(g)+w_4(g)}{\mu_D(g)}$. Therefore we can rewrite the objective of optimization problem \eqref{eq:gmin} as:
\begin{eqnarray*}
\arg\max_{g\in \cG }\mu_{D}(g)\cdot (L(D,f_{t-1},g) - L(D,h_t,g)) &=& \arg\max_{g\in \cG }(w_2(g) + w_3(g)) - (w_2(g) + w_4(g)) \\
&=& \arg\max_{g\in \cG } w_3(g) - w_4(g) \\
&=& \arg\min_{g\in \cG } w_4(g) - w_3(g) \\
&=& \arg\min_{g \in \cG} L(D_{h_t},g)
\end{eqnarray*}
\end{proof}
\else
The proof can be found in the Appendix.
\fi
With these two components in hand, we can describe our alternating maximization algorithm for finding $\cG\times\cH$ certificates of sub-optimality for a model $f_{t-1}$ (Algorithm \ref{alg:altmin}):

\begin{algorithm}
\KwInput{A dataset $D$, a model $f_{t-1}$, group and model classes $\cG$ and $\cH$, and an error parameter $\epsilon$.}
Let $(g^*,h^*) \in \cG \times \cH$ be an arbitrary initial certificate.  \\
CurrentValue $\leftarrow \mu_{D}(g^*)\cdot (L(D,f_{t-1},g^*) - L(D,h^*,g^*))$ \\
Let:
$$h^* = \arg\min_{h \in \cH} L(D_{g^*},h) \ \ \ g^* = \arg\min_{g \in \cG} L(D_{h^*},g)$$
\While{$\mu_{D}(g^*)\cdot (L(D,f_{t-1},g^*) - L(D,h^*,g^*) \geq \textrm{CurrentValue}+\epsilon$}{
CurrentValue $\leftarrow \mu_{D}(g^*)\cdot (L(D,f_{t-1},g^*) - L(D,h^*,g^*))$ \\
Let:
$$h^* = \arg\min_{h \in \cH} L(D_{g^*},h) \ \ \ g^* = \arg\min_{g \in \cG} L(D_{h^*},g)$$
}
\KwOut{Certificate $(g_t,h_t) = (g^*,h^*)$}
\caption{AltMinCertificateFinder($D,f_{t-1},\cG,\cH,\epsilon)$: An algorithm for finding $\epsilon$-locally optimal certificates of sub-optimality $(g_t,h_t) \in \cG\times \cH$ for model $f_{t-1}$.}
\label{alg:altmin}
\end{algorithm}

We have the following theorem:
\begin{restatable}{theorem}{localopt}
Let $D \in (\cX\times\cY)^n$ be an arbitrary dataset, $f_{t-1}:\cX\rightarrow \cY$ be an arbitrary model, $\cG$ and $\cH$ be arbitrary group and model classes, and $\epsilon > 0$. Then after solving at most $2/\epsilon$ empirical risk minimization problems over each of $\cG$ and $\cH$, AltMinCertificateFinder (Algorithm \ref{alg:altmin}) returns a $(\mu,\Delta)$-certificate of sub-optimality $(g_t,h_t)$ for $f_{t-1}$  that is an $\epsilon$-approximate local optimum (or saddle point) in the sense that:
\begin{enumerate}
    \item For every $h \in \cH$, $(g_t,h)$ is not a $(\mu', \Delta')$ certificate of sub-optimality for $f_{t-1}$ for any $\mu'\cdot \Delta' \geq \mu\cdot \Delta + \epsilon$, and
    \item For every $g \in \cG$, $(g,h_t)$ is not a $(\mu', \Delta')$ certificate of sub-optimality for $f_{t-1}$ for any $\mu'\cdot \Delta' \geq \mu\cdot \Delta + \epsilon$
\end{enumerate}
\end{restatable}
\ifdraft
\begin{proof}
 By  the halting condition of the While loop, every iteration of the While loop increases  $\mu_{D}(g^*)\cdot (L(D,f_{t-1},g^*) - L(D,h^*,g^*))$ by at least $\epsilon$. Since this quantity is bounded in $[-1,1]$, there cannot be more than $2/\epsilon$ iterations. Each iteration solves a single empirical risk minimization problem over each of $\cH$ and $\cG$.

 The certificate $(g_t,h_t)$ finally output is a $(\mu,\Delta)$ certificate of sub-optimality for $\mu = \mu_{D}(g_t)$ and $\Delta = (L(D,f_{t-1},g_t) - L(D,h_t,g_t))$. It must be  that the objective $\mu_{D}(g_t)\cdot (L(D,f_{t-1},g_t) - L(D,h_t,g_t)) = \mu\cdot \Delta$ cannot be improved by more than $\epsilon$ by re-optimizing either $h_t$ or $g_t$, by definition of the halting condition and by Lemmas \ref{lem:hopt} and \ref{lem:gopt}. The theorem follows.
\end{proof}
\else The proof can be found in the Appendix.
\fi
\fi

\section{Experiments}
\label{sec:experiments}

We next provide an experimental illustration and evaluation of our main algorithm
MonotoneFalsifyAndUpdate (Algorithm~\ref{alg:MonFalsifyAndUpdate}), as well as our optimization approach (Algorithm \ref{alg:TrainByOpt}) paired with our reduction to ternary cost sensitive classification. 
We report findings on a number of different
datasets and classification tasks from the recently published Folktables package~\citep{ding2021retiring},
which provides extensive U.S. census-derived Public Use Microdata Samples (PUMS). These granular and large datasets
are well-suited to experimental evaluations of algorithmic fairness, as they include
a number of demographic or protected features including race categories, age, disability status, and binary sex categories.
We refer to~\citep{ding2021retiring} for details of the datasets examined below.

We implemented Algorithm~\ref{alg:MonFalsifyAndUpdate} and divided each dataset examined into
80\% for training $(g,h)$ pairs as proposed certificates of sub-optimality for the algorithm's current model,
and 20\% for use by the certificate checker to validate improvements.
In order to clearly illustrate the gradual subgroup and overall error improvements, we train a 
deliberately simple initial model (a decision stump). In the first set of experiments, the sequence
of subgroups is 11 different demographic groups given by features common to all of the Folktables datasets.
We consider four different datasets \ifdraft \else (three pictured in Figure \ref{fig:acs-results}, a fourth in the appendix) \fi corresponding to four different states and prediction tasks, all for the
year 2018 (the most recent available). Table~\ref{table:data} \ifdraft\else in the Appendix \fi and its caption detail the states and tasks, the total dataset
sizes, and the definitions and counts of each demographic subgroup considered.  Figure~\ref{fig:acs-description} \ifdraft\else in the Appendix \fi shows a sample
fragment and the full feature set for one of the prediction tasks. 

\ifdraft
\begin{table}[h!]
\smaller
\begin{tabular}{lllllllllllll}
Dataset & Total & White & Black & Asian & Native & Other & Two+ & Male & Female & Young & Middle & Old \\
\hline
NY Employment & 196966 & 138473 & 24024 & 17030 & 10964 & 5646 & 829 & 95162 & 101804 & 68163 & 47469 & 81334 \\
Oregon Income & 21918 & 18937 & 311 & 923 & 552 & 823 & 372 & 11454 & 10464 & 5041 & 9124 & 7753 \\
Texas Coverage & 98927 & 72881 & 11192 & 4844 & 6660 & 2555 & 795 & 42128 & 56799 & 42048 & 31300 & 25579 \\
Florida Travel & 88070 & 70627 & 10118 & 2629 & 2499 & 1907 & 290 & 45324 & 42746 & 16710 & 35151 & 36209 \\

\end{tabular}
\caption{Summaries of the four Folktables state/task datasets we use in the experiments discussed below,
indicating total population size and the sizes of the different demographic subgroups considered.
The full descriptors of the demographic subgroups are: (race subgroups) White; Black or African American; Asian;  Native Hawaiian, Native American, Native Alaskan, or Pacific Islander; Some Other Race; Two or More Races; (gender subgroups) Male; Female;
(age subgroups) Young; Middle; Old.}
\label{table:data}
\end{table}
\fi

\ifdraft
\begin{figure}[h!]
        \center
        \subfloat{\includegraphics[width=.7\columnwidth,trim={0.5cm 0.5cm 0.5cm 0.5cm}]{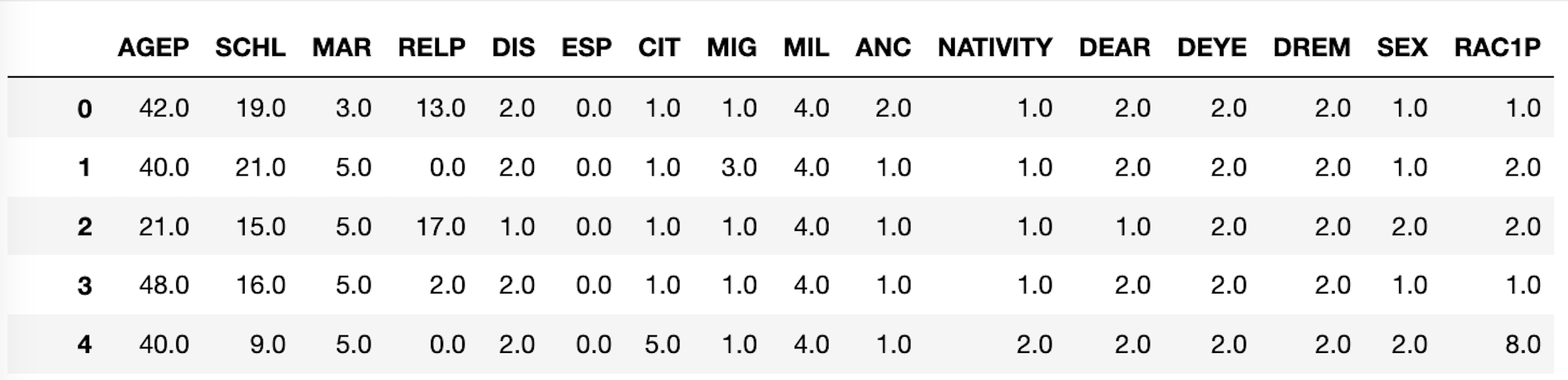}}
        \caption{A fragment of the Folktables ACS PUMS dataset, with the 16 features considered for an employment prediction task: age (AGEP), education (SCHL), marital status (MAR), relationship (RELP), disability status (DIS), parent employee status (ESP), citizen status (CIT), mobility status (MIG), military service (MIL), ancestry record (ANC), nation of origin (NATIVITY), hearing difficulty (DEAR), visual dificulty (DEYE), learning disability (DREM), sex (SEX), and race (RAC1P).}
    \label{fig:acs-description}
\end{figure}
\fi

In our first experiments, these 11 subgroups $g$
were introduced to Algorithm~\ref{alg:MonFalsifyAndUpdate} in some order, and for each a depth 10 decision tree $h$
was trained on just the training data falling into that subgroup. These $(g,h)$ pairs are then given to the certificate
checker, and if  $h$ improves the holdout set error for $g$, the improvement is accepted and incorporated into
the model, and any needed monotonicity repairs of previously introduced groups are made. Otherwise the pair is rejected and we
proceed to the next subgroup.

\ifdraft
Each of the five rows of Figure~\ref{fig:acs-results} shows the results of one experiment of this type, 
with the left panel in each row showing the test error of each subgroup and overall;
the middle panel showing the training error of each subgroup and overall;
and the right panel showing the absolute difference of train and test errors.
The $x$ axis corresponds to the rounds of Algorithm~\ref{alg:MonFalsifyAndUpdate}.
Since the subgroups are introduced in some sequential order (given in the legends),
we plot the subgroup errors using dashed lines up until the round at which the group is
introduced, and solid lines thereafter for visual clarity. Subgroups whose trained decision tree
was rejected by the certificate checker are not shown at all. In Figure \ref{fig:ny_employ_pdl}, a depiction of the decision list generated by the employment task in row 1 of Figure \ref{fig:acs-results} is given.
\else
Each of the plots in Figure ~\ref{fig:acs-results} shows the results of one experiment of this type on test data. Since the subgroups are introduced in some sequential order (given in the legends),
we plot the subgroup errors using dashed lines up until the round at which the group is
introduced, and solid lines thereafter. Subgroups whose trained decision tree
was rejected by the certificate checker are not shown at all. In Figure \ref{fig:ny_employ_pdl}, a depiction of the decision list generated by the employment task shown in the leftmost plot in Figure \ref{fig:acs-results} is given. 
\fi

\ifdraft A number of overarching observations are in order. First, as \else As \fi promised by Theorem~\ref{thm:bounty-monotone-SC}, the overall and subgroup test errors \ifdraft (left panels) \else\fi are all monotonically non-increasing
after the introduction of the subgroup,
and generally enjoy significant improvement as the algorithm proceeds. Note that even prior to
the introduction of a subgroup (dashed lines), the test errors are generally non-increasing. Although this is not guaranteed by the theory, it is not necessarily surprising, since before groups are introduced they often start from a high baseline error. However, this is
not always the case --- for example, on the Oregon income prediction task \ifdraft (third row) \else (middle panel) \fi the test error for the Native subgroup (light green) increases significantly with the
introduction of the subgroups for young and middle-aged at rounds 3 and 4 before it has been introduced.

\ifdraft
 
We also note that, as shown in Figure \ref{fig:ny_employ_pdl}, the need for repair backpointers for previously introduced groups upon introduction of a new group is rather common on these datasets, which demonstrates that this is a crucial feature of our algorithm for guaranteeing group monotonicity. Note that this is not in opposition to our observation that we often have monotonic improvement for groups before they are introduced, since after a group is introduced, the model performs well on them, and so it is more likely for changes in group performance to increase error.
\fi

\ifdraft
Turning next to the subgroup and overall training errors (middle panels), we observe that while they are generally non-increasing, they are not exclusively non-increasing, and 
this is not guaranteed by the theory --- indeed, the entire purpose of the certificate checker enforcing monotonicity on the test
set is to validate proposed improvements from training. Examining the differences between train and test errors (right panels):
as expected, the gaps tend to be larger for smaller subgroups since overfitting is more likely. Nevertheless, by the final rounds, train and test errors are generally close, almost always
less than 0.1 and often less than 0.05. Second, while volatile, the train-test gaps generally increase with rounds of
the algorithm. This is because the algorithm's model increases in complexity with rounds, thus
training errors underestimate test errors more as the algorithm proceeds. 
\fi

\begin{figure*}[h!]
        \ifdraft
        \subfloat{\includegraphics[width=.3\columnwidth,trim={0.5cm 0.5cm 0.5cm 0.5cm}]{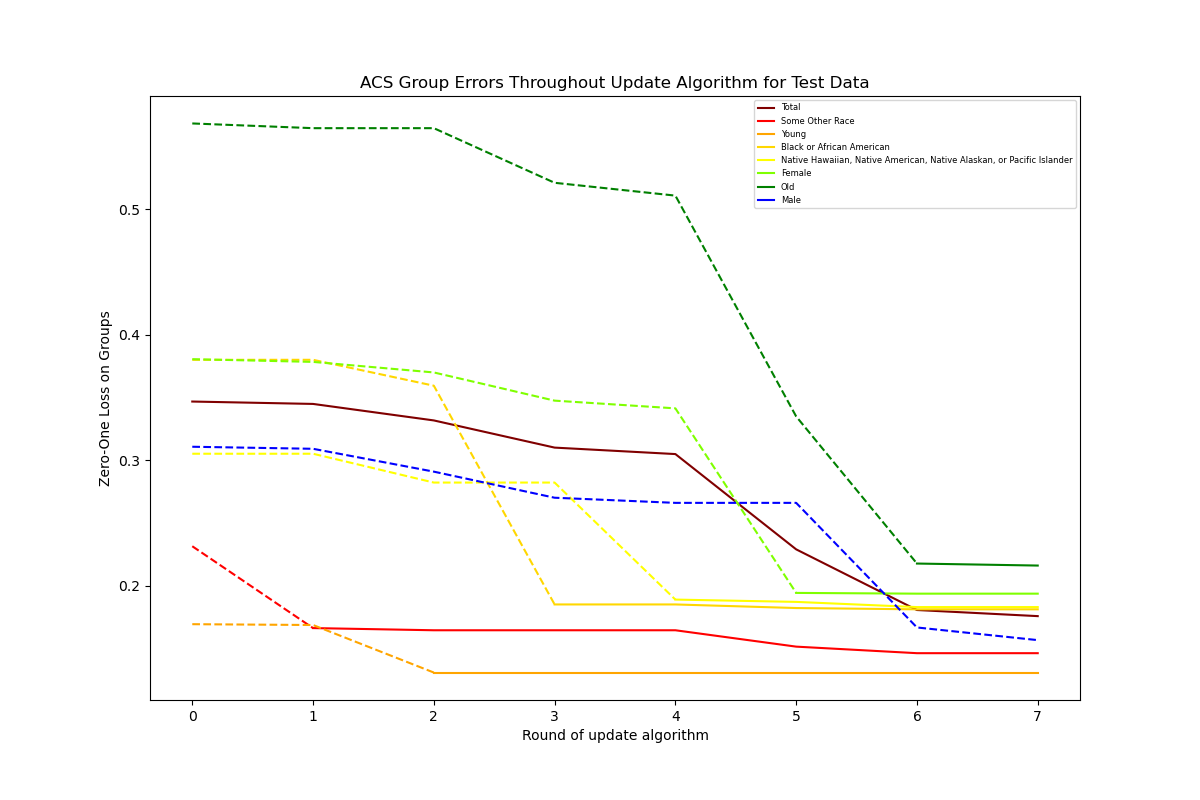}} \subfloat{\includegraphics[width=.3\columnwidth,trim={0.5cm 0.5cm 0.5cm 0.5cm}]{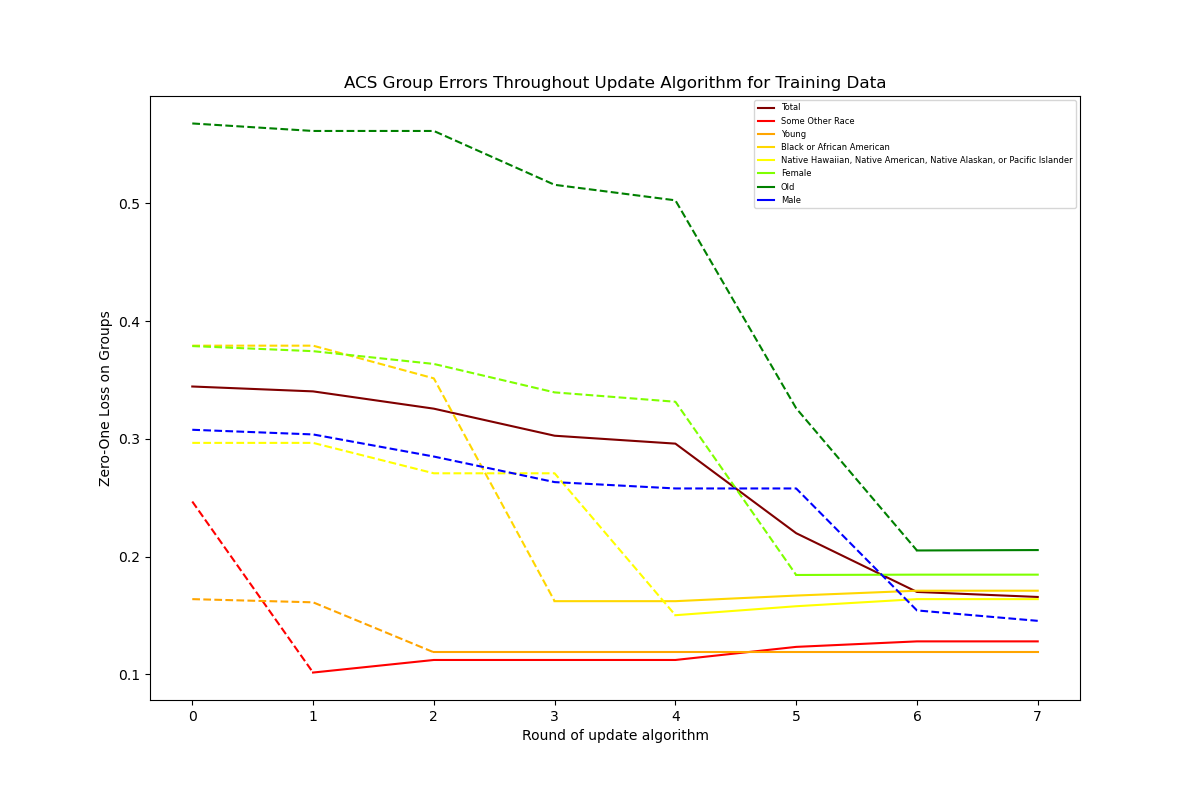}}
        \subfloat{\includegraphics[width=.3\columnwidth,trim={0.5cm 0.5cm 0.5cm 0.5cm}]{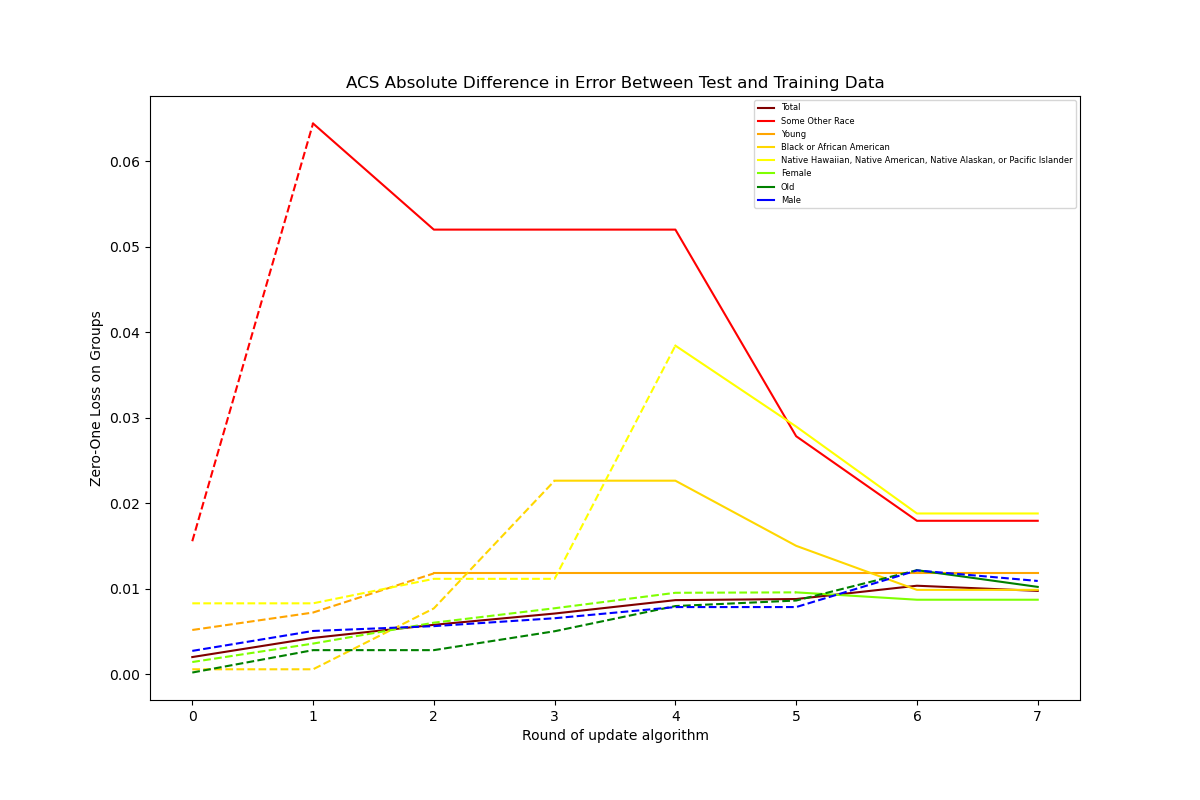}} \\
        \subfloat{\includegraphics[width=.3\columnwidth,trim={0.5cm 0.5cm 0.5cm 0.5cm}]{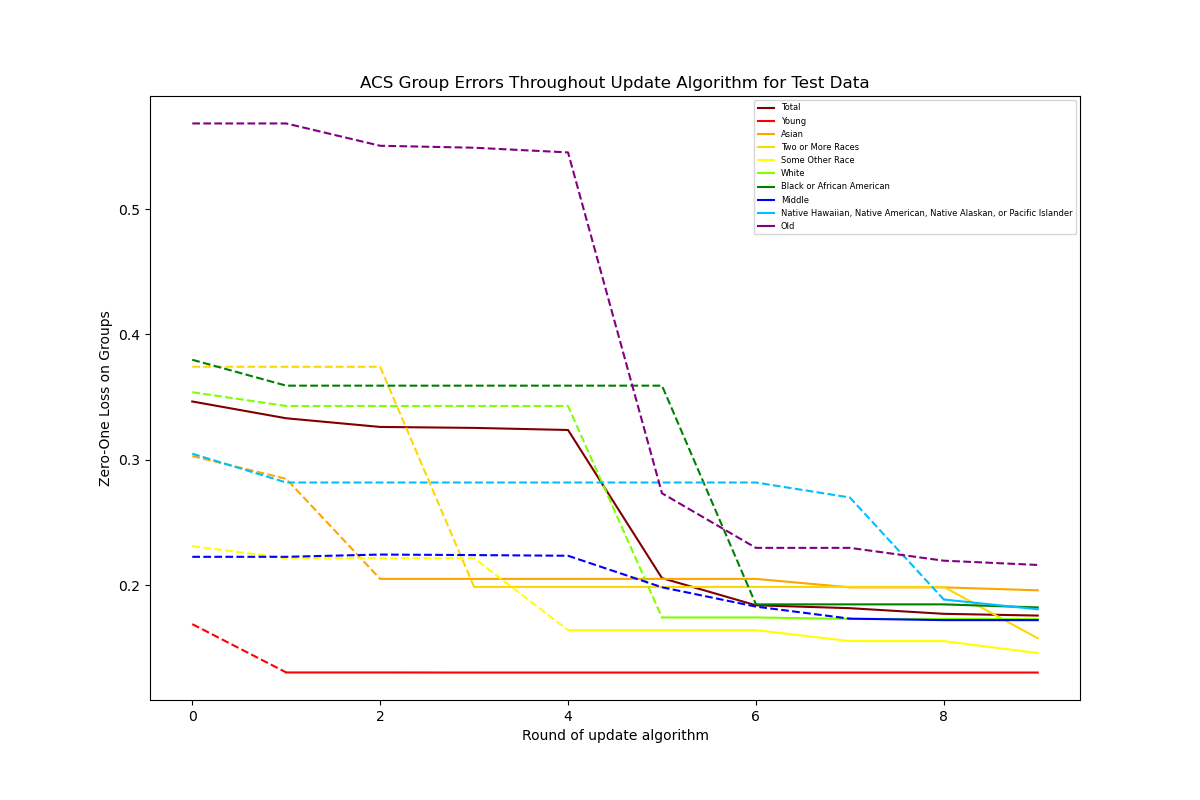}} \subfloat{\includegraphics[width=.3\columnwidth,trim={0.5cm 0.5cm 0.5cm 0.5cm}]{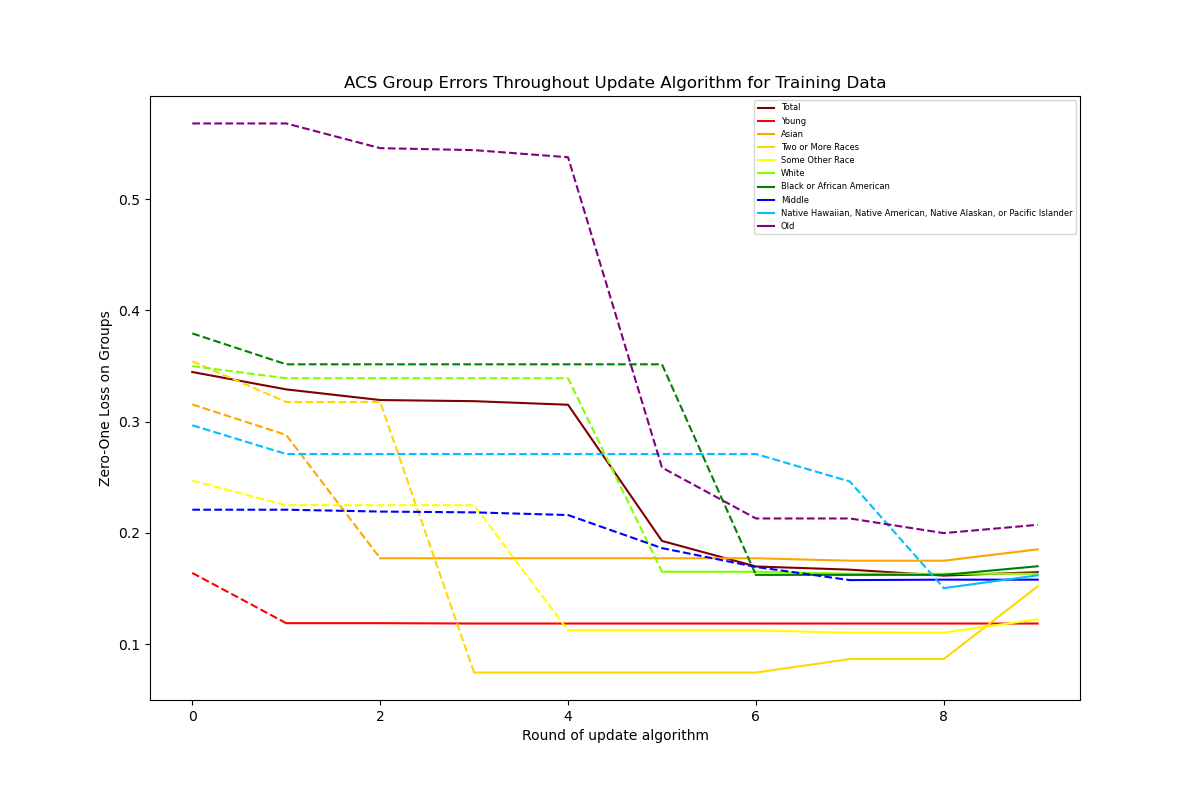}}
        \subfloat{\includegraphics[width=.3\columnwidth,trim={0.5cm 0.5cm 0.5cm 0.5cm}]{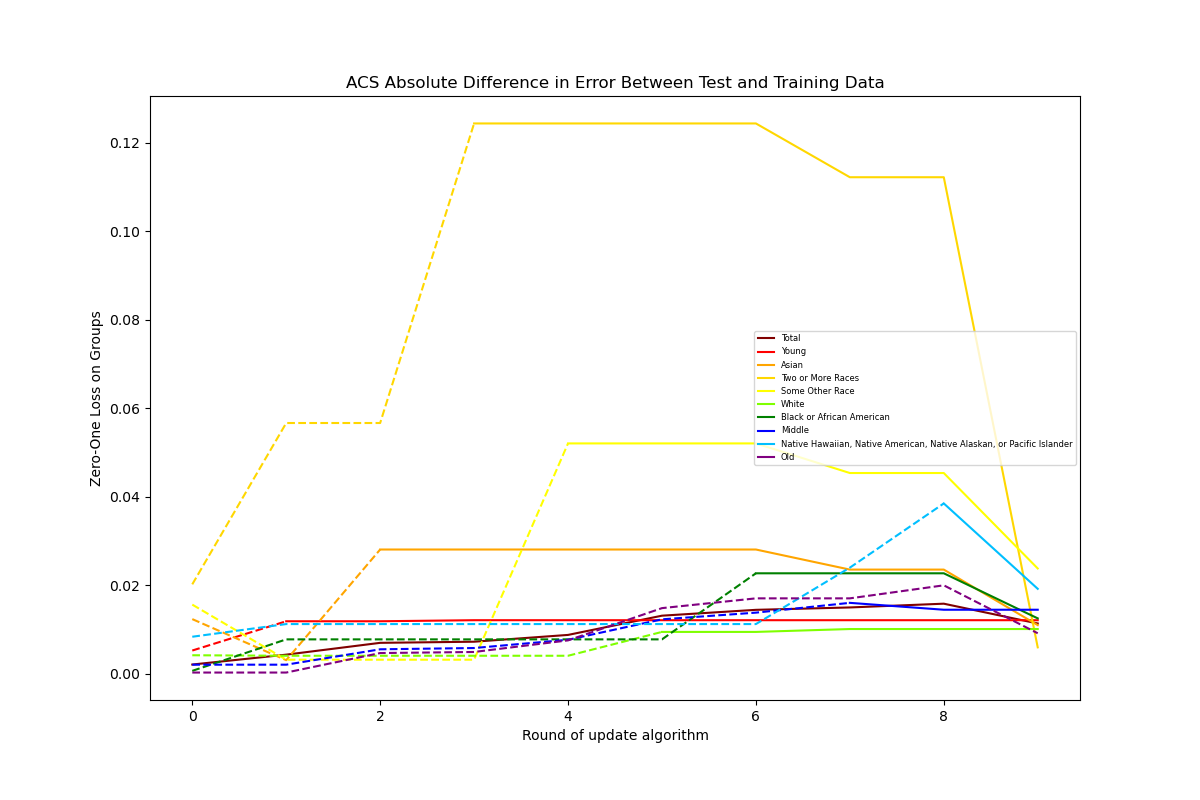}} \\
        \subfloat{\includegraphics[width=.3\columnwidth,trim={0.5cm 0.5cm 0.5cm 0.5cm}]{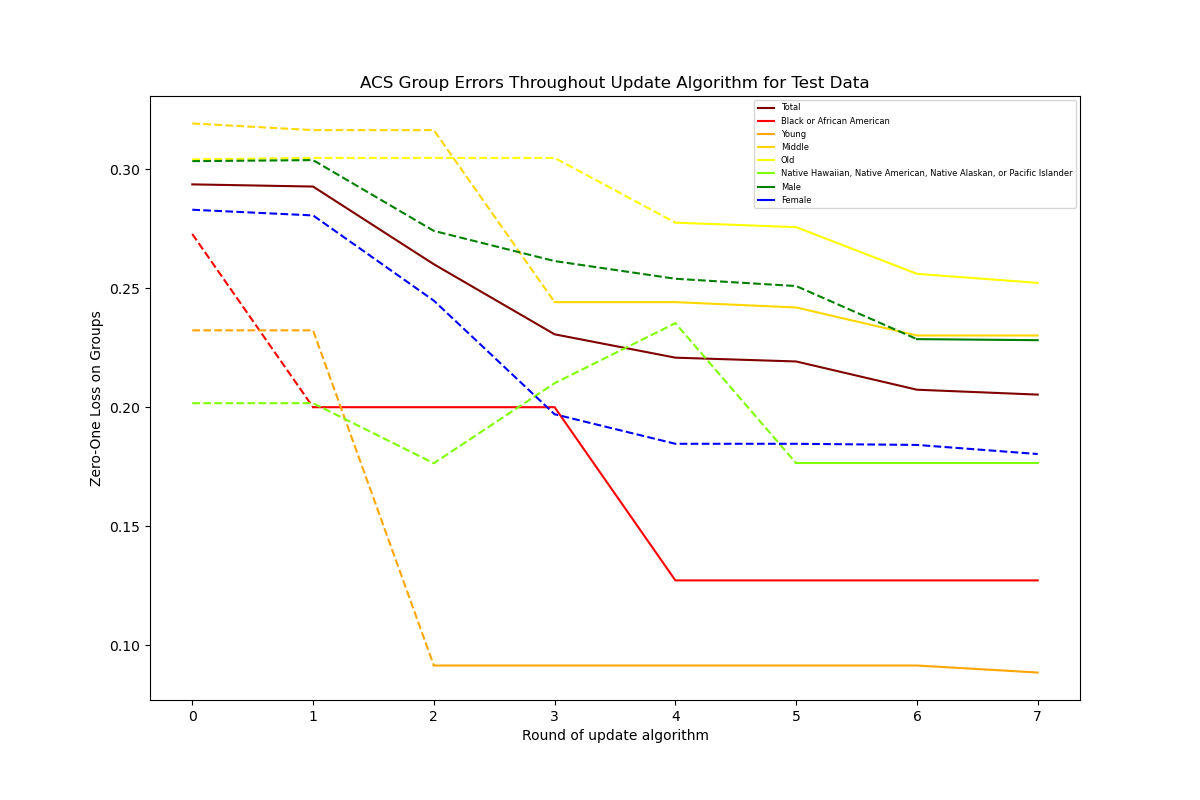}} \subfloat{\includegraphics[width=.3\columnwidth,trim={0.5cm 0.5cm 0.5cm 0.5cm}]{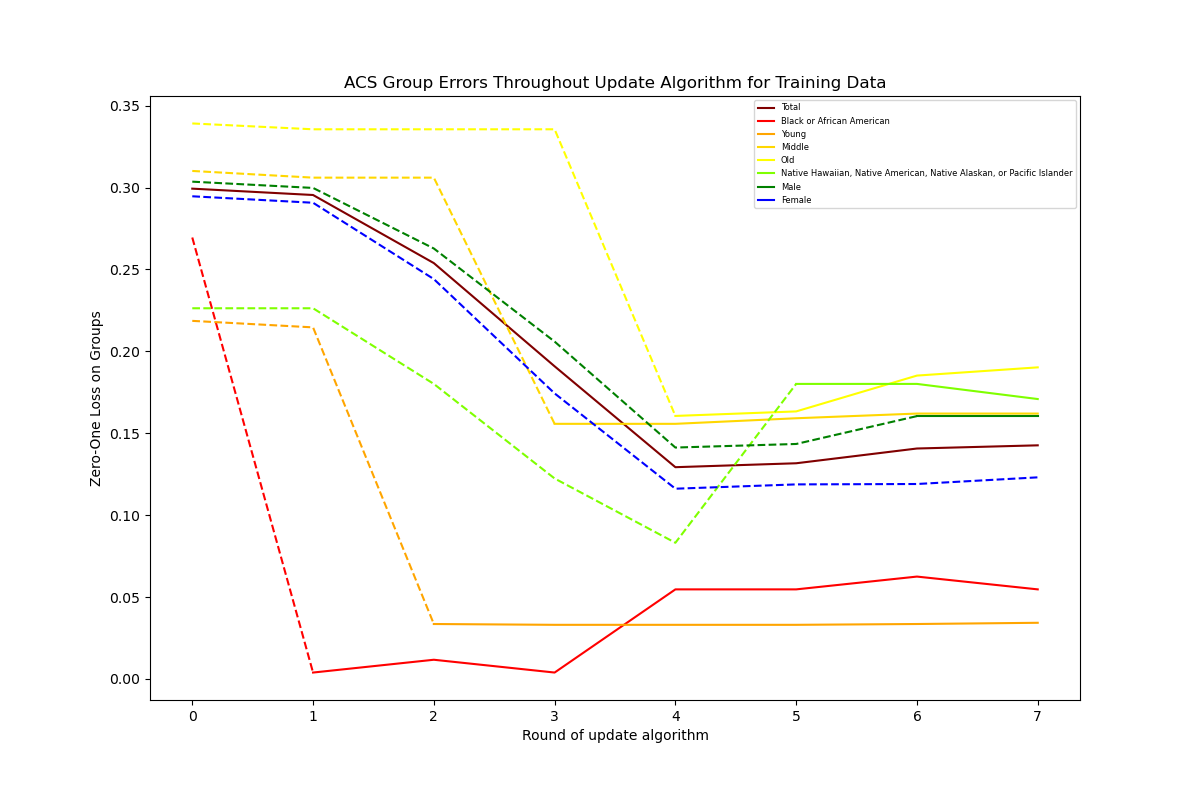}}
        \subfloat{\includegraphics[width=.3\columnwidth,trim={0.5cm 0.5cm 0.5cm 0.5cm}]{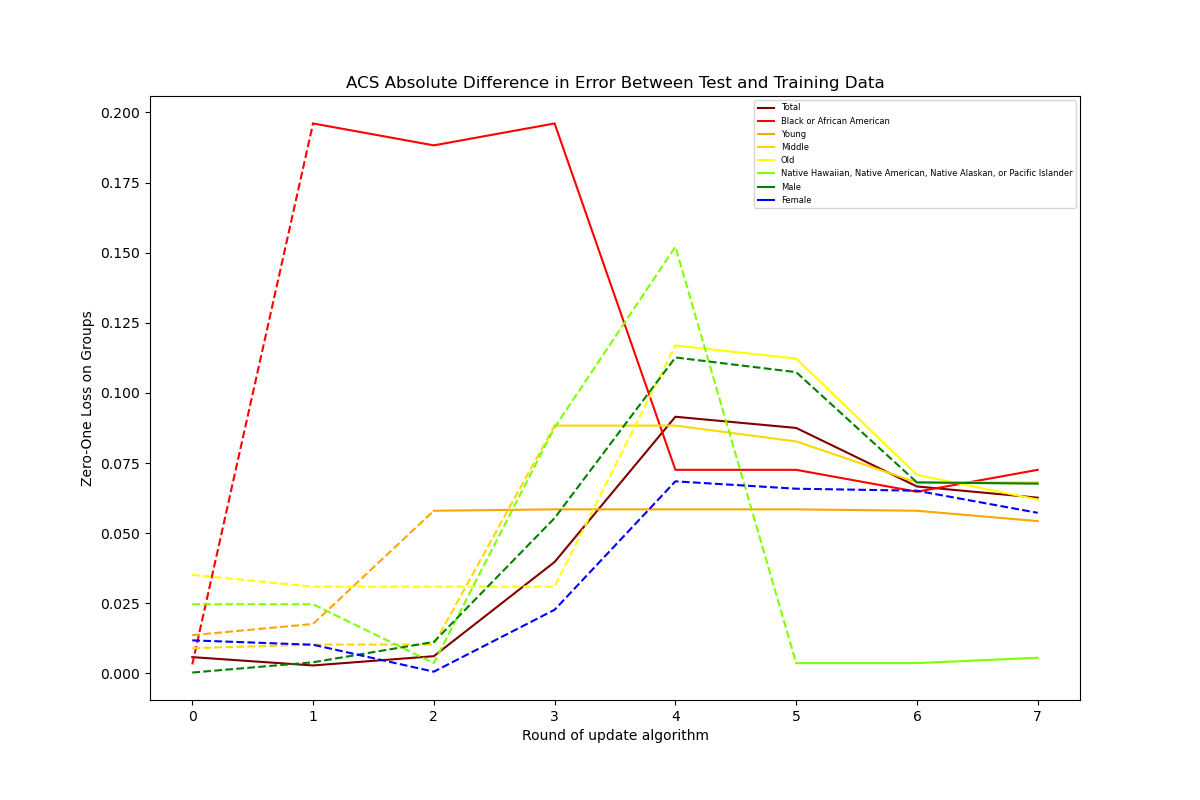}} \\
        \else
        \subfloat{\includegraphics[width=0.3 \textwidth,trim={0.5cm 0.5cm 0.5cm 0.5cm}]{figures/ny_18_employ_12345_test.png}}
        \subfloat{\includegraphics[width=0.3 \textwidth,trim={0.5cm 0.5cm 0.5cm 0.5cm}]{figures/or_18_income_12340_test.png}}
        \subfloat{\includegraphics[width=.3\textwidth,trim={0.5cm 0.5cm 0.5cm 0.5cm}]{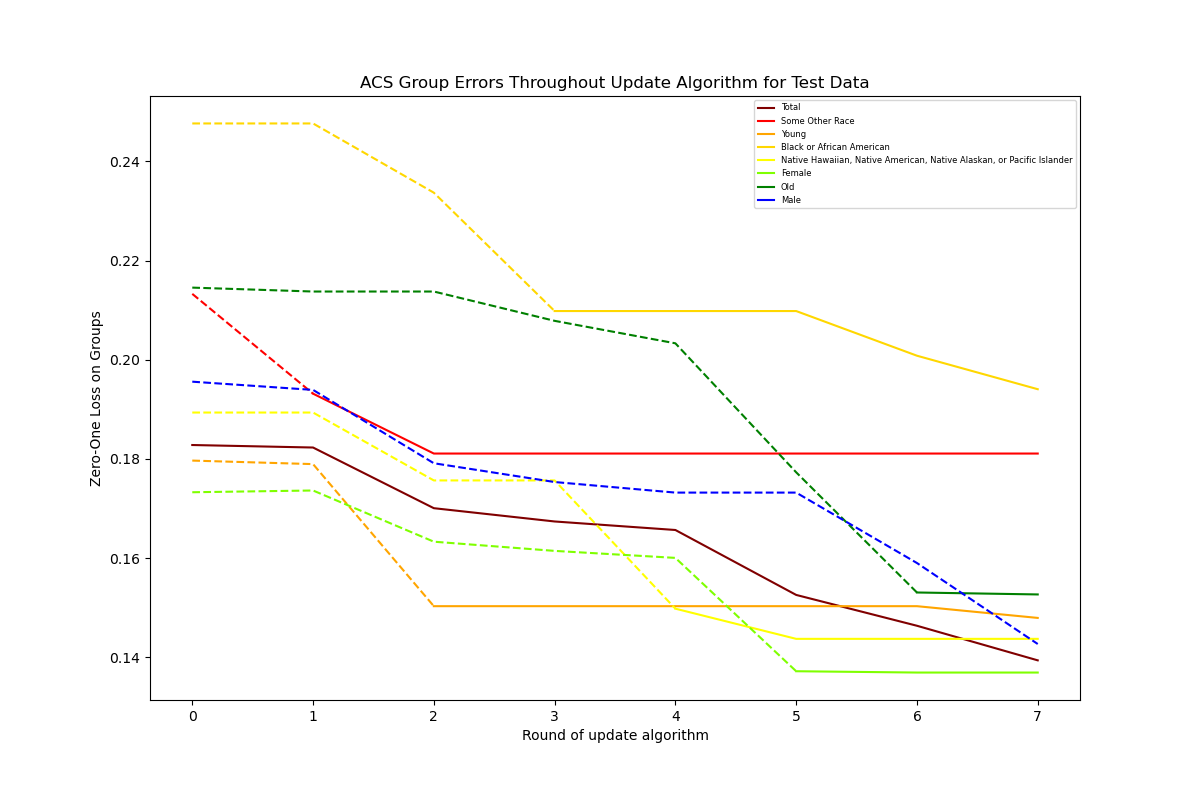}}
        \fi 
        \ifdraft
        \subfloat{\includegraphics[width=.3\columnwidth,trim={0.5cm 0.5cm 0.5cm 0.5cm}]{figures/tx_18_cover_12345_test.png}} \subfloat{\includegraphics[width=.3\columnwidth,trim={0.5cm 0.5cm 0.5cm 0.5cm}]{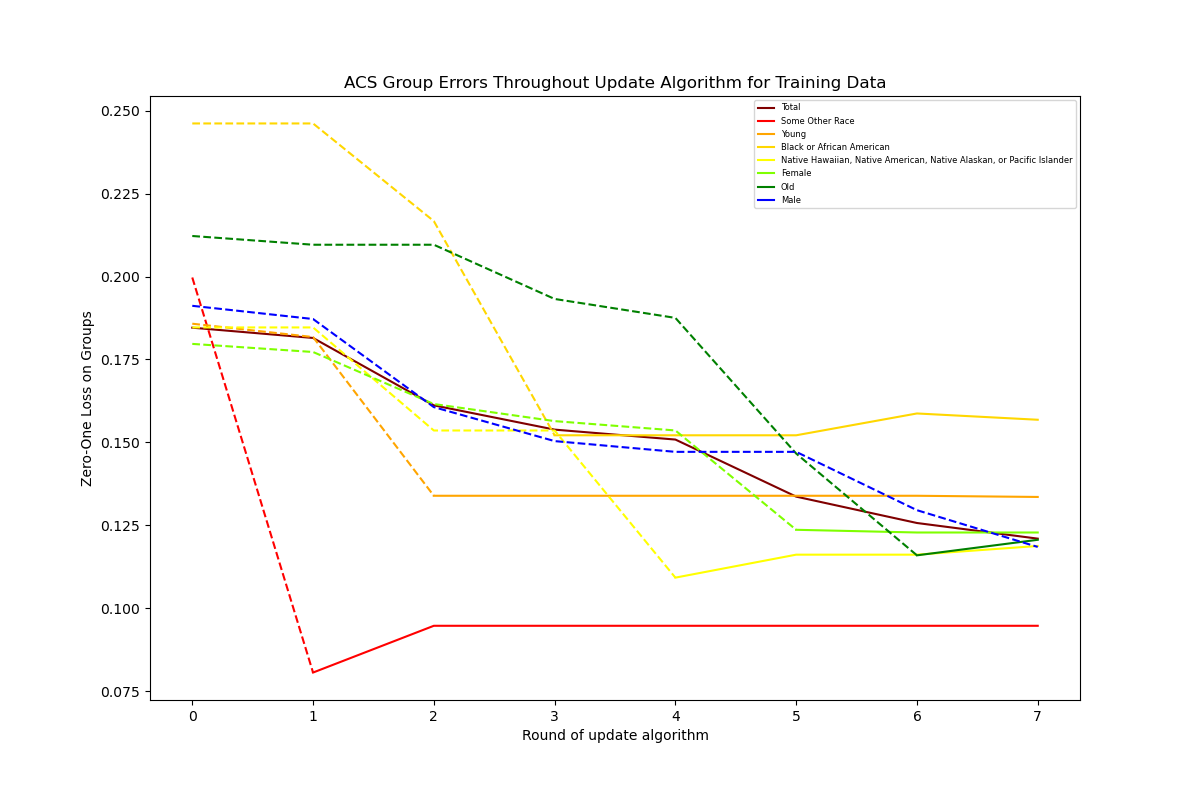}}
        \subfloat{\includegraphics[width=.3\columnwidth,trim={0.5cm 0.5cm 0.5cm 0.5cm}]{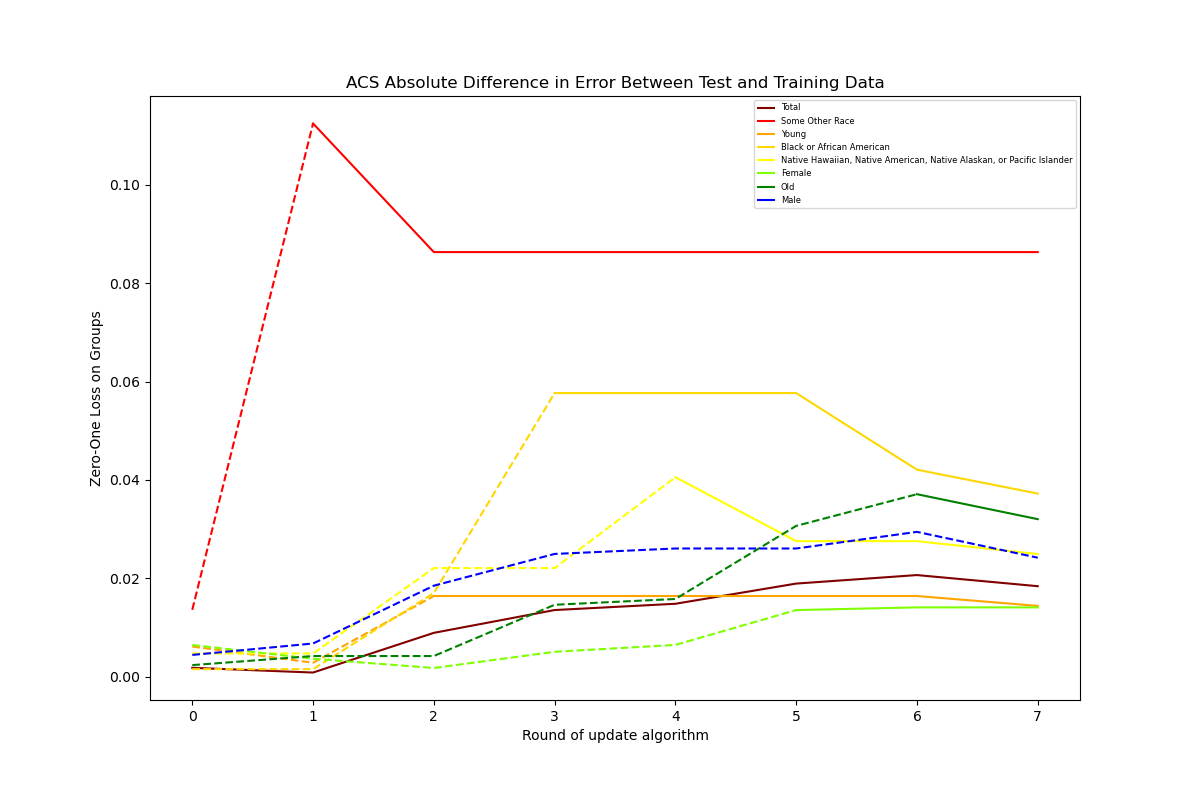}} \\
        \subfloat{\includegraphics[width=.3\columnwidth,trim={0.5cm 0.5cm 0.5cm 0.5cm}]{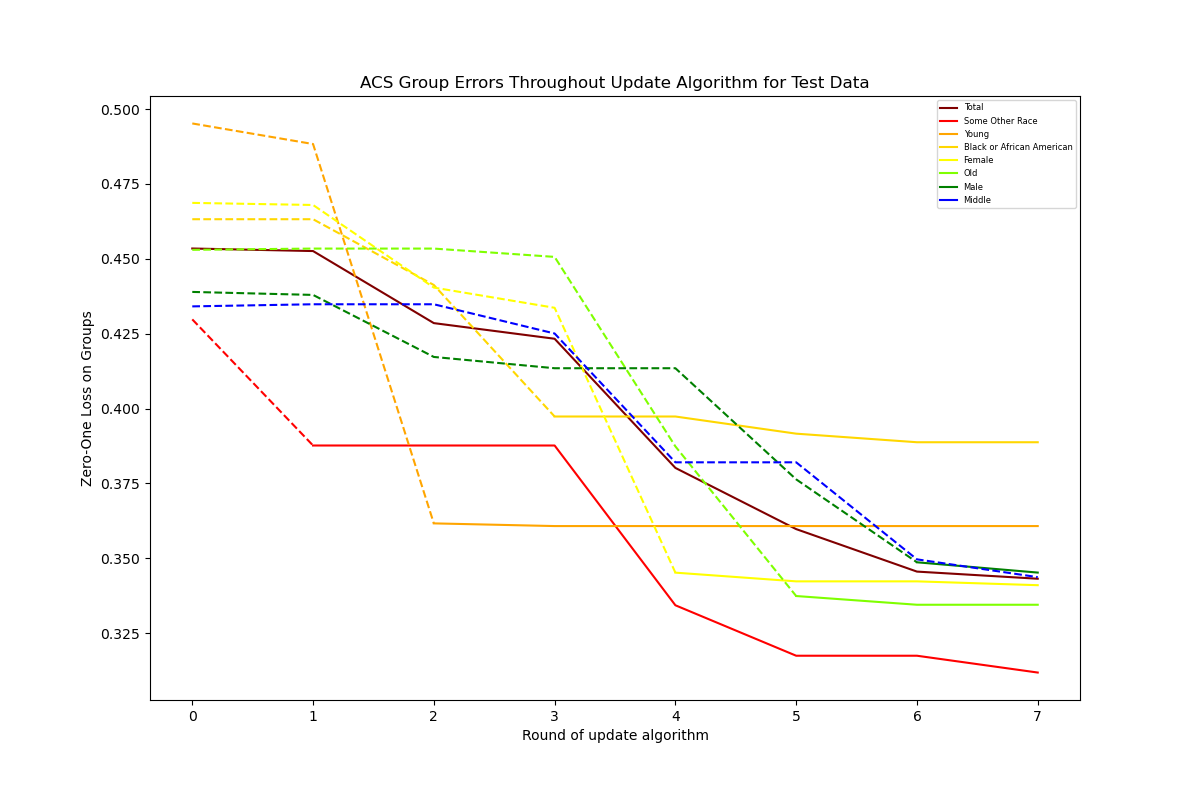}} \subfloat{\includegraphics[width=.3\columnwidth,trim={0.5cm 0.5cm 0.5cm 0.5cm}]{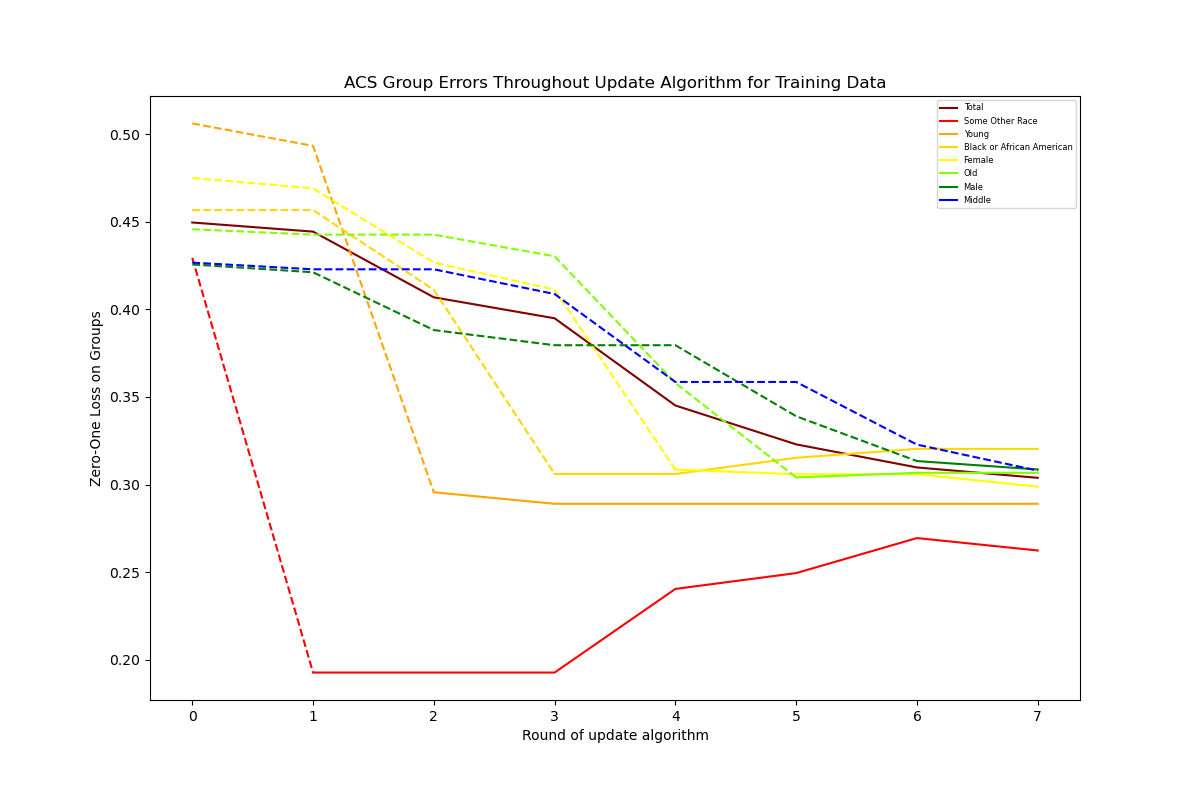}}
        \subfloat{\includegraphics[width=.3\columnwidth,trim={0.5cm 0.5cm 0.5cm 0.5cm}]{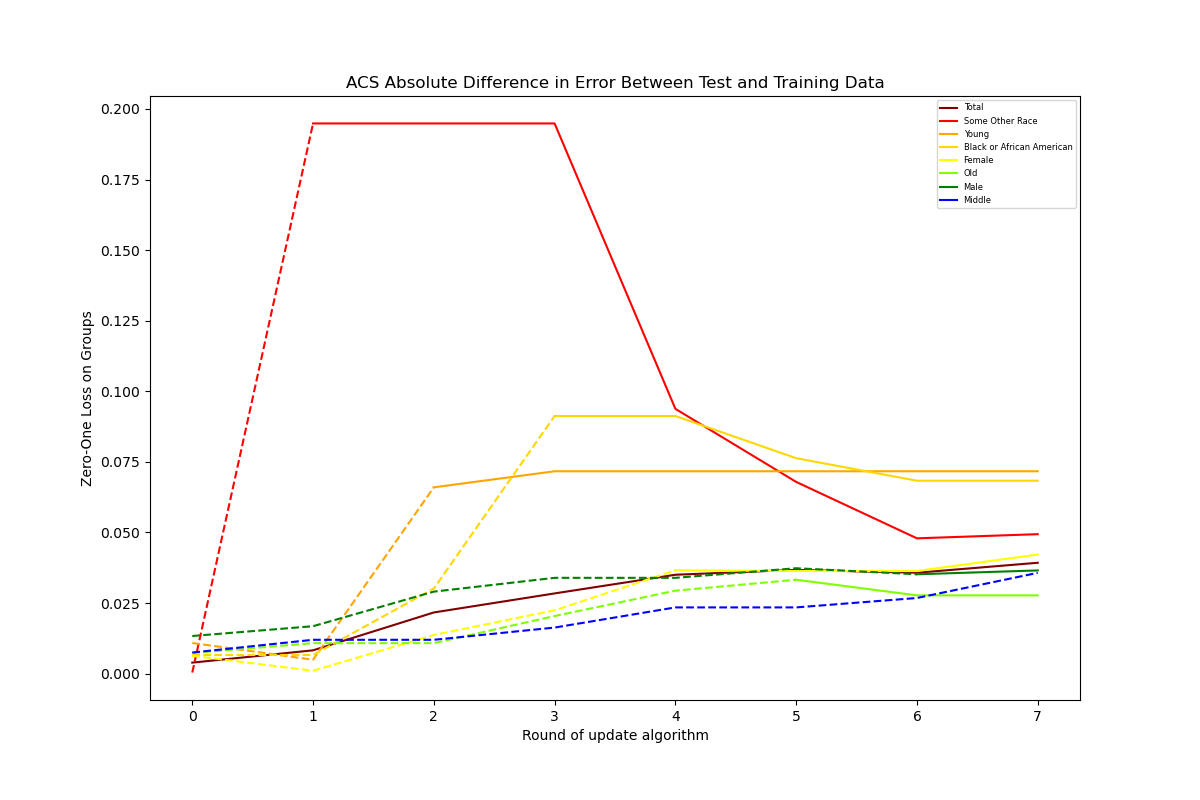}}
        \else \fi 
        \caption{\ifdraft \footnotesize \fi \ifdraft Sample results of an implementation of algorithm MonotoneFalsifyAndUpdate (Algorithm~\ref{alg:MonFalsifyAndUpdate}) on a
        number of different task and U.S. state datasets from the Folktables package. In each row, the 
        left panel shows group error rates on the test
        data; the middle panel shows group error rates on the training data; and the right panel shows the absolute difference of test and training error rates. First row: ACS employment task for New York state. Second row: ACS employment task for New York but with a
        different order of introduction of the groups. Third row: ACS income task for Oregon. Fourth row: ACS coverage task for Texas.
        Fifth row: ACS travel time task for Florida. \else 
        Results of an implementation of algorithm Algorithm~\ref{alg:MonFalsifyAndUpdate} on a
        number of different tasks and U.S. state datasets from the Folktables package. Left: employment task for NY, middle: income task for OR, and right: coverage task for TX. 
        \fi See text for discussion. }
    \label{fig:acs-results}
\end{figure*}

\ifdraft
\begin{figure}[h!]
    \centering
    \includegraphics[scale=0.18]{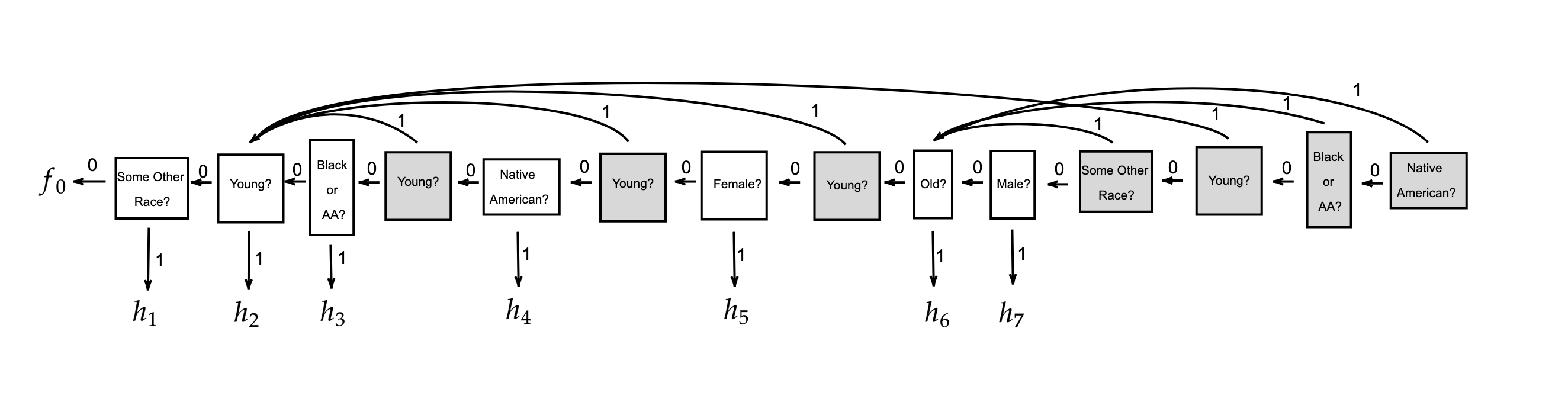}
    \caption{The pointer decision list generated by an implementation of algorithm MonotoneFalsifyAndUpdate (Algorithm~\ref{alg:MonFalsifyAndUpdate}) on an ACS employment task for New York State. Repair nodes are in gray. There were three additional updates that were rejected for models trained on Asian individuals, middle-age individuals, and those belonging to the ``other" race category.}
    \label{fig:ny_employ_pdl}
\end{figure}
\else
\begin{figure*}[h!]
    \centering
    \includegraphics[width = \textwidth]{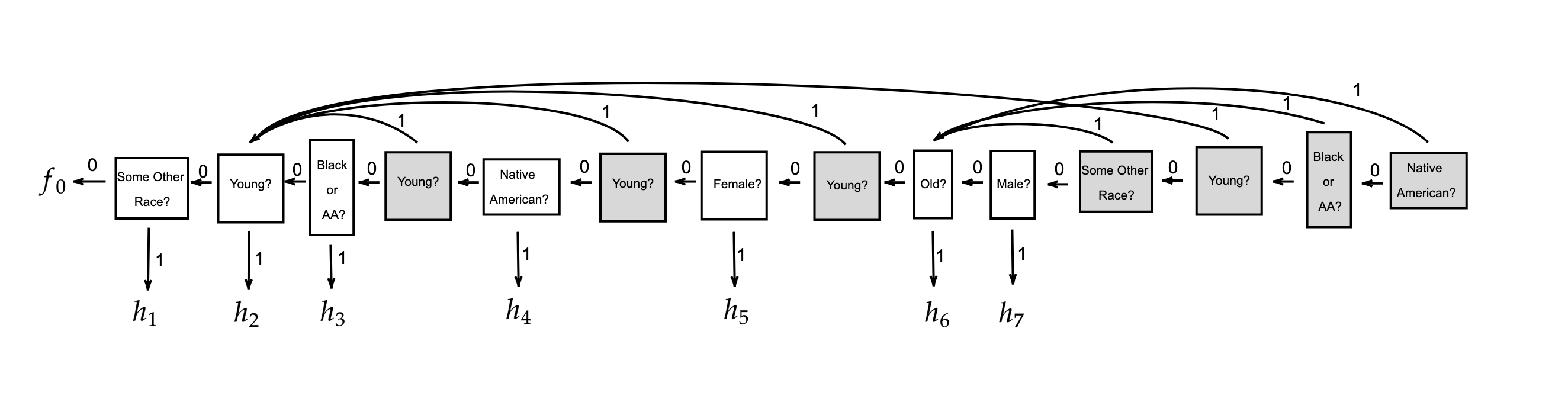}
    \caption{The pointer decision list generated by an implementation of algorithm MonotoneFalsifyAndUpdate (Algorithm~\ref{alg:MonFalsifyAndUpdate}) on an ACS employment task for New York State. Repair nodes are in gray.}
    \label{fig:ny_employ_pdl}
\end{figure*}
\fi 

\ifdraft
The first two rows of plots in Figure~\ref{fig:acs-results} correspond to experiments on the same dataset (New York employment task),
with the only difference being the order in which the demographic subgroups are introduced to 
Algorithm~\ref{alg:MonFalsifyAndUpdate}. This ordering can have qualitative effects on the behavior of the algorithm, including
which subgroup improvements are accepted. For example, the male and female subgroups appear in the first row but are absent in the second due to their proposed improvements being rejected, while the Asian, 2 or more races, white, and middle-age subgroups
appear in the second row but not the first. While the evolution of errors for the subgroups
common to both rows is also quite different, in general both the overall and all 11 subgroup test errors end up being quite similar 
for both orderings at their final round, with the differences generally being in the second or third decimal 
place, as shown by Table~\ref{table:order-errors}.
\else 
While not visible in Figure \ref{fig:acs-results}, we note that the order the demographic groups are introduced can have qualitative effects on the behavior of the algorithm, including
which subgroup improvements are accepted, as shown in the first two rows of Figure \ref{fig:acs-results-app} in the Appendix. However, regardless of the ordering, the final errors across all groups end up comparable, with the differences generally being in the second or third decimal 
place, as shown by Table~\ref{table:order-errors} ~in the Appendix.
\fi

\ifdraft
\begin{table}[h!]
\begin{center}
\tabcolsep=0.11cm
\begin{tabular}{lllllllllllll}
Total & White & Black & Asian & Native  & Other & Two+ & Male & Female & Young & Middle & Old \\
\hline \\
0.0004 & 0.016 & 0.065& 0.023 & 0.0367 & 0.020 & 0.033 & 0.016 & 0.008 & 0.040 & 0.045 & 0.030 \\
\end{tabular}
\end{center}
\caption{Absolute differences in final overall and subgroup test errors between two different sequential orderings.
See text for details.}
\label{table:order-errors}
\end{table}
\fi 

Examining the actual pointer decision list produced in one of these experiments is also enlightening --- see Figure \ref{fig:ny_employ_pdl}. We note a couple of things. First, observing the repairs for the group ``Young'', we see that via the sequence of repairs, any ``young'' individual will always be classified by the model $h_2$ introduced at round 2 together with the group ``young". This is in contrast to the other groups, whose members are split across a number of different models by the pointer decision list. This indicates that being in the group ``young'' is more salient than being in any of the racial or gender groups in this task from the point of view of accurate classification. Second, we note that unlike for the group ``young'', repairs for other groups do not necessarily point back to the round at which they were introduced. For example, the right-most repairs for ``Native American'', ``Black or AA'', and ``Some Other Race'' point back to round 6, at which the group ``Old'' was introduced. This indicates that the pointer decision list at round 6 is more accurate for each of these groups then the decision trees $h_4$ $h_3$ and $h_1$ that were introduced along with these groups, despite the fact that these decision trees were trained \emph{only} on examples from their corresponding groups.  

The experiments described so far apply our framework and algorithm to a setting in which the groups
considered are simple demographic groups, introduced in a fixed ordering.
We finish with a brief experimental investigation of the approach described in Section~\ref{sec:restricted}\ifdraft \else and Appendix \ref{app:algopt}\fi, in
which the discovery of updates is posed as an optimization problem.  Specifically,
we implemented the ternary CSC approach of Section~\ref{sec:ternary}. In our implementation,  following the
approach taken in \citep{agarwal2018reductions,kearns2019empirical}, we learn a ternary
classifier by first learning two separate depth 7 decision tree regression functions for the costs of predicting 0 and 1, 
while the cost of predicting "?" is always 0 as per Section~\ref{sec:ternary}. Our final ternary classifier
chooses the prediction that minimizes the costs predicted by the regression functions.

We applied the CSC approach to the ACS income task using Folktables datasets for thirteen different states in the year 2016.
One interesting empirical phenomenon is the speed with which this approach converges, generally stopping after two
to five rounds and not finding any further subgroup improvements to the current model. \ifdraft
While the subgroups identified by the algorithm are not especially simple or interpretable (being defined by the
minimum over two depth 7 regression trees), one way of understanding behavior and performance is by direct
comparison of the test errors on the 11 basic demographic subgroups, which are available as features (along with various
non-demographic features) to the regression trees. \else\fi In Figure~\ref{fig:csc-errors}, we compare these subgroup test
errors to those obtained by the simple sequential introduction of those subgroups discussed earlier. In the left panel,
the $x$ axis represents each of the 11 subgroups, and the color coding of the points represents which of the 13 state
datasets is considered. The $y$ values of the points measure the signed difference of the test errors of the CSC approach
and the simple sequential approach (using a fixed ordering), 
with positive values being a win for sequential and negative values a win for CSC.
The right panel visualizes the same data, but now grouped by states on the $x$ axis and with color coding for the groups.

The overarching message of the figure is that though CSC does not directly consider these subgroups, and instead optimizes for the complex subgroup giving the best weighted improvement to the overall error, it is nevertheless
quite competitive on the simple subgroups, with the mass of points above and below $y = 0$ being approximately the same.
\ifdraft More specifically, averaged over all 143 state-group pairs, the average difference is 0.018, only a slight win for the sequential
approach; and for 5 of the 11 groups (Black or African American, some other race, Native, female and young) the CSC averages are better. \else \fi Combined with the very rapid
convergence of CSC, we can thus view it as an approach that seeks rich subgroup optimization while providing 
strong basic demographic group performance ``for free''. Further, combining the two approaches (e.g. running the CSC approach then sequentially
introducing the basic groups) should only yield further improvement.

\ifdraft
\begin{figure}[h!]
        \center
        \subfloat{\includegraphics[width=.45\columnwidth,trim={0.5cm 0.5cm 0.5cm 0.5cm}]{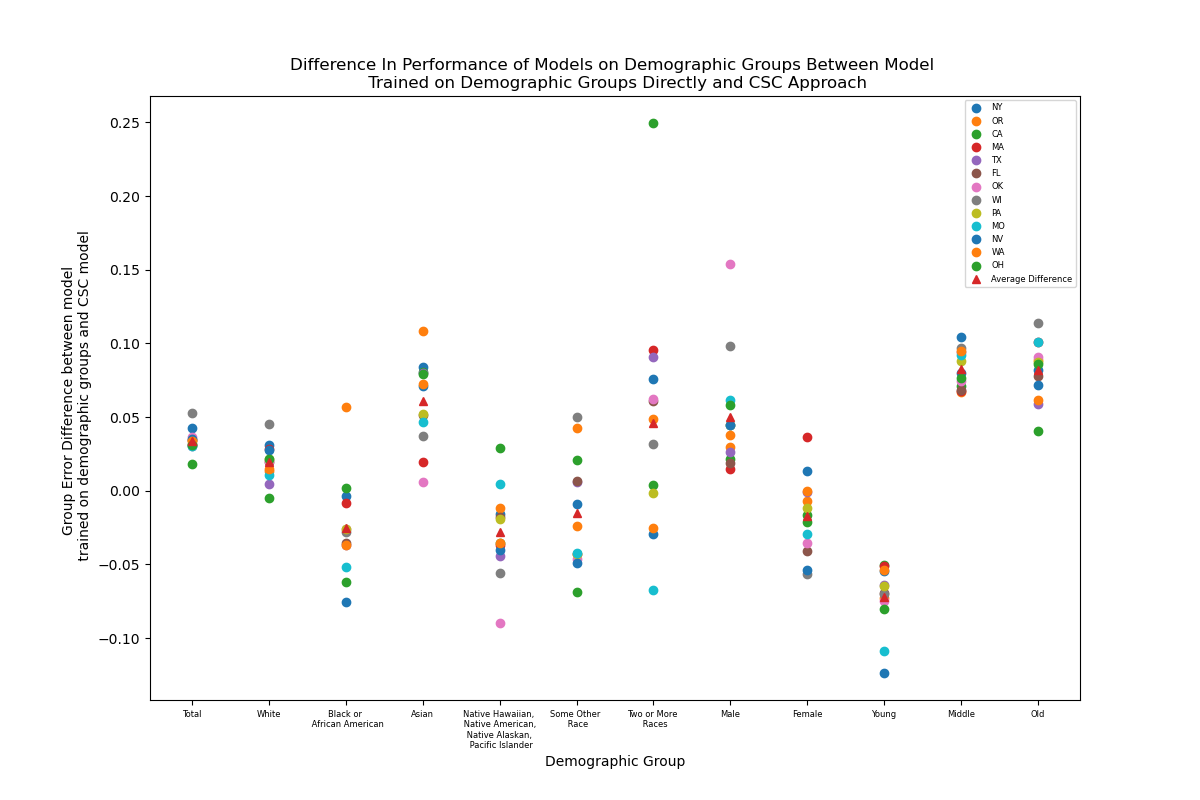}}
        \subfloat{\includegraphics[width=.45\columnwidth,trim={0.5cm 0.5cm 0.5cm 0.5cm}]{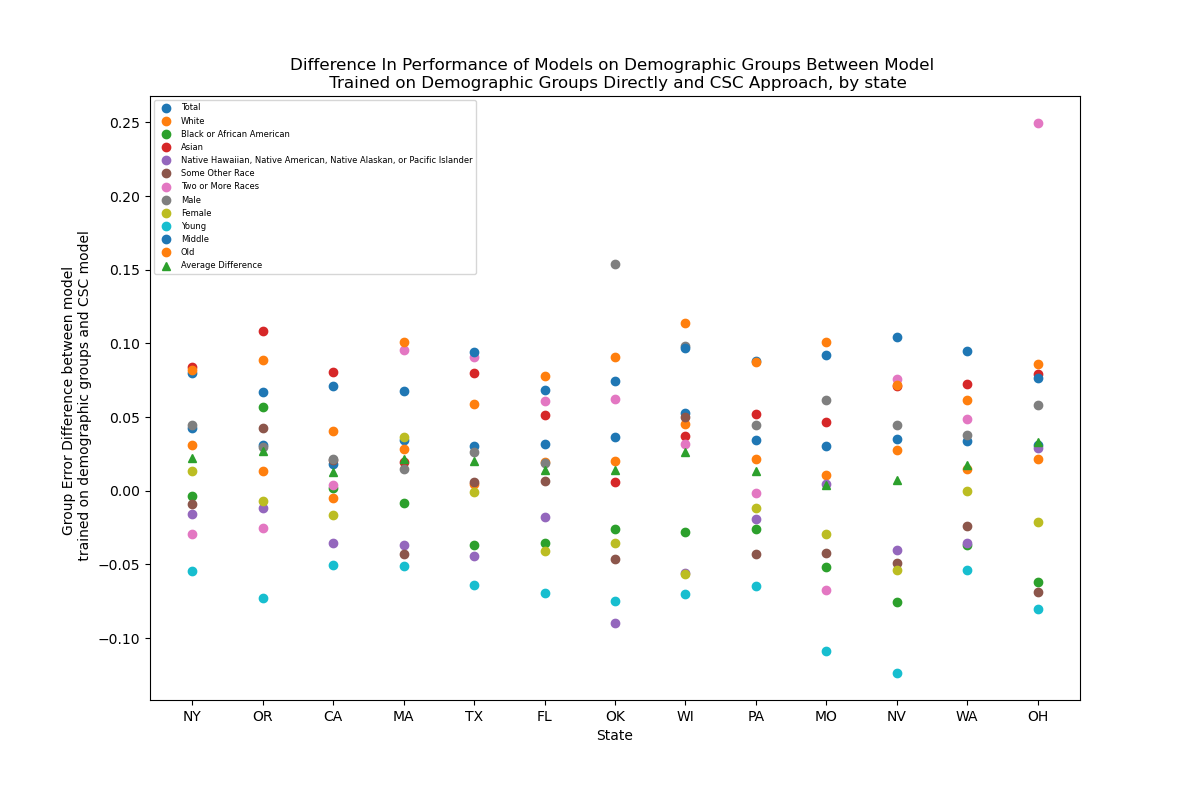}}
        \caption{Comparison of basic demographic subgroup test errors for sequential and CSC approaches. See text for details.}
    \label{fig:csc-errors}
\end{figure}
\else
\begin{figure*}[h!]
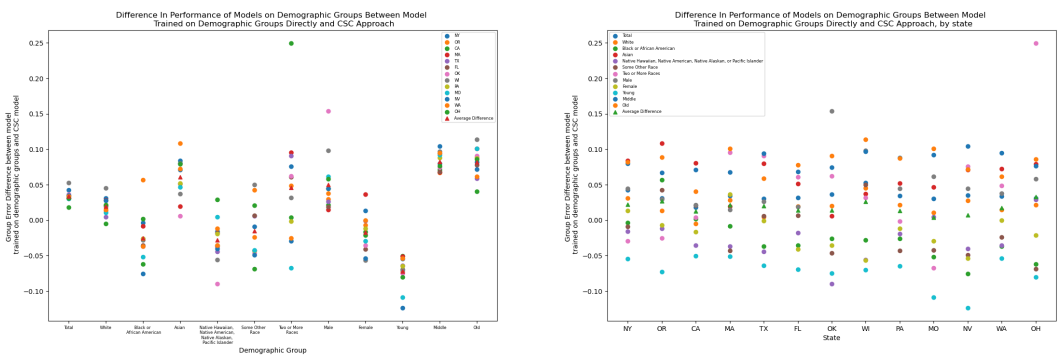

        \center
        \subfloat{\includegraphics[width=0.4\textwidth]{figures/error_diffs_groups.png}}
        \subfloat{\includegraphics[width=0.4\textwidth]{figures/state_demo_errs.png}}
        \caption{Comparison of basic demographic subgroup test errors for sequential and CSC approaches. See text for details.}
    \label{fig:csc-errors}
\end{figure*}
\fi 

\section{A Preliminary Deployment}
\label{sec:deployment}

We conclude by informally reporting findings from a preliminary deployment of our bias
bounty framework to a group of 83 students divided into 36 teams in an undergraduate class
on ethical algorithm design at the University of Pennsylvania in the Spring of 2022. In order
to expedite this deployment in a limited time window, we made a few departures from the formal framework. 
First, we did not implement the adaptive data analysis mechanisms to
prevent overfitting to the holdout set, but instead reported errors on the holdout set
regardless of whether a proposed update is accepted or not. Second, in order to avoid the
systems overhead of implementing a full client-server architecture and its potential
security mechanisms, each team was instead given a self-contained notebook implementing both
tools for finding $(g,h)$ pairs and the functionality for accepting updates and incorporating them
into a revised decision list. Each notebook contained access to one of two Folktables
datasets and classification tasks, which were divided roughly evenly between the teams. Students
were told there were absolutely no constraints on the approaches they could take to finding
improving $g,h$ pairs --- they could do so by human intuition, manual exploratory data analysis, or 
they could entirely
automate the process using ML or other techniques. They were generally
encouraged to find as many accepted updates as they could during the duration of the deployment.

Despite the very preliminary nature of this exercise, a number of informal and high-level
observations can nevertheless be made. First, despite wide variations in the quantitative
backgrounds of the students (the only prerequisite to the course being exposure to introductory
programming), virtually all of the most successful teams adopted hybrid approaches that involved some combination
of manual or data-driven discovery of groups along with more automated means of finding $g$, 
followed by training of $h$ on $g$. 

For instance, many groups automatically scanned through the feature space by column, or combinations of a small number of columns, to identify which values for each feature the current model performed poorly on, and then trained models on these subsets of the data. Note this has an advantage of interpretability over the groups generated by the algorithmic approaches developed in Section \ref{sec:restricted}, as the groups are always explicit functions of categorical values over a small number of features (e.g. ``Native American Women" or ``Young Men"). However, this method is also restricted to tabular data that has a nicely structured feature space: if instead students were tasked with bias hunting over a collection of images, given only the input pixel matrix, such a strategy would not be sufficient. 

Another notable feature of the exercise is that the purely algorithmic approaches appeared insufficient: groups who employed a combination of manual data analytics and algorithmic methods were able to bring errors down better than groups who employed purely algorithmic approaches similar to those in Section \ref{sec:restricted}. This emphasizes the fact that there is really an advantage to doing a true deployment of such a system within a community, as opposed to attempting to automatically identify all issues.

Some detail on the efforts of one of the most successful teams is provided in Table~\ref{tab:class}, where for each successive $(g,h)$
pair accepted, we show an English description of the group, and indication of whether it was discovered by (A)utomated
or (M)anual techniques, as well as the pre- and post-update group errors, the group weight, and the overall model error after each update.
A number of observations are in order. First, in general the earlier groups discovered are simpler and often manually
discovered, whereas the later ones tend to be more complex and automated. Indeed, the final group accepted is not 
even defined by features of the dataset at all,
but rather is trained on the errors of the overall model.
Second, echoing the increasing complexity and specificity of successive groups,
we see that there are generally diminishing marginal returns in both the group
weights and the reduction of overall model error. We might expect both increasing group complexity and diminishing returns
to be general features of bias bounty events conducted under our particular algorithmic framework.

\begin{table}
\centering
\resizebox{\columnwidth}{!}{%
\begin{tabular}{c l >{\centering\arraybackslash}m{.16\columnwidth} >{\centering\arraybackslash}m{.16\columnwidth} >{\centering\arraybackslash}m{.17\columnwidth} >{\centering\arraybackslash}m{.14\columnwidth} >{\centering\arraybackslash}m{.18\columnwidth}}
Round & Group Description & Group Error Pre Update & Group Error Post Update & Group Error Improvement & Group Weight & Overall Model Error \\
 &  &  &  &  &  & \\
0 &  &  &  &  & 1.0000 & 0.2796 \\
1 & male sex (M) & 0.2866 & 0.1734 & 0.1133 & 0.5238 & 0.2203 \\
2 & female sex (M) & 0.2719 & 0.1658 & 0.1061 & 0.4762 & 0.1698 \\
3 & ages 17 to 24 and non-white (M) & 0.0447 & 0.0414 & 0.0032 & 0.0526 & 0.1696 \\
4 & over the age of 30 working without pay in a family business (A) & 0.2258 & 0.1452 & 0.0806 & 0.0021 & 0.1694 \\
5 & self-employed in own not incorporated business, professional practice, or farm (A) & 0.2255 & 0.2205 & 0.0050 & 0.0884 & 0.1690 \\
6 & over the age of 62 (retired) (M) & 0.2322 & 0.2291 & 0.0031 & 0.1193 & 0.1686 \\
7 & First-line supervisors of office and administrative support workers (A) & 0.3073 & 0.2661 & 0.0413 & 0.0074 & 0.1683 \\
8 & real estate brokers and sales agents (A) & 0.3154 & 0.2905 & 0.0249 & 0.0082 & 0.1681 \\
9 & first-line supervisors of retail sales workers (A) & 0.2583 & 0.2494 & 0.0088 & 0.0154 & 0.1680 \\
10 & office clerks, general (A) & 0.2135 & 0.1842 & 0.0292 & 0.0117 & 0.1676 \\
11 & paint, coating, and adhesive manufacturing (A) & 0.1713 & 0.1657 & 0.0055 & 0.0185 & 0.1675 \\
12 & those born in California, Mexico, or Southeast Asia generally working as medical technicians (M) & 0.2126 & 0.2032 & 0.0094 & 0.0255 & 0.1671 \\
13 & accountants/Auditors who work 40 hour work weeks (A) & 0.2883 & 0.2793 & 0.0090 & 0.0076 & 0.1671 \\
14 & female sex secretaries and administrative assistants, except legal, medical, and executive (A) & 0.2586 & 0.2533 & 0.0053 & 0.0129 & 0.1670 \\
15 & machine learning classifier over errors on current model (A) & 0.5600 & 0.3200 & 0.2400 & 0.0009 & 0.1668 \\
\end{tabular}}
\caption{Details of the efforts of one particular team, see text for discussion.}
\label{tab:class}
\end{table}

\subsection*{Acknowledgements} We give warm thanks to Peter Hallinan, Pietro Perona and Alex Tolbert for many enlightening conversations at an early stage of this work,
and to Bri Cervantes and Declan Harrison for providing details of their class bias bounty efforts, and permission to include them here.

\newpage

\ifdraft
\bibliographystyle{plainnat}
\bibliography{refs}
\else
\bibliographystyle{ACM-Reference-Format}
\bibliography{refs}
\fi

\appendix
\input{appendix}
\end{document}

%% file: abstract.tex
\begin{abstract}

We propose and analyze an algorithmic framework for ``bias bounties'' --- events in which
external participants are invited to propose improvements to a trained model, akin to bug bounty
events in software and security. Our framework allows participants to submit arbitrary
subgroup improvements, which are then algorithmically incorporated into an updated model.
Our algorithm has the property that there is no tension between overall and subgroup
accuracies, nor between different subgroup accuracies, and it enjoys provable convergence
to either the Bayes optimal model or a state in which no further improvements can be
found by the participants. We provide formal analyses of our framework, experimental
evaluation, and findings from a preliminary bias bounty event.\footnote{The most up-to-date version of this paper may be found at \url{https://arxiv.org/abs/2201.10408}}

\end{abstract}

%% file: appendix.tex
\onecolumn

\section{Desiderata for a Bias Bounty Program}
\label{sec:desiderata}

Here we lay out a number of desiderata for a bias bounty program that would be run continuously and at scale, and how our method addresses these needs.

\begin{enumerate}
    \item The program should be able to accept a very large number of submissions on an ongoing basis. As a result, there should be a category of bounties for which the submissions should be able to be evaluated automatically. Our method suggests the following category: submissions that identify a group of non-negligible size such on which our existing model has performance that is worse than optimal by some non-negligible amount. 
    \item Since the program allows open participation and judges submissions automatically, it should be prepared for adversarial submissions, and be able to distinguish between genuine problems discovered on the underlying population or distribution, from cherrypicked examples designed to overfit the competition dataset. Our framework suggests a natural solution to this problem: bounty hunters submit a model describing a group $g$ on which performance is sub-optimal, and a model $h$ which demonstrates the sub-optimality of our existing model on $g$. Since $g$ and $h$ are themselves models, they can be evaluated on a holdout set to prevent overfitting, and we can use techniques from adaptive data analysis to safely re-use the same holdout set across many submissions \citep{DFHPRR1,DFHPRR2,DFHPRR3}.
    
    \item For bounties it awards automatically, it should be able to correct the problem identified automatically. Our framework naturally does this, given that bounty hunters submit pairs $(g,h)$, which are the objects we need to correct our model. 
    
    \item The process should be convergent in the sense that there should be no way to produce a sequence of submissions that guarantees the submitter an unlimited amount of money. Said another way, if we award bounties for identifying legitimate problems that we then go on to ``fix'', our fixes should actually make progress in some quantifiable way. The convergence analysis of our framework establishes this. 
    \item The model that results from the automatic patches resulting from the submitted bounties should be ``simple'' in the sense that the patching operation should itself be computationally easy (since it may need to be repeated many times), and evaluating the final model should also be computationally easy. The ``updates'' in our framework are elementary, and produces a simple object (a decision list composed of the models submitted) that is easy to evaluate. 
    \item Awarding a bounty (and correcting the corresponding problem) should not decrease the performance for any other group: fixing one problem should not introduce others. The monotone improvement property of our method satisfies this.  
    \item We should impose as few burdens on the ``bounty hunters'' as possible. Here, it might seem that our framework falls short because it places two kinds of burdens on ``bounty hunters'': first, it requires that they submit models $g$ that identify the group on which the current model performs poorly, rather than just identifying subsets of examples on which the model performs poorly. Second, it requires that they submit models $h$ that demonstrate improvements on these groups --- and perhaps the burden of improving on a group $g$ should be left to the organization deploying the model in the first place! But we note that these requirements are solving real problems that would otherwise arise:
    \begin{enumerate}
    \item As noted above, if bounty hunters simply submitted examples on which the current model performed poorly, rather than a model $g$ identifying such examples, then we would be unable to distinguish cherrypicked examples from real distributional problems. 
    \item If we did not require that bounty hunters submit a model $h$ that demonstrates improved performance on the group $g$ they submitted, then we would be subject to the following difficulty: it might be, for example, that the group $g$ that the bounty hunter identifies (correctly) as having high error corresponds to a group like ``blurry images'' on which improved performance is impossible. Problems of this sort cannot be fixed, and do not correspond to deviations from Bayes optimality. Requiring bounty hunters to submit a model $h$ that demonstrates improved performance disambiguates ``blurry images'' problems from real deviations from Bayes optimality.
\end{enumerate}
Of course, in a real system, we should strive to reduce the burden on the bounty hunter as much as possible. So, for example, we might provide an interface that allows the bounty hunter to identify a collection of examples on which the current model performs poorly, and then attempts to automatically train a model $g$ to capture examples like those identified by the bounty hunter. Similarly, given a group $g$, we could provide an interface that automatically attempts to use standard methods to train a good model $h$ on examples from $g$. Interfaces like this could make the bounty program accessible to a wider audience, without requiring any machine learning expertise.  
\end{enumerate}

\ifdraft
\else
\section{Proofs from Section \ref{sec:certificates}}
\BOQ*
\begin{proof}
  We need to prove two directions. First, we will assume that $f$ is $(\epsilon,\cC)$-Bayes optimal, and show that in this case there do not exist any pairs $(g,h) \in \cC$ such that $(g,h)$ form a $(\mu,\Delta)$-certificate of sub-optimality with $\mu\cdot \Delta > \epsilon$.  Fix a pair $(g,h) \in \cC$. Without loss of generality, we can take $\Delta = L(\cD,f,g) - L(\cD,h,g)$ (and if $\Delta \leq 0$ we are done, so we can also assume that $\Delta > 0$). Since $f$ is $(\epsilon,\cC)$-Bayes optimal, by definition we have that: 
$$\Delta =L(\cD,f,g) - L(\cD,h,g) \leq \frac{\epsilon}{\mu_g(\cD)}$$
Solving, we get that $\Delta\cdot \mu_g \leq \epsilon$ as desired.

Next, we prove the other direction: We assume that there exists a pair $(g,h) \in \cC$ that form an $(\mu, \Delta)$-certificate of sub-optimality, and show that $f$ is not $(\epsilon, \cC)$-Bayes optimal for any $\epsilon < \mu\cdot \Delta$. Without loss of generality we can take $\mu = \mu_g(\cD)$ and conclude:
 $$L(\cD,f,g) - L(\cD,h,g) \geq \Delta = \frac{\mu\cdot \Delta}{\mu_g(\cD)} \geq \frac{\epsilon}{\mu_g(\cD)}$$
which falsifies $(\epsilon, \cC)$-Bayes optimality for any $\epsilon < \mu\cdot \Delta$ as desired.
\end{proof}

\progress*
\begin{proof}
It is immediate from the definition of $f_{t+1}$ that $L(\cD,f_{t+1},g_{t+1}) = L(\cD, h_{t+1},g_{t+1})$, since for any $x$ such that  $g_{t+1}(x) = 1$, $f_{t+1}(x) = h_{t+1}(x)$. It remains to verify the 2nd condition. 
Because we also have that for every $x$ such that $g_{t+1}(x) = 0$, $f_{t+1}(x) = f_t(x)$, we can calculate:
\begin{eqnarray*}
L(\cD,f_{t+1}) &=& \Pr_{\cD}[g_{t+1}(x) = 0]\cdot \E_{\cD}[\ell(f_{t+1}(x), y) | g_{t+1}(x) = 0] + \Pr_{\cD}[g_{t+1}(x) = 1]\cdot \E_{\cD}[\ell(f_{t+1}(x), y) | g_{t+1}(x) = 1] \\
&=&  \Pr_{\cD}[g_{t+1}(x) = 0]\cdot \E_{\cD}[\ell(f_{t}(x), y) | g_{t+1}(x) = 0] + \Pr_{\cD}[g_{t+1}(x) = 1]\cdot \E_{\cD}[\ell(h_{t+1}(x), y) | g_{t+1}(x) = 1]\\
&\leq& \Pr_{\cD}[g_{t+1}(x) = 0]\cdot \E_{\cD}[\ell(f_{t}(x), y) | g_{t+1}(x) = 0] + \Pr_{\cD}[g_{t+1}(x) = 1]
\left(\E_{\cD}[\ell(f_{t}(x), y) | g_{t+1}(x) = 1] - \Delta \right)\\
&\leq& L(\cD, f_t) - \mu\Delta
\end{eqnarray*}
\end{proof}

\updatebound*
\begin{proof}
By assumption $L(\cD,f_0) = \ell_0$. Because each $(g_i,h_{i})$ is a $(\mu,\Delta)$-certificate of suboptimality of $f_{i-1}$ with $\mu\cdot \Delta \geq \epsilon$, we know from Theorem \ref{thm:progress} that for each $i$, $L(\cD,f_i) \leq L(\cD,f_{i-1}) - \epsilon$. Hence $L(\cD,f_T) \leq \ell_0 - T\epsilon$. But loss is non-negative: $L(\cD,f_T) \geq 0$. Thus it must be that $T \leq \frac{\ell_0}{\epsilon}$ as desired.
\end{proof}

\section{Material from Section \ref{sec:reusable}}
\certchecker*
\begin{proof}
We first consider any \emph{fixed} triple of functions $f_i:\cX\rightarrow \cY, h_i:\cX\rightarrow,\cY, g_i:\cX\rightarrow \{0,1\}$. Observe that we can write:
\begin{eqnarray*}
\mu_i\cdot \Delta_i &=& \mu_D(g_i)\cdot \left(
 L(D,f_i,g_i) - L(D,h_i,g_i) \right) \\
 &=& \frac{\sum_{(x,y)\in D}\mathbbm{1}[g_i(x) = 1]}{n} \cdot \frac{\sum_{(x,y) \in D} \mathbbm{1}[g_i(x) = 1]\cdot \left(\ell(f_i(x),y) - \ell(h_i(x),y)\right)}{\sum_{(x,y) \in D} \mathbbm{1}[g_i(x) = 1]} \\
 &=& \frac{1}{n} \sum_{(x,y) \in D} \mathbbm{1}[g_i(x) = 1]\cdot \left(\ell(f_i(x),y) - \ell(h_i(x),y)\right)
\end{eqnarray*}
Since each $(x,y) \in D$ is drawn independently from $\cD$, each term in the sum $\mathbbm{1}[g_i(x) = 1]\cdot \left(\ell(f_i(x),y) - \ell(h_i(x),y)\right)$ is an independent random variable taking value in the range $[-1,1]$. Thus $\mu_i\cdot \Delta_i$ is the average of $n$ independent bounded random variables and we can apply a Chernoff bound to conclude that for any value of $\delta' > 0$: 
$$\Pr_{D \sim \cD^n}\left[\left| \mu_i\cdot \Delta_i -\mu_\cD(g_i)\cdot \left(
 L(\cD,f_i,g_i) - L(\cD,h_i,g_i)\right) \right|  \geq \sqrt{\frac{2\ln(2/\delta')}{n}}\right] \leq \delta'$$
 Solving for $n$ we have that with probability $1-\delta'$, we have $\left| \mu_i\cdot \Delta_i -\mu_\cD(g_i)\cdot \left(
 L(\cD,f_i,g_i) - L(\cD,h_i,g_i)\right) \right| \leq \frac{\epsilon}{4}$ if:
 $$n \geq \frac{32 \ln(2/\delta')}{\epsilon^2}$$
 This analysis was for a fixed triple of functions $(f_i,h_i,g_i)$, but these triples  can be chosen arbitrarily as a function of the transcript $\pi$. We therefore need to count how many transcripts $\pi$ might arise. By construction, $\pi$ has length at most $U$ and has at most $2/\epsilon$ indices such that $\pi_i = \top$. Thus the number of transcripts $\pi$ that can arise is at most: ${U \choose \frac{2}{\epsilon}}\cdot 2^{2/\epsilon} \leq U^{2/\epsilon}$, and each transcript results in some sequence of $U$ triples $(f_i,h_i,g_i)$. Thus for any mechanism for generating triples from transcript prefixes, there are at most $U^{2/\epsilon + 1}$ triples that can ever arise. We can complete the proof by union bounding over this set. Taking $\delta' = \frac{\delta}{U^{2/\epsilon + 1}}$ and plugging into our Chernoff bound above, we obtain that with probability $1-\delta$ over the choice of $D$, for any method of generating a sequence of $U$ triples $\{(f_i,h_i,g_i)\}_{i=1}^U$ from transcripts $\pi$, we have that:
 $\max_i\left| \mu_i\cdot \Delta_i -\mu_\cD(g_i)\cdot \left(
 L(\cD,f_i,g_i) - L(\cD,h_i,g_i)\right) \right| \leq \frac{\epsilon}{4}$ so long as:
  $$n \geq \frac{32(\frac{2}{\epsilon} + 1) \ln(2U/\delta')}{\epsilon^2} \geq \frac{65 \ln(2U/\delta')}{\epsilon^3}$$
  
 Finally, note that whenever this event obtains, the conclusions of the theorem hold, because we have that $\pi_i = \top$ exactly when $\mu_i\cdot \Delta_i \geq \frac{3\epsilon}{4}$. In this case, $\mu_\cD(g_i)\cdot \left(
 L(\cD,f_i,g_i) - L(\cD,h_i,g_i)\right) \geq \frac{3\epsilon}{4} - \frac{\epsilon}{4} = \frac{\epsilon}{2}$ as desired. Similarly, whenever $\pi_i = \bot$, we have that $\mu_\cD(g_i)\cdot \left(
 L(\cD,f_i,g_i) - L(\cD,h_i,g_i)\right) \leq \frac{3\epsilon}{4} + \frac{\epsilon}{4} = \epsilon$ as desired.
\end{proof}

\bountySC*
\begin{proof}
  This theorem follows straightforwardly from the properties of Algorithm \ref{alg:basic} and Algorithm \ref{alg:checker}. From Theorem \ref{thm:adaptive}, we have that with probability $1-\delta$, every submission accepted by CertificateChecker (and hence by FalsifyandUpdate) is a $(\mu,\Delta)$-certificate of sub-optimality for $f_t$ with $\mu\cdot \Delta \geq \epsilon/2$ and every submission rejected is not a $(\mu,\Delta)$-certificate of sub-optimality for any $\mu, \Delta$ with $\mu\cdot \Delta \geq \epsilon$. 
  
  Whenever this event obtains, then for every call that FalsifyAndUpdate makes to $\textrm{ListUpdate}(f_{t-1},(g_i,h_i))$ is such that $(g_i,h_i)$ is a $(\mu,\Delta)$-certificate of sub-optimality for $f_{t-1}$ for $\mu\cdot\Delta \geq \epsilon/2$. Therefore by Theorem \ref{thm:progress}, we have that $L(\cD,f_{t+1}, g_i) = L(\cD, h_i, g_i)$ and $L(\cD,f_{t+1}) \leq L(\cD,f_t) - \frac{\epsilon}{2}$. Finally, by Theorem \ref{thm:updatebound}, if each invocation of  the iteration $f_t = \textrm{ListUpdate}(f_{t-1},(g_i,h_i))$ is such that $(g_i,h_i)$ is a $(\mu,\Delta)$-certificate of sub-optimality for $f_{t-1}$ with $\mu\cdot \Delta \geq \epsilon/2$, then there can be at most $2/\epsilon$ such invocations. Since FalsifyAndUpdate makes one such invocation for every submission that is accepted, this means there can be at most $2/\epsilon$ submissions accepted in total. But CertificateChecker has only two halting conditions: it halts when either more than $2/\epsilon$ submissions are accepted, or when $U$ submissions have been made in total. Because with probability at least $1-\delta$ no more than $2/\epsilon$ submissions are accepted, it must be that with probability $1-\delta$, FalsifyAndUpdate does not halt until all $U$ submissions have been received. 
\end{proof}

\bountymonotone*
\begin{proof}
The proof that MonotoneFalsifyAndUpdate satisfies the first two conclusions of Theorem \ref{thm:bounty-SC}:
\begin{enumerate}
  \item If $(g_i,h_i)$ is rejected, then $(g_i,h_i)$ is not a $(\mu,\Delta)$-certificate of sub-optimality for $f_t$, where $f_t$ is the current model at the time of submission $i$, for any $\mu, \Delta$ such that $\mu\cdot \Delta \geq \epsilon$. 
    \item If $(g_i,h_i)$ is accepted, then $(g_i,h_i)$ is a $(\mu,\Delta)$-certificate of sub-optimality for $f_t$, where $f_t$ is the current model at the time of submission $i$, for some $\mu, \Delta$ such that $\mu\cdot \Delta \geq \frac{\epsilon}{2}$. Moreover, the new model $f_{t+1}$ output satisfies  $L(\cD,f_{t+1}) \leq L(\cD,f_t) - \frac{\epsilon}{2}$.
\end{enumerate}
are identical and we do not repeat them here. We must show that with probability $1-\delta$, CertificateChecker (and hence MonotoneFalsifyAndUpdate) does not halt before processing all $U$ submissions. Note that MonotoneFalsifyAndUpdate initializes an instance of CertificateChecker that will not halt before receiving $U + \frac{8}{\epsilon^3}$ many submissions. Thus it remains to verify that our algorithm does not produce more than $8/\epsilon^3$ many submissions to CertificateChecker in its monotonicity update process. But this will be the case, because by Theorem \ref{thm:updatebound}, $t \leq \frac{2}{\epsilon}$, and so after each call to ListUpdate, we generate at most $4/\epsilon^2$ many submissions to CertificateChecker. Since there can be at most $2/\epsilon$ such calls to ListUpdate, the claim follows. 

To see that the monotonicity property holds, assume for sake of contradiction that it does not --- i.e. that there is a model $f_t$, a group $g_j \in \cG(f_t)$, and a model $f_\ell$ with $\ell < t$ such that:
$$L(\cD,f_t,g_j) >  L(\cD, f_{\ell}, g_j) + \frac{\epsilon}{\mu_\cD(g_j)}$$
In this case, the pair $(g_j,f_{\ell})$ would form a $(\mu,\Delta)$-certificate of sub-optimality for $f_t$ with $\mu\cdot \Delta \geq \epsilon$. But if $L(\cD,f_t,g_j) >  L(\cD, f_{\ell}, g_j) + \frac{\epsilon}{\mu_\cD(g_j)}$, then this certificate must have been rejected, which we have already established is an event that occurs with probability at most $\delta$.
\end{proof}

\subsection{Certificates of Bounded Complexity and Algorithmic Optimization}
\label{app:algopt}
In this section we show how to use our ListUpdate method as part of an algorithm for explicitly computing $(\epsilon,\cC)$-Bayes optimal models from data sampled i.i.d. from the underlying distribution $\cD$. First, we must show that if we find certificates of sub-optimality $(g,h) \in C$ on our dataset, that we can be assured that they are certificates of sub-optimality on the underlying distribution. Here, we invoke uniform convergence bounds that depend on $\cC$ being a class of bounded complexity. Next, we must describe an algorithmic method for finding $(\mu,\Delta)$ certificates of sub-optimality that maximize $\mu\cdot \Delta$. Here we give two approaches. The first approach is a reduction to cost sensitive (ternary) classification: the result of the reduction is that the ability to solve weighted multi-class classification problems over some class of models gives us the ability to find certificates of sub-optimality over a related class whenever they exist. The second approach takes an ``EM'' style alternating maximization approach over $g_i \in \cG$ and $h_i \in \cH$ in turn, where each alternating maximization step can be reduced to a binary classification problem. It is only guaranteed to converge to a local optimum (or saddlepoint) of its objective --- i.e. to find a certificate of sub-optimality $(g_i,h_i)$ that cannot be improved by changing either $g_i$ or $h_i$ unilaterally --- but has the merit that it requires only standard binary classification algorithms for a class $\cG$ and $\cH$ to search for certificates of sub-optimality in $\cG\times \cH$. For simplicity of exposition, in this section we restrict attention to the binary classification problem, where the labels are binary $(\cY = \{0,1\}$) and our loss function corresponds to classification error $(\ell(\hat y, y) = \mathbbm{1}[\hat y \neq y])$ --- but the approach readily extends to more general label sets (replacing VC-dimension with the appropriate notion of combinatorial dimension as necessary).

The algorithmic problem we need to solve at each round $t$ is, given an existing model $f_{t-1}$, find a $(\mu,\Delta)$-certificate of optimality $(g,h) \in \cC$ for $f_{t-1}$ that maximizes $\mu\cdot \Delta$ as computed on the empirical data $D$ --- i.e. to solve:
\begin{equation}
\label{eq:optproblem}
(g_t,h_t) =\arg\max_{(g,h) \in \cC }\mu_{D}(g)\cdot (L(D,f_{t-1},g) - L(D,h,g))
\end{equation}

We defer for now the algorithmic problem of finding these certificates, and describe a generic algorithm that can be invoked with any method for finding such certificates. First, we state a useful sample complexity bound proven in \citep{gerrymandering} in a related context of multi-group fairness --- we here state the adaptation to our setting.
\begin{lemma}[Adapted from \citep{gerrymandering}]
\label{lem:samplecomplexity}
Let $\cG$ denote a class of group indicator functions with VC-dimension $d_G$ and $\cH$ denote a class of binary models with VC-dimension $d_H$.  Let $f$ be an arbitrary binary model. Then if:
 $$n \geq \tilde O\left(\frac{(d_H + d_G) + \log(1/\delta)}{\eta^2} \right)$$
 we have that with probability $1-\delta$ over the draw of a dataset $D \sim \cD^n$, for every $g \in \cG$ and $h \in \cH$:
 $$\left|\mu_{\cD}(g)\cdot (L(\cD,f,g) - L(\cD,h,g)) - \mu_{D}(g_p)\cdot (L(D,f,g) - L(D,h,g)) \right| \leq \eta$$
\end{lemma}

With our sample complexity lemma in hand we are ready to present our generic reduction from training an $(\epsilon, \cC)$-Bayes optimal model to the optimization problem over $\cC$ given in \eqref{eq:optproblem}.
It is a reduction from the problem of training an $(\epsilon,\cC)$-Bayes optimal model to the optimization problem over $\cC$ given in \eqref{eq:optproblem}.

\begin{algorithm}
\KwInput{A dataset $D$, a set of group/model pairs $\cC$, and accuracy parameter $\epsilon$.}
Let $f_0:\cX\rightarrow \{0,1\}$ be an arbitrary initial model and $T = 2/\epsilon$ \\

Randomly divide $D$ into $T$ equally sized datasets: $D_1,\ldots,D_T$.

\For{$t = 1$ to $T$}{
Let
$$(g_t,h_t) =\arg\max_{(g,h) \in \cC }\mu_{D}(g)\cdot (L(D,f_{t-1},g) - L(D,h,g))$$

\If{$\mu_{D}(g_t)\cdot (L(D,f_{t-1},g_t) - L(D,h_t,g_t)) \leq \frac{3\epsilon}{4}$}{
\KwOut{Model $f_{t-1}$}
}
\Else{
Let $f_t = \textrm{ListUpdate}(f_{t-1},(g_t,h_t))$
}
}
\KwOut{Model $f_T$}

\caption{TrainByOpt($D,\cC,\epsilon)$: An algorithm for training an $(\epsilon,\cC)$-Bayes optimal model given the ability to optimize over certificates $(g,h) \in \cC$.}
\label{alg:TrainByOpt}
\end{algorithm}

\begin{restatable}{theorem}{samplecomplexity}
Fix an arbitrary distribution $\cD$ over $\cX\times \cY$, a class of group indicator functions $\cG$ of VC-dimension $d_G$ and a class of binary models $\cH$ of VC-dimension $d_H$. Let $\cC \subseteq \cG \times \cH$ and $\epsilon > 0$ be arbitrary.  If:
$$n \geq \tilde O\left(\frac{(d_H + d_G) + \log(1/\delta)}{\epsilon^3} \right)$$
and $D \sim \cD^n$,
then with probability $1-\delta$,  TrainByOpt$(D,\cC,\epsilon)$ (Algorithm \ref{alg:TrainByOpt}) returns a model $f$ that is $(\epsilon,\cC)$-Bayes optimal, using at most $2/\epsilon$ calls to a sub-routine for solving the optimization problem over $(g,h) \in \cC$ given in \eqref{eq:optproblem}.
\end{restatable}

\begin{proof}
  Each of the partitioned datasets $D_t$ has size at least $\tilde O\left(\frac{(d_H + d_G) + \log(1/\delta)}{\epsilon^2} \right)$ and is selected independently of $f_{t-1}$ and so we can invoke Lemma \ref{lem:samplecomplexity} with $\eta = \epsilon/8$ and $\delta = \delta/T$ to conclude that with probability $1-\delta$, for every round $t$:
  $$\left|\mu_{D}(g_t)\cdot (L(D,f_{t-1},g_t) - L(D,h_{t},g_t)) -\mu_{\cD}(g_t)\cdot (L(\cD,f_{t-1},g_t) - L(\cD,h_{t},g_t))\right| \leq \frac{\epsilon}{8}$$
  For the rest of the argument we will assume this condition obtains.
 By assumption, at every round $t$ the models $(g_t,h_t)$ exactly maximize $\mu_{D}(g_t)\cdot (L(D,f_{t-1},g_t) - L(D,h_{t},g_t))$ amongst all models $(g,h) \in \cC$, and so in combination with the above uniform convergence bound, they are $\epsilon/4$-approximate maximizers of $\mu_{\cD}(g_t)\cdot (L(\cD,f_{t-1},g_t) - L(\cD,h_{t},g_t))$. Therefore, if at any round $t \leq T$ the algorithm outputs a model $f_{t-1}$, it must be because:
 $$\max_{(g,h) \in \cC}\mu_{\cD}(g)\cdot (L(\cD,f_{t-1},g) - L(\cD,h,g)) \leq \frac{3\epsilon}{4} + \frac{\epsilon}{4} = \epsilon$$
 Equivalently, for every $(g,h) \in \cC$, it must be that:
 $$L(\cD,f_{t-1},g) \leq L(\cD,h,g) + \frac{\epsilon}{\mu_{\cD}(g)}$$
 By definition, such a model is $(\epsilon,\cC)$-Bayes optimal.

Similarly, if at any round $t < T$ we do not output $f_{t-1}$, then it must be that $(g_t,h_t)$ forms a $(\mu,\Delta)$-certificate of sub-optimality for $\mu = \mu_{\cD}(g_t)$ and $\Delta = L(\cD,f_{t-1},g) - L(\cD,h,g)$ such that $\mu\cdot \Delta \geq \frac{3\epsilon}{4} - \frac{\epsilon}{4} = \frac{\epsilon}{2}$. By Theorem \ref{thm:updatebound} we have that for any sequence of models $f_0, f_1, \ldots, f_T$ such that $f_t = \textrm{ListUpdate}(f_{t-1},(g_t,h_t))$ and $(g_t,h_t)$ forms a $(\mu,\Delta)$-certificate of sub-optimalty for $f_{t-1}$ with $\mu\cdot \Delta \geq \frac{\epsilon}{2}$, it must be that $T \leq \frac{2}{\epsilon}$. Therefore, it must be that if our algorithm outputs model $f_T$, then $f_T$ must be $\frac{\epsilon}{2}$-Bayes optimal. If this were not the case, then by Theorem \ref{thm:BOQ}, there would exist a $(\mu,\Delta)$-certificate of sub-optimality $(g^*,h^*)$ for $f_T$ with $\mu\cdot \Delta \geq \frac{\epsilon}{2}$, which we could use to extend the sequence by setting $f_{T+1} = \textrm{ListUpdate}(f_T,(g^*,h^*))$ --- but that would contradict Theorem \ref{thm:updatebound}. Since $\epsilon/2$-Bayes optimality is strictly stronger than $(\epsilon,\cC)$-Bayes optimality, this completes the proof.
\end{proof}

Algorithm \ref{alg:TrainByOpt} is an efficient algorithm for training an $(\epsilon, \cC)$-Bayes optimal model whenever we can solve optimization problem \eqref{eq:optproblem} efficiently. We now turn to this optimization problem. In Section \ref{sec:ternary} we show that the ability to solve cost sensitive classification problems over a \emph{ternary} class $K$ gives us the ability to solve optimization problem \ref{eq:optproblem} over a related class $\cC_K$. In Section \ref{sec:EM} we show that the ability to solve standard empirical risk minimization problems over classes $\cG$ and $\cH$ respectively give us the ability to run iterative updates of an alternating-maximization (``EM style'') approach to finding certificates $(g,h) \in \cG \times \cH$.

\subsubsection{Finding Certificates via a Reduction to Cost-Sensitive Ternary Classification}
\label{sec:ternary}

For this approach, we start with an arbitrary class $K$ of \emph{ternary} valued functions $p:\cX\rightarrow \{0,1,?\}$. It will be instructive to think of the label ``$?$'' as representing the decision to ``defer'' on an example, leaving the classification outcome to another model. We will identify such a ternary-valued function $p$ with a pair of binary valued functions $g_p:\cX\rightarrow \{0,1\}, h_p:\cX\rightarrow \{0,1\}$ representing a group indicator function and a binary model $h$ that might form a certificate of sub-optimality $(g_p,h_p)$. They are defined as follows:
\begin{definition}
Given a ternary valued function $p:\cX\rightarrow \{0,1,?\}$, the $p$-derived group $g_p$ and model $h_p$ are defined as:

$$g_p(x) = \begin{cases}
      1 & \text{if } p(x) \in \{0,1\} \\
      0 & \text{if } p(x) = ?
  \end{cases} \ \ \ \ \ \ \  \ h_p(x) = \begin{cases}
      p(x) & \text{if } p(x) \in \{0,1\} \\
      0 & \text{if } p(x) = ?
  \end{cases}$$

In other words, interpreting $p(x) = ?$ as the decision for $p$ to ``defer'' on $x$, $g_p$ defines exactly the group of examples that $p$ does \emph{not} defer on, and $h_p$ is the model that makes the same prediction as $p$ on every example that $p$ does not defer on. Given a class of ternary functions $K$, let the $K$-derived certificates $\cC_K$ denote the set of pairs $(g_p,h_p)$ that can be so derived from some $p \in K$: $\cC_K = \{(g_p,h_p) : p \in K\}$. Similarly let $\cG_K = \{g_p : p \in K\}$ and $\cH_K = \{h_p : p \in K\}$ denote the class of group indicator functions and models derived from $K$ respectively.
\end{definition}

Given a model $f:\cX\rightarrow \{0,1\}$, our goal is to reduce the problem of solving optimization problem \eqref{eq:optproblem} over $\cC_K$ to the problem of solving a ternary \emph{cost-sensitive classification problem} over $K$:
\begin{definition}
A cost-sensitive classification problem is defined by a model class $K$ consisting of functions $p:\cX\rightarrow \cY$, where $\cY$ is some finite label set,  a distribution $\cD$ over $\cX\times \cY$, and a set of real valued costs $c_{(x,y)}(\hat y)$ for each pair $(x, y)$ in the support of $\cD$ and each label $\hat y \in \cY$. A solution $p^* \in K$ to the cost-sensitive classification problem $(\cD, K, \{c_{(x,y)}\})$ is given by:
$$p^* \in \arg\min_{p \in K} \E_{(x,y) \sim \cD}\left[ c_{(x,y)}(p(x))\right]$$
i.e. the model that minimizes the expected costs for the labels it assigns to points drawn from $\cD$.
\end{definition}

The reduction will make use of the following induced costs:
\begin{definition}
Given a binary model $f:\cX\rightarrow \cY$, the induced costs of $f$  are defined as follows:
$$c^f_{(x,y)}(\hat y) = \begin{cases}
      0 & \text{if } \hat y = ? \\
      1 & \text{if } f(x) = y \neq \hat y \\
      -1 & \text{if } \hat y = y \neq f(x) \\
      0 & \text{otherwise.}
  \end{cases}$$
Intuitively, it costs nothing to defer a decision to the existing model $f$ --- or equivalently, to make the same decision as $f$. On the other hand, making the \emph{wrong} decision on an example $x$ costs $1$ when $f$ would have made the right decision, and making the right decision ``earns'' $1$ when $f$ would not have.
\end{definition}

\begin{restatable}{theorem}{lemCSC}
\label{lem:CSC}
Fix an arbitrary distribution $\cD$ over $\cX\times \cY$, let $K$ be a class of ternary valued functions, and let $f:\cX\rightarrow \{0,1\}$ be any binary valued model. Let $p^*$ be a solution to the cost-sensitive classification problem $(\cD, K, \{c_{(x,y)}^f\})$, where $c_{(x,y)}^f$ are the induced costs of $f$. We have that:
$$(g_{p^*},h_{p^*}) \in \arg\max_{(g,h) \in \cC_K }\mu_{\cD}(g)\cdot (L(\cD,f,g) - L(\cD,h,g))$$
In other words, when $\cD$ is the empirical distribution over $D$, $(g_{p^*},h_{p^*})$ form a solution to optimization problem \eqref{eq:optproblem}.
\end{restatable}
\begin{proof}
 For any model $p \in K$, we can calculate its expected cost under the induced costs of $f$:
 \begin{eqnarray*}
  \E_{(x,y) \sim \cD}\left[ c_{(x,y)}(p(x))\right] &=& \E_{(x,y) \sim \cD}\left[\mathbbm{1}[p(x) \neq ?]\cdot \mathbbm{1}[p(x) \neq f(x)]\cdot (\mathbbm{1}[p(x) \neq y]-\mathbbm{1}[p(x) = y]) \right] \\
  &=&  \E_{(x,y) \sim \cD}\left[\mathbbm{1}[g_p(x) = 1]\cdot \mathbbm{1}[h_p(x) \neq f(x)](\mathbbm{1}[h_p(x) \neq y]-\mathbbm{1}[h_p(x) = y]) \right] \\
  &=& \mu_{\cD}(g_p)\cdot (L(\cD,h_p,g_p) - L(\cD,f,g_p)) \\
\end{eqnarray*}
Thus \emph{minimizing}   $\E_{(x,y) \sim \cD}\left[ c_{(x,y)}(p(x))\right]$ is equivalent to \emph{maximizing} $\mu_{\cD}(g_p)\cdot (L(\cD,f,g_p)-L(\cD,h_p,g_p))$ over $p \in K$.
\end{proof}

In other words, in order to be able to efficiently implement algorithm \ref{alg:TrainByOpt} for the class $\cC_K$, it suffices to be able to solve a ternary cost sensitive classification problem over $K$.  The up-shot is that if we can efficiently solve weighted multi-class classification problems (for which we have many algorithms which form good heuristics) over $K$, then we can find approximately $\cC_K$-Bayes optimal models as well.

\subsubsection{Finding Certificates Using Alternating Maximization}
\label{sec:EM}
The reduction from optimization problem \eqref{eq:optproblem} that we gave in Section \ref{sec:ternary} to ternary cost sensitive classification starts with a ternary class $K$ and then finds certificates of sub-optimality over a derived class $\cC_K$. What if we want to \emph{start} with a pre-defined class of group indicator functions $\cG$ and models $\cH$, and find certificates of sub-optimality $(g,h) \in \cG\times \cH$? In this section we give an alternating maximization method that attempts to solve optimization problem \ref{eq:optproblem} by alternating between maximizing over $g_t$ (holding $h_t$ fixed), and maximizing over $h_t$ (holding $g_t$ fixed):

\begin{equation}
\label{eq:gmin}
g_t=\arg\max_{g\in \cG }\mu_{D}(g)\cdot (L(D,f_{t-1},g) - L(D,h_t,g))
\end{equation}
\begin{equation}
\label{eq:hmin}
h_t =\arg\max_{h \in \cH }\mu_{D}(g_t)\cdot (L(D,f_{t-1},g_t) - L(D,h,g_t))
\end{equation}

We show that each of these alternating maximization steps can be reduced to solving a standard (unweighted) empirical risk minimization problem over $\cG$ and $\cH$ respectively. Thus, each can be solved for any heuristic for standard machine learning problems --- we do not even require support for weighted examples, as we do in the cost sensitive classification approach from Section \ref{sec:ternary}. This alternating maximization approach quickly converges to a local optimum or saddle point of the optimization objective from from \eqref{eq:optproblem} --- i.e. a solution that cannot be improved by either a unilateral change of either $g_t \in \cG$ or $h_t \in \cH$. 

We begin with the minimization problem \eqref{eq:hmin} over $h$, holding $g_t$ fixed, and show that it reduces to an  empirical risk minimization problem over $\cH$.

\begin{restatable}{lemma}{lemhopt}
\label{lem:hopt}
 Fix any $g_t \in \cG$ and dataset $D \in \cX^n$. Let $D_{g_t} = \{(x,y) \in \cD : g_t(x) = 1\}$ be the subset of $D$ consisting of members of group $g$. Let $h^* = \arg\min_{h \in \cH} L(D_{g_t},h)$. Then $h^*$ is a solution to optimization problem \ref{eq:hmin}.
\end{restatable}
\begin{proof}
We observe that only the final term of the optimization objective in \eqref{eq:hmin} has any dependence on $h$ when $g_t$ is held fixed. Therefore:
\begin{eqnarray*}
\arg\max_{h \in \cH }\mu_{D}(g_t)\cdot (L(D,f_{t-1},g_t) - L(D,h,g_t)) &=& \arg\max_{h \in \cH }-\mu_{D}(g_t) \cdot L(D,h,g_t)) \\
&=& \arg\min_{h \in \cH} L(D,h,g_t) \\
&=& \arg\min_{h \in \cH} L(D_{g_t},h)
\end{eqnarray*}
\end{proof}

Next we consider the minimization problem \eqref{eq:gmin} over $g$, holding $h_t$ fixed and show that it reduces to an empirical risk minimization problem over $\cG$. The intuition behind the below construction is that the only points $x$ that matter in optimizing our objective are those on which the models $f_{t-1}$ and $h_t$ disagree. Amongst these points $x$ on which the two models disagree, we want to have $g(x) = 1$ if $h$ correctly predicts the label, and not otherwise.
\begin{restatable}{lemma}{lemgopt}
\label{lem:gopt}
 Fix any $h_t \in \cH$ and dataset $D \in \cX^n$. Let:

 $$D_{h_t}^1 = \{(x,y) \in D : h_t(x) = y \neq f_{t-1}(x)\} \ \ \ \ D_{h_t}^0 = \{(x,y) \in D : h_t(x) \neq y = f_{t-1}(x)\}$$

$$D_{h_t} = \left(\bigcup_{(x,y) \in D_{h_t}^1} (x, 1)\right) \cup \left(\bigcup_{(x,y) \in D_{h_t}^0} (x, 0)\right)$$

Let $g^* = \arg\min_{g \in \cG} L(D_{h_t},g)$. Then $g^*$ is a solution to optimization problem \eqref{eq:gmin}.
\end{restatable}

\begin{proof}
  For each $g \in \cG$, we partition the set of points $(x,y) \in D_{h_t}$ such that $g(x) = 1$ according to how they are labelled by $f_{t_1}$ and $h_t$:
  $$S_1(g) = \{(x,y) \in D_{h_t} : g(x) = 1, f_{t-1}(x) = h_t(x) = y\} \ \ S_2(g) = \{(x,y) \in D_{h_t} : g(x) = 1, f_{t-1}(x) = h_t(x) \neq y\}$$
    $$S_3(g) = \{(x,y) \in D_{h_t} : g(x) = 1, f_{t-1}(x) \neq h_t(x) = y\} \ \ S_4(g) = \{(x,y) \in D_{h_t} : g(x) = 1, f_{t-1}(x) =  y \neq h_t(x)\}$$

    i.e. amongst the points in group $g$, $S_1(g)$ consists of the points that both $f_{t-1}$ and $h_t$ classify correctly, $S_2(g)$ consists of the points that both classify incorrectly, and $S_3(g)$ and $S_4(g)$ consist of points that $f_{t-1}$ and $h_t$ disagree on: $S_3(g)$ are those points that $h_t$ classifies correctly, and $S_4(g)$ are those points that $f_{t-1}$ classifies correctly. Write:
    $$w_1(g) = \frac{|S_1(g)|}{|D_{h_t}|} \ \ \ w_2(g) = \frac{|S_2(g)|}{|D_{h_t}|} \ \ \ w_3(g)  = \frac{|S_3(g)|}{|D_{h_t}|} \ \ \ w_4(g) = \frac{|S_4(g)|}{|D_{h_t}|}$$
to denote the corresponding proportions of each of the sets $S_i$ within $D_{h_t}$.

Now observe that $\mu_D(g) = w_1(g) + w_2(g) + w_3(g) + w_4(g)$, $L(D,f_{t-1},g) = \frac{w_2(g)+w_3(g)}{\mu_D(g)}$, and $L(D,h_t,g) = \frac{w_2(g)+w_4(g)}{\mu_D(g)}$. Therefore we can rewrite the objective of optimization problem \eqref{eq:gmin} as:
\begin{eqnarray*}
\arg\max_{g\in \cG }\mu_{D}(g)\cdot (L(D,f_{t-1},g) - L(D,h_t,g)) &=& \arg\max_{g\in \cG }(w_2(g) + w_3(g)) - (w_2(g) + w_4(g)) \\
&=& \arg\max_{g\in \cG } w_3(g) - w_4(g) \\
&=& \arg\min_{g\in \cG } w_4(g) - w_3(g) \\
&=& \arg\min_{g \in \cG} L(D_{h_t},g)
\end{eqnarray*}
\end{proof}

With these two components in hand, we can describe our alternating maximization algorithm for finding $\cG\times\cH$ certificates of sub-optimality for a model $f_{t-1}$ (Algorithm \ref{alg:altmin})

\begin{algorithm}
\KwInput{A dataset $D$, a model $f_{t-1}$, group and model classes $\cG$ and $\cH$, and an error parameter $\epsilon$.}
Let $(g^*,h^*) \in \cG \times \cH$ be an arbitrary initial certificate.  \\
CurrentValue $\leftarrow \mu_{D}(g^*)\cdot (L(D,f_{t-1},g^*) - L(D,h^*,g^*))$ \\
Let:
$$h^* = \arg\min_{h \in \cH} L(D_{g^*},h) \ \ \ g^* = \arg\min_{g \in \cG} L(D_{h^*},g)$$
\While{$\mu_{D}(g^*)\cdot (L(D,f_{t-1},g^*) - L(D,h^*,g^*) \geq \textrm{CurrentValue}+\epsilon$}{
CurrentValue $\leftarrow \mu_{D}(g^*)\cdot (L(D,f_{t-1},g^*) - L(D,h^*,g^*))$ \\
Let:
$$h^* = \arg\min_{h \in \cH} L(D_{g^*},h) \ \ \ g^* = \arg\min_{g \in \cG} L(D_{h^*},g)$$
}
\KwOut{Certificate $(g_t,h_t) = (g^*,h^*)$}
\caption{AltMinCertificateFinder($D,f_{t-1},\cG,\cH,\epsilon)$: An algorithm for finding $\epsilon$-locally optimal certificates of sub-optimality $(g_t,h_t) \in \cG\times \cH$ for model $f_{t-1}$.}
\label{alg:altmin}
\end{algorithm}

We have the following theorem:
\begin{restatable}{theorem}{localopt}
Let $D \in (\cX\times\cY)^n$ be an arbitrary dataset, $f_{t-1}:\cX\rightarrow \cY$ be an arbitrary model, $\cG$ and $\cH$ be arbitrary group and model classes, and $\epsilon > 0$. Then after solving at most $2/\epsilon$ empirical risk minimization problems over each of $\cG$ and $\cH$, AltMinCertificateFinder (Algorithm \ref{alg:altmin}) returns a $(\mu,\Delta)$-certificate of sub-optimality $(g_t,h_t)$ for $f_{t-1}$  that is an $\epsilon$-approximate local optimum (or saddle point) in the sense that:
\begin{enumerate}
    \item For every $h \in \cH$, $(g_t,h)$ is not a $(\mu', \Delta')$ certificate of sub-optimality for $f_{t-1}$ for any $\mu'\cdot \Delta' \geq \mu\cdot \Delta + \epsilon$, and
    \item For every $g \in \cG$, $(g,h_t)$ is not a $(\mu', \Delta')$ certificate of sub-optimality for $f_{t-1}$ for any $\mu'\cdot \Delta' \geq \mu\cdot \Delta + \epsilon$
\end{enumerate}
\end{restatable}

\begin{proof}
 By  the halting condition of the While loop, every iteration of the While loop increases  $\mu_{D}(g^*)\cdot (L(D,f_{t-1},g^*) - L(D,h^*,g^*))$ by at least $\epsilon$. Since this quantity is bounded in $[-1,1]$, there cannot be more than $2/\epsilon$ iterations. Each iteration solves a single empirical risk minimization problem over each of $\cH$ and $\cG$.

 The certificate $(g_t,h_t)$ finally output is a $(\mu,\Delta)$ certificate of sub-optimality for $\mu = \mu_{D}(g_t)$ and $\Delta = (L(D,f_{t-1},g_t) - L(D,h_t,g_t))$. It must be  that the objective $\mu_{D}(g_t)\cdot (L(D,f_{t-1},g_t) - L(D,h_t,g_t)) = \mu\cdot \Delta$ cannot be improved by more than $\epsilon$ by re-optimizing either $h_t$ or $g_t$, by definition of the halting condition and by Lemmas \ref{lem:hopt} and \ref{lem:gopt}. The theorem follows.
\end{proof}

\section{Additional Material from Section \ref{sec:experiments}}

\begin{table}[h!]
\begin{tabular}{lllllllllllll}
Dataset & Total & White & Black & Asian & Native & Other & Two+ & Male & Female & Young & Middle & Old \\
\hline
NY Employment & 196966 & 138473 & 24024 & 17030 & 10964 & 5646 & 829 & 95162 & 101804 & 68163 & 47469 & 81334 \\
Oregon Income & 21918 & 18937 & 311 & 923 & 552 & 823 & 372 & 11454 & 10464 & 5041 & 9124 & 7753 \\
Texas Coverage & 98927 & 72881 & 11192 & 4844 & 6660 & 2555 & 795 & 42128 & 56799 & 42048 & 31300 & 25579 \\
Florida Travel & 88070 & 70627 & 10118 & 2629 & 2499 & 1907 & 290 & 45324 & 42746 & 16710 & 35151 & 36209 \\

\end{tabular}
\caption{Summaries of the four Folktables state/task datasets we use in the experiments discussed below,
indicating total population size and the sizes of the different demographic subgroups considered.
The full descriptors of the demographic subgroups are: (race subgroups) White; Black or African American; Asian;  Native Hawaiian, Native American, Native Alaskan, or Pacific Islander; Some Other Race; Two or More Races; (binarized sex subgroups) Male; Female;
(age subgroups) Young; Middle; Old.}
\label{table:data}
\end{table}

\begin{figure}[h!]
        \center
        \subfloat{\includegraphics[width=.7\columnwidth,trim={0.5cm 0.5cm 0.5cm 0.5cm}]{figures/ACS-description2.png}}
        \caption{A fragment of the Folktables ACS PUMS dataset, with the 16 features considered for an employment prediction task: age (AGEP), education (SCHL), marital status (MAR), relationship (RELP), disability status (DIS), parent employee status (ESP), citizen status (CIT), mobility status (MIG), military service (MIL), ancestry record (ANC), nation of origin (NATIVITY), hearing difficulty (DEAR), visual dificulty (DEYE), learning disability (DREM), sex (SEX), and race (RAC1P).}
    \label{fig:acs-description}
\end{figure}

\begin{figure}[h!]
        \subfloat{\includegraphics[width=.3\columnwidth,trim={0.5cm 0.5cm 0.5cm 0.5cm}]{figures/ny_18_employ_12345_test.png}} \subfloat{\includegraphics[width=.3\columnwidth,trim={0.5cm 0.5cm 0.5cm 0.5cm}]{figures/ny_18_employ_12345_train.png}}
        \subfloat{\includegraphics[width=.3\columnwidth,trim={0.5cm 0.5cm 0.5cm 0.5cm}]{figures/ny_18_employ_12345_delta.png}} \\
        \subfloat{\includegraphics[width=.3\columnwidth,trim={0.5cm 0.5cm 0.5cm 0.5cm}]{figures/ny_18_employ_12346_test.png}} \subfloat{\includegraphics[width=.3\columnwidth,trim={0.5cm 0.5cm 0.5cm 0.5cm}]{figures/ny_18_employ_12346_train.png}}
        \subfloat{\includegraphics[width=.3\columnwidth,trim={0.5cm 0.5cm 0.5cm 0.5cm}]{figures/ny_18_employ_12346_delta.png}} \\
        \subfloat{\includegraphics[width=.3\columnwidth,trim={0.5cm 0.5cm 0.5cm 0.5cm}]{figures/or_18_income_12340_test.png}} \subfloat{\includegraphics[width=.3\columnwidth,trim={0.5cm 0.5cm 0.5cm 0.5cm}]{figures/or_18_income_12340_train.png}}
        \subfloat{\includegraphics[width=.3\columnwidth,trim={0.5cm 0.5cm 0.5cm 0.5cm}]{figures/or_18_income_12340_delta.png}} \\
        \subfloat{\includegraphics[width=.3\columnwidth,trim={0.5cm 0.5cm 0.5cm 0.5cm}]{figures/tx_18_cover_12345_test.png}} \subfloat{\includegraphics[width=.3\columnwidth,trim={0.5cm 0.5cm 0.5cm 0.5cm}]{figures/tx_18_cover_12345_train.png}}
        \subfloat{\includegraphics[width=.3\columnwidth,trim={0.5cm 0.5cm 0.5cm 0.5cm}]{figures/tx_18_cover_12345_delta.png}} \\
        \subfloat{\includegraphics[width=.3\columnwidth,trim={0.5cm 0.5cm 0.5cm 0.5cm}]{figures/fl_18_travel_12345_test.png}} \subfloat{\includegraphics[width=.3\columnwidth,trim={0.5cm 0.5cm 0.5cm 0.5cm}]{figures/fl_18_travel_12345_train.png}}
        \subfloat{\includegraphics[width=.3\columnwidth,trim={0.5cm 0.5cm 0.5cm 0.5cm}]{figures/fl_18_travel_12345_delta.png}}
        \caption{Sample results of an implementation of algorithm MonotoneFalsifyAndUpdate (Algorithm~\ref{alg:MonFalsifyAndUpdate}) on a
        number of different task and U.S. state datasets from the Folktables package. In each row, the 
        left panel shows group error rates on the test
        data; the middle panel shows group error rates on the training data; and the right panel shows the absolute difference of test and training error rates. First row: ACS employment task for New York state. Second row: ACS employment task for New York but with a
        different order of introduction of the groups. Third row: ACS income task for Oregon. Fourth row: ACS coverage task for Texas. Fifth row: ACS travel time task for Florida. See text for discussion. }
    \label{fig:acs-results-app}
\end{figure}

\begin{table}[h!]
\tabcolsep=0.11cm
\begin{tabular}{lllllllllllll}
Total & White & Black  & Asian & Native & Some Other Race & Two+ & Male & Female & Young & Middle & Old \\
0.0004 & 0.016 & 0.065 & 0.023 & 0.037 & 0.020 & 0.033 & 0.016 & 0.008 & 0.040 & 0.045 & 0.030 \\
\end{tabular}
\caption{Absolute differences in final overall and subgroup test errors between two different sequential orderings.
See text for details.}
\label{table:order-errors}
\end{table}

\fi